\documentclass[journal]{IEEEtran}
\ifCLASSINFOpdf
\else
\fi

\usepackage{epsfig}\usepackage{epsfig}
\usepackage{amssymb}
\usepackage{amsthm}
\usepackage{amsmath}
\usepackage{cite}
\usepackage{mathrsfs}
\usepackage{cases}
\usepackage{algorithm}
\usepackage{booktabs}
\usepackage{multirow}
\usepackage{threeparttable}
\usepackage{makecell}
\usepackage{url}
\usepackage{algorithmic}
\usepackage{subeqnarray}
\usepackage{subfigure}
\graphicspath{ {figure/} }
\usepackage{pifont,bbm,amsfonts,mathrsfs,amsmath,stmaryrd,amssymb}
\usepackage{amscd,graphicx,array,dsfont,texdraw,tikz,footnote,verbatim}%
\usepackage[normalem]{ulem}
\usepackage{mathtools}
\usepackage{bbm,bbding,amssymb,pifont,mathrsfs,amsfonts,graphicx,mathtools,color}%
\usepackage{amscd,array,enumerate,dsfont,texdraw,tikz,multicol,bbm}

\usetikzlibrary{arrows,shapes,snakes,shadows,positioning,automata,patterns}
\usetikzlibrary{trees,decorations.pathmorphing,decorations.markings}
\usepackage{setspace} 
\usepackage{tikz}
\usepackage{pgf}
\usepackage{amsmath, amsfonts, amssymb, amsthm}
\usepackage{pgfplots}
\pgfplotsset{compat=newest}
\usepackage{xxcolor}
\usetikzlibrary{arrows, decorations.markings}
\usetikzlibrary{arrows,shadows,petri}

\newcommand{\beq}{\begin{equation}}
\newcommand{\eeq}{\end{equation}}
\newcommand{\bqa}{\begin{eqnarray}}
\newcommand{\eqa}{\end{eqnarray}}

\definecolor{maroon}{rgb}{0.7,0,0}

\definecolor{ngreen}{rgb}{0.3,0.7,0.3}

\definecolor{golden}{rgb}{0.8,0.6,0.1}

\newtheorem{theorem}{\indent Theorem}
\newtheorem{lemma}{\indent Lemma}
\newtheorem{corollary}{\indent Corollary}
\newtheorem{definition}{\indent Definition}
\newtheorem{assumption}{\indent Assumption}
\newtheorem{myremark}{\indent Remark}

\newenvironment{remark}{\begin{myremark}\normalfont}
	{\end{myremark}}

\begin{document}
\title{Locally Differentially Private Distributed Online Learning with Guaranteed Optimality}

\author{Ziqin Chen and Yongqiang Wang, \textit{Senior Member, IEEE}
  \thanks{The work was supported in part by the National Science Foundation under Grants  ECCS-1912702, CCF-2106293, CCF-2215088,
  CNS-2219487, and CCF 2334449.}
\thanks{Ziqin Chen and Yongqiang Wang are with the Department of Electrical and Computer Engineering, Clemson University, Clemson, SC 29634 USA (e-mail: yongqiw@clemson.edu).}}

\maketitle
\begin{abstract}
Distributed online learning is gaining increased traction due to its unique ability to process large-scale datasets and streaming data. To address the growing public awareness and concern on privacy protection, plenty of algorithms have been proposed to enable differential privacy in distributed online optimization and learning. However, these algorithms often face the dilemma of trading learning accuracy for privacy. By exploiting the unique characteristics of online learning, this paper proposes an approach that tackles the dilemma and ensures both differential privacy and learning accuracy in distributed online learning. More specifically, while ensuring a diminishing expected instantaneous regret, the approach can simultaneously ensure
a finite cumulative privacy budget, even in the infinite time horizon. {\color{blue}To cater for the fully distributed setting, we adopt the local differential-privacy framework, which avoids the reliance on a trusted data curator that is required in the classic ``centralized" (global) differential-privacy framework}. To the
best of our knowledge, this is the first algorithm that successfully ensures both rigorous local differential privacy and learning accuracy. The effectiveness of the proposed algorithm is evaluated using machine learning tasks, including logistic regression on the the ``mushrooms" datasets and CNN-based image classification on the ``MNIST" and ``CIFAR-10" datasets.~\nocite{arxiv}
\end{abstract}


\IEEEpeerreviewmaketitle

\section{Introduction}
The modern data landscape, fueled by advances in web technologies, social media, and sensory devices, calls for evolved machine learning methods to handle the ``big data" challenge \cite{Bigdata1}. Due to its unique ability to handle streaming data, online learning has emerged as an attractive paradigm to address this challenge~\cite{onlineapply}. In online learning, data are accessed and processed in a sequential manner, thereby obviating the requirement to process the entire dataset at once. This feature makes online learning algorithms particularly appealing for large-scale datasets and dynamic scenarios, where data are continually generated, ranging from financial markets, social media streams, to real-time sensor interpretation.

Traditional online learning algorithms (e.g., \cite{centralizedonline1,centralizedonline2,centralizedonline3}) require transmitting all data streams to a central location for processing, leading to potential security risks like information leakage or model compromises in the event of a server attack \cite{attack1,attack2,attack3}. Distributed online learning algorithms mitigate these risks by dispersing data among multiple networked learners, each updating its model with local streaming data, and then sharing updates across the network for parameter synchronization (see, e.g., \cite{distributedonline1,distributedonline2,distributedonline3,distributedonline4,distributedonline5,distributedonline6}). While these algorithms eliminate the need for centralized data storage and associated security risks, information leakage during parameter transmission remains a concern, particularly via unencrypted communication channels. 
In fact, using these shared parameters, not only can an adversary infer sensitive attributes of the original data \cite{attack2}, but it can also precisely reversely infer raw training data (pixel-wise accurate for images \cite{attack3}). To mitigate privacy breaches in distributed online learning, one natural approach is to patch an online learning algorithm with existing privacy mechanisms. For example, partially homomorphic encryption has been employed in both our prior results as well as others' to ensure privacy in distributed optimization \cite{encryption1,encryption2,encryption3}. However, such approaches suffer from heavy communication and computation overheads. Alternatively, time or spatially correlated noise-based approaches preserve privacy while maintaining accuracy by canceling out injected noises \cite{structured1,structured2,structured3,structured4,structured5}. However, such approaches require each learner to have at least one neighbor not sharing information with potential adversaries, a condition that is difficult to guarantee in many multi-agent networks. 
\subsection{Related Literature}
As differential privacy (DP) is gaining increased traction due to its mathematical rigor, implementation simplicity, and post-processing immunity~\cite{Dwork2010,Dwork2014}, plenty of results have been proposed to enable differential privacy in distributed optimization/learning~\cite{explicity1,explicity2,explicity3,explicity4,PDOP,offlineDP1,offlineDP2,offlineDP3,offlineDP4,offlineDP5,tailoring,nonconvexwang}. However, most existing differential-privacy results for distributed optimization/learning explicitly rely on a trusted curator to aggregate and publish data in a centralized manner~\cite{explicity1,explicity2,explicity3,explicity4}. Recently, some differential-privacy solutions have been proposed for fully distributed optimization algorithms, including~\cite{PDOP,offlineDP1,offlineDP2,offlineDP3,offlineDP4,offlineDP5} as well as our own prior work~\cite{tailoring,nonconvexwang,wangcentrial}. However, since these
results still use {\color{blue}the classical centralized differential-privacy framework\footnote{{\color{blue}By centralized differential privacy, we mean the traditional differential-privacy framework, where a data aggregator/curator is needed to collect data from all learners and inject differential-privacy noises. Note that although results such as\cite{offlineDP1,offlineDP3,offlineDP4,offlineDP5,tailoring,nonconvexwang,wangcentrial} do not explicitly assume the existence of a data aggregator/curator, they still require participating learners to trust each other to cooperatively determine the amount of noises needed to achieve a certain level of privacy protection (privacy budget). Hence, they are also somewhat ``centralized," and hence, different from the local model of differential privacy in this paper.}}}, they do not explicitly address protection against information inference by participating learners~\cite{wangcentrial}. To ensure
privacy in the scenario where a learner does not trust anyone else (including other participating learners) and aims to protect against an adversary that can observe every message shared in the network, we have to use local differential privacy (LDP), which obviates the need for a data curator/aggregator that is required in the traditional centralized differential-privacy framework to collect data and inject noises~\cite{LDP1,LDP2,LDP3}.

Unfortunately, the benefit of LDP comes at a great cost in optimization/learning accuracy. To the best of our knowledge, all existing differential-privacy solutions for distributed online learning have to either sacrifice learning accuracy~\cite{sacrifice1,sacrifice2,sacrifice3} or allow the cumulative privacy budget to grow to infinity with time, implying diminishing privacy protection as time tends to infinity~\cite{DOLA,Zhu2018,Xiong2020,Lu2020,Han2022,Cheng2023,Yuan2023,Lu2023,Zhao2024}. It is worth noting that our own prior work~\cite{tailoring,nonconvexwang} as well as others'~\cite{offlineDP5,wang2024} have managed to retain provable convergence accuracy and differential privacy in distributed offline optimization. However, they still use the classical centralized DP framework, and it is unclear if the offline learning approaches can be extended to the online learning scenario, where data arrive sequentially. 
\subsection{Contributions}
In this paper, we propose a locally differentially private distributed online learning algorithm that efficiently circumvents the tradeoff between privacy and learning accuracy. Our key idea is to exploit both the unique data patterns in online learning and a decaying interaction strength which enables the injection of DP noises with increasing variances (in contrast to decreasing DP-noise variances commonly used in the literature). The main contributions are summarized as follows:
\begin{itemize}
\item We demonstrate that for both strongly convex and general convex objective functions, our proposed locally differentially private distributed
online learning algorithm ensures that the
expected instantaneous regret decreases to zero, even in the presence of increasing DP-noise variances. Moreover, in the strongly convex scenario, we further prove that
the expected tracking error (the deviation between the online algorithm's output and the optimal solution) also converges to zero. To the best of our knowledge, no such results have been reported before.

\item In addition to ensuring provable convergence, we prove that our algorithm can simultaneously ensure rigorous LDP, even in the infinite time horizon. To our knowledge,
this is the first time that both goals of LDP and provable convergence are achieved simultaneously in distributed online learning. This is in sharp contrast to existing results on differentially private distributed online learning
in~\cite{DOLA,Zhu2018,Xiong2020,Lu2020,Han2022,Cheng2023,Yuan2023,Lu2023,Zhao2024}, where the cumulative privacy budget grows
to infinity when time tends to infinity. 

\item Moreover, our LDP framework allows individual learners to choose heterogeneous privacy budgets in a fully distributed manner, making individual learners free to
choose desired privacy strengths depending on practical needs.

\item Besides providing a theoretic approach to selecting stepsizes based on global parameters such as graph Laplacian and the global Lipschitz constant (which is common
in most existing distributed online optimization algorithms~\cite{distributedonline1,distributedonline2,distributedonline3,distributedonline4,distributedonline5,distributedonline6}), we also provide an approach for individual
learners to select stepsizes independently of any global parameters, which is more amenable to distributed implementations.

\item We evaluated the performance of our approach using
several benchmark machine learning datasets, including the ``mushrooms" dataset for logistic regression and the ``MNIST" and ``CIFAR-10" datasets for CNN-based image classification. The results corroborate
the effectiveness of our approach. Notably, compared with existing differentially private distributed learning/optimization methods in~\cite{PDOP,DOLA,DSGD}, our algorithm demonstrates higher
training and test accuracies. 
\end{itemize}

The organization of the paper is as follows. Sec.~\ref{problemstatement} introduces the problem formulation and definitions for LDP. Sec.~\ref{algorithmdesign} presents a locally differentially private distributed online learning algorithm and discusses its computational complexity. Sec.~\ref{convergenceanalysis} analyzes the learning accuracy of the proposed algorithm. Sec.~\ref{extension} provides an approach to selecting stepsizes independently of any global parameters. Sec.~\ref{privacy} establishes the LDP guarantees. Sec.~\ref{example} presents experimental
results on benchmark datasets. Sec.~\ref{conclusion} concludes the paper.

\textbf{Notations:} We use $\mathbb{R}^{n}$ to denote the $n$-dimensional Euclidean space. We also use $\mathbb{N}$ and $\mathbb{N}^{+}$ to denote the natural number and the positive natural number, respectively. $I_{n}$ represents the identity matrix of dimension $n$ and $\mathbf{1}_{n}$ represents the $n$-dimensional column vector with all entries equal to $1$. We use $\|\cdot\|$ and $\|\cdot\|_{1}$ to represent the Euclidean norm and $l^{1}$-norm of a vector, respectively. The Kronecker product is denoted by $\otimes$. The stacked column vector of vectors or scalars $\theta_{1},\cdots,\theta_{m}$ is denoted by ${\rm col}\{\theta_{1},\cdots,\theta_{m}\}$. The transpose of a matrix $A$ is written as $A^{T}$. The notation $\lceil a \rceil$ refers to the smallest integer no less than $a$, and $\lfloor a \rfloor$ represents the largest integer no greater than $a$. We use $[m]$ to denote the set $\{1,2,\cdots,m\}$. For any $\theta\in \mathbb{R}^{n}$, we use $\text{Pro}_{\Theta}(\theta)=\text{argmin}_{\theta'\in\Theta}\|\theta-\theta'\|_{2}$ to represent 
the Euclidean projection onto a set $\Theta\subseteq \mathbb{R}^{n}$. We also use $\text{Lap}(\varrho)$ to denote Laplace distribution with parameter $\varrho>0$, featuring a probability density function  $p_{\varrho}(x)\triangleq\frac{1}{2\varrho}e^{\frac{-|x|}{\varrho}}$. $\text{Lap}(\varrho)$ has a mean of zero and a variance of $2 \varrho^2$.

\section{Problem Statement}\label{problemstatement}
\subsection{Distributed online learning}
In distributed online learning, each Learner $i,~i\in[m]$ must perform learning on streaming data that arrive sequentially. More specifically, at time $t$, Learner $i$ acquires a data point $a_{t}^{i}$, which is independently and identically sampled from an unknown distribution over a sample space $\Omega_{i}$. Using model parameter $\theta_t^{i}$ learned from data prior to time $t$, which is usually constrained in a convex subset $\Theta$ of $\mathbb{R}^n$, Learner $i$ predicts a label $\hat{b}_{t}^{i}$ for the data $x_t^i$ acquired at time $t$. When the true label $b_{t}^{i}\in \mathbb{R}$ is revealed, Learner $i$ experiences a loss $l(\theta_{t}^{i},\xi_{t}^{i})$, where $\xi_{t}^{i}\!=\!(a_{t}^{i},b_{t}^{i})$ resides in $\mathcal{P}_{i}\!=\!\Omega_{i}\times\mathbb{R}$. The loss prompts Learner $i$ to adjust its model parameter $\theta_t^i$. The goal of distributed online learning is to let the $m$ learners cooperatively find a common optimal parameter, based on sequentially acquired streaming data, for the following stochastic optimization problem:
\vspace{-0.5em}
\begin{equation}
\text{min}_{\theta\in \Theta}\quad F(\theta):=\frac{1}{m}\sum_{i=1}^{m}f_{i}(\theta),\label{primal}
\end{equation}
where $f_{i}(\theta)\!=\!\mathbb{E}_{\xi^{i}\sim\mathcal{P}_{i}}\!\left[l(\theta,\xi^{i})\right]$ satisfies the following assumption:

\begin{assumption}\label{A1}
(i) $\Theta$ is a convex and compact subset of $\mathbb{R}^{n}$ with nonempty interior; (ii) for all $i\in [m]$ and $x,y\in \Theta$, there exists some $\mu\geq 0$ such that 
$f_{i}(y)\geq f_{i}(x)+\nabla f_{i}(x)^{T}(y-x)+\frac{\mu}{2}\|x-y\|^2$ holds; and (iii) there exists some positive constant $D$ such that $\|\nabla f_{i}(\theta)\|\leq D$ holds for all $\theta\in\Theta.$
\end{assumption}

We describe the communication pattern among learners using an $m\times m$ matrix $W$. If Learners $i$ and $j$ can communicate with each other, then $w_{ij}$ is positive, and $w_{ij}=0$ otherwise. The set of learners that can directly interact with Learner $i$ is called the neighboring set of Learner $i$ and is represented as $\mathcal{N}_{i}$. We let $w_{ii}=-\sum_{j\in{\mathcal{N}_{i}}}w_{ij}$. The matrix $W$ satisfies the following assumption:
\begin{assumption}\label{A2}
The matrix $W$\footnote{Our matrix $I+\epsilon W$ corresponds to the Perron matrix $P_{\epsilon}=I-\epsilon L$ used in~\cite{matrix}, where $L$ is the Laplacian matrix.} satisfies $\mathbf{1}^{T}W=\mathbf{0}^{T}$ and $W\mathbf{1}=\mathbf{0}$. The eigenvalues of $W$ satisfy
(after arranged in an increasing order) $-1<\delta_{m}<\delta_{m-1}\leq \cdots\leq \delta_{1}\leq 1$.
\end{assumption}
To solve for~\eqref{primal} with $f_{i}$ equal to the expected value of the loss function $l(\theta,\xi^{i})$, we have to know the distribution $\mathcal{P}_{i}$ of $\xi^{i}$. However, in practice, the distribution $\mathcal{P}_{i}$ is usually unknown, which makes it impossible to directly compute $\mathbb{E}_{\xi^{i}\sim\mathcal{P}_{i}}\left[l(\theta,\xi^{i})\right]$. To circumvent this problem, a common approach is reformulating~\eqref{primal} as the following Empirical Risk Minimization (ERM) problem:
\begin{equation}
\text{min}_{\theta\in \Theta} F_{t}(\theta)\triangleq\frac{1}{m}\sum_{i=1}^{m}f_{t}^{i}(\theta),~ f_{t}^{i}(\theta)=\frac{1}{t+1}\sum_{k=0}^{t}l(\theta,\xi_{k}^{i}).\label{primalt}
\end{equation}

According to the law of large numbers, we have $\lim_{t\rightarrow\infty}F_{t}(\theta)=F(\theta)$, implying that the solution $\theta_{t}^*$ to the ERM problem~\eqref{primalt} will gradually approach the solution $\theta^*$ to the problem~\eqref{primal} as time $t$ tends to infinity (detailed proofs can be found in Lemma 2 in~\cite{zijiGT} and Section 5.1.2 in~\cite{SAA}). This is an intrinsic property of our ERM problem setting. It is worth noting that different from the conventional ERM problem, where all data are accumulated prior to performing training, here we have to perform online training from experience as more data are observed. In addition, since $\xi^{i}\sim\mathcal{P}_{i}$ are randomly streaming data, the gradients $\nabla l(\theta,\xi^{i})$ are stochastic, which we assume to satisfy the following standard assumption~\cite{wang2024}:

\begin{assumption}\label{A3}
The random data points $\{\xi^i\}$ are independent of each other. In addition, (i) $\mathbb{E}[\nabla l(\theta,\xi^{i})]=\nabla f_{i}(\theta)$; (ii) $\mathbb{E}[\|\nabla l(\theta,\xi^{i})-\nabla f_{i}(\theta)\|^2]\leq \kappa^2$; and (iii) $\|\nabla l(x,\xi^{i})-\nabla l(y,\xi^{i})\|\leq L\|x-y\|$ for any $x,y\in\Theta$. 
\end{assumption}
Furthermore, given the streaming nature of data,	
the objective function $F_{t}(\theta)$ in~\eqref{primalt} varies with time, which further leads to time-varying optimal solutions $\theta^{\ast}_t$. Hence, to evaluate the quality of the parameters learned by learners through a distributed online learning algorithm at each time instant, we employ metrics of the \textit{expected tracking error} $\mathbb{E}[\|\theta_{t}^{i}-\theta_{t}^*\|^2]$ and the \textit{expected instantaneous regret} $\mathbb{E}[F_{t}(\theta_{t}^{i})-F_{t}(\theta_{t}^{*})]$.

\begin{remark}
The expected tracking error and the expected instantaneous regret are commonly used metrics in existing literature on online optimization and learning~\cite{Trackerror1, Trackerror2, losserror1, losserror2}. They capture the real-time performance of an online algorithm, and, hence, are well-suited in the online learning setting where data arrive sequentially~\cite{losserror3}.
\end{remark}

\subsection{Local differential privacy}
Local differential privacy is a local (distributed) model of differential privacy for scenarios where no trusted data aggregator (curator) exists to aggregate data and execute a privacy mechanism. It is contrasted with the classic centralized differential privacy, where a trusted aggregator gathers all raw data and then executes a differentially private data publishing mechanism. In distributed learning and optimization, each learner maintains a local dataset and shares learned parameters with neighbors to collaboratively optimize these parameters. This information exchange has the risk of information leakage as malicious external attackers or curious neighbors might try to recover raw training data from shared parameters \cite{attack1,attack2}. To protect the privacy of all learners, we adopt LDP to address the most severe scenario: all communication channels can be compromised by malicious attackers and no learners are trustworthy. Consequently, not only does our LDP-based approach deters external adversaries from extracting raw data through shared information, but it also shields against neighboring curious learners within the network.

To facilitate privacy analysis, we need the definition of adjacency of local datasets:~\cite{DOLA,Cheng2023,wang2024}:
\begin{definition}\label{Definition3}
(Adjacency) For any $t\in{\mathbb{N}^{+}}$ and any learner $i\in[m]$, given two local datasets $\mathcal{D}_{t}^{i}=\{\xi_{1}^{i},\cdots,\xi_{k}^{i},\cdots,\xi_{t}^{i}\}$ and ${\mathcal{D}_{t}^{i}}'=\{\xi_{1}^{i},\cdots,{\xi_{k}^{i}}',\cdots,\xi_{t}^{i}\}$, $\mathcal{D}_{t}^{i}$ is said to be adjacent to ${\mathcal{D}_{t}^{i}}'$ if there exists a time instant $k\in [1,t]$ such that $\xi_{k}^{i}\neq{\xi_{k}^{i}}'$ while $\xi_{p}^{i}={\xi_{p}^{i}}'$ for all $p\in[1,t]$ and $p\neq k$.
\end{definition}

It can be seen that for any given time $t$, $\mathcal{D}_{t}^{i}$ is adjacent to ${\mathcal{D}_{t}^{i}}'$ if and only if $\mathcal{D}_{t}^{i}$ and ${\mathcal{D}_{t}^{i}}'$ differ in a single entry while all other entries are the same. Definition~\ref{Definition3} also implies that for any given time $t$, two adjacent datasets $\mathcal{D}_{t}=\mathcal{D}_{t}^{1}\cup\cdots\mathcal{D}_{t}^{i}\cup\cdots\cup\mathcal{D}_{t}^{m}$ and ${\mathcal{D}}'_{t}={\mathcal{D}_{t}^{1}}'\cup\cdots{\mathcal{D}_{t}^{i}}'\cup\cdots\cup{\mathcal{D}_{t}^{m}}'$ differ in $m$ entries. We use $\text{Adj}(\mathcal{D}_{t}^{i},{\mathcal{D}_{t}^{i}}')$ to denote the adjacent relationship between two local datasets $\mathcal{D}_{t}^{i}$ and ${\mathcal{D}_{t}^{i}}'$.
\begin{remark}
It is worth noting that our definition of adjacency corresponds to the so-called event-level LDP in the literature \cite{Dwork2014}. It allows $m$ entries in the global datasets of all learners to be different, and is more stringent than most existing results using the traditional centralized version of DP (e.g., \cite{Zhu2018,Xiong2020,Lu2020,Han2022,Yuan2023,Lu2023,Zhao2024}), where at each time instant $t$, only one entry is allowed to be different. It is also worth noting that allowing one learner to have all data entries to be different (called user-level DP \cite{Dwork2014}) has been proven infeasible in distributed optimization/learning under the local model of DP \cite{geyer2017,mcmahan2018,melis2019}.  
\end{remark}

\begin{algorithm}
\caption{Locally differentially private distributed online learning for $i\in[m]$}
\label{algorithm1} 
\begin{algorithmic}[1]
\STATE {\bfseries Input:} Random initialization $\theta_{0}^{i}\in{\Theta}$; $\lambda_{t}=\frac{\lambda_{0}}{(t+1)^{v}}$ with 
$\lambda_{0}>0$ and $v\in(\frac{1}{2},1)$; decaying sequence $\gamma_{t}=\frac{\gamma_{0}}{(t+1)^{u}}$ with $\gamma_{0}>0$ and $u\in(\frac{1}{2},1)$.
\FOR{$t=0,1,\cdots,T-1$}
\STATE Use all available data up to time $t$, i.e., $\xi_{k}^{i}\in \mathcal{D}_{t}^{i},~k\in [0,t]$ and the current parameter $\theta_{t}^{i}$ to compute:
\STATE $d_{t}^{i}(\theta_{t}^{i})=\frac{1}{t+1}\sum_{k=0}^{t}\nabla l(\theta_{t}^{i},\xi_{k}^{i}).$
\STATE Add DP noises $\zeta_{t}^{i}$ to $\theta_{t}^{i}$, and then send the obscured value $y_{t}^{i}\triangleq\theta_{t}^{i}+\zeta_{t}^{i}$ to neighbors $j\in{\mathcal{N}_{i}}$.
\STATE Receive $y_{t}^{j}$ from neighbors $j\in{\mathcal{N}_{i}}$.
\STATE $\hat{\theta}_{t+1}^{i}=\theta_{t}^{i}+\sum_{j\in{\mathcal{N}_{i}}}\gamma_{t}w_{ij}(y_{t}^{j}-\theta_{t}^{i})-\lambda_{t}d_{t}^{i}(\theta_{t}^{i});$
\STATE $\theta_{t+1}^{i}=\text{Pro}_{\Theta}(\hat{\theta}_{t+1}^{i}).$
\ENDFOR
\end{algorithmic}
\end{algorithm}

Given a distributed online learning problem~\eqref{primalt}, we denote the implementation of an online algorithm by Learner $i\in[m]$ as $\mathcal{A}_{i}$. Now we are in a position to present the definition of LDP~\cite{Dwork2014}:
\begin{definition}\label{Definition4}
	(Local differential privacy). Let $\mathcal{A}_{i}(\mathcal{D}^{i},\theta^{-i})$ denote the output of Learner $i$ under a distributed learning algorithm with its local dataset $\mathcal{D}^{i}$ and all received information from neighbors $\theta^{-i}$. Then, Learner $i$'s implementation $\mathcal{A}_i$ is $\epsilon_{i}$ locally differentially private if the following inequality always holds for any two adjacent datasets $\mathcal{D}^{i}$,~${\mathcal{D}^{i}}'$: 
	\begin{flalign}
		\mathbb{P}[\mathcal{A}_{i}(\mathcal{D}^{i},\theta^{-i})\in \mathcal{O}^{i}] \leq e^{\epsilon_{i}} \mathbb{P}[\mathcal{A}_{i}({\mathcal{D}^{i}}',\theta^{-i})\in \mathcal{O}^{i}],
	\end{flalign}
	where $\mathcal{O}^{i}$ represents the set of all possible observations.
	\vspace{-0.5em}
\end{definition}

The parameter $\epsilon_{i}$ measures the similarity (indistinguishability) of Learner $i$'s output distributions under two adjacent datasets. A smaller value of $\epsilon_{i}$ indicates greater indistinguishability between the outputs for two adjacent datasets, implying a higher level of privacy protection. 

In our definition of LDP, for Learner $i$, all received information from neighbors, i.e., $\theta^{-i}$, is regarded as external information and beyond its control. This is different from the classic centralized DP definition used in existing differentially private distributed optimization/learning approaches~\cite{offlineDP1,offlineDP3,offlineDP4,offlineDP5,tailoring,nonconvexwang,wangcentrial}, which, in the absence of a data aggregator/curator, requires participating learners to trust each other and cooperatively determine the amount of noises needed to achieve a certain level of privacy protection (privacy budget). In fact, when no data aggregator/curator exists, such a centralized DP framework even allows agents to cooperatively decide (like a centralized data curator) how to mask shared information~\cite{wangcentrial}.

\section{Locally Differentially Private Distributed online learning algorithm}\label{algorithmdesign}
\subsection{Algorithm design}
Our locally differentially private distributed online learning algorithm to solve problem~\eqref{primalt} is summarized in Algorithm~\ref{algorithm1}, in which DP noises $\zeta_{t}^{i}\in{\mathbb{R}^{n}}$ satisfy the following assumption: 
\begin{assumption}\label{A4}
For every learner $i\in[m]$ and $t\in{\mathbb{N}}$, each element of the DP-noise vector $\zeta_{t}^{i}$ follows Laplace distribution $\text{Lap}(\varrho_{t}^{i})$ with $\varrho_{t}^{i}=\frac{\sigma^{i}}{\sqrt{2}}(t+1)^{\varsigma^{i}}$, where $\sigma^{i}$ is a positive constant and the increasing rate of noise variances $\varsigma^{i}\in(0,\frac{1}{2})$ satisfies 
\begin{equation}
\max_{i\in[m]}\{\varsigma^{i}\}+\frac{1}{2}<u<v<1,
\end{equation}
with $u$ and $v$ the decaying rates of the decaying sequence $\gamma_{t}$ and the stepsize $\lambda_{t}$ in Algorithm~\ref{algorithm1}, respectively.
\end{assumption}

Instead of using DP noises with decaying variances, we employ DP noises with increasing variances in Algorithm~\ref{algorithm1}. This is fundamentally different from existing results on differentially private distributed optimization, such as \cite{PDOP,offlineDP1,offlineDP2,offlineDP4,DOLA,Zhu2018,Xiong2020,Lu2020,Han2022,Cheng2023,Yuan2023,Lu2023}, and is key for us to ensure both accurate convergence and strong differential privacy with a finite cumulative privacy budget even in the infinite time horizon. In fact, most existing results on differentially private distributed optimization have to either sacrifice accurate convergence~\cite{sacrifice1,sacrifice2,sacrifice3} or allow the cumulative privacy budget to grow to infinity (meaning diminishing privacy protection as iteration tends to infinity)~\cite{DOLA,Zhu2018,Xiong2020,Lu2020,Han2022,Cheng2023,Yuan2023,Lu2023,Zhao2024}, and, to our knowledge, our approach is the first to achieve both accurate convergence and differential privacy in the infinite time horizon for online learning. 

One key reason for our algorithm to ensure robustness to DP noises is using a decaying sequence $\gamma_{t}$, which can effectively suppress the influence of DP noises with increasing variances, and, hence, ensuring accurate convergence. This approach is inspired by our recent result on distributed offline optimization~\cite{tailoring}. Nevertheless, it is worth noting that compared with the result in~\cite{tailoring}, where the objective function is predetermined and the same for all iterations, the objective function here changes over iterations due to sequentially arriving data. Furthermore, unlike~\cite{tailoring} where the optimal solutions can be any point in $\mathbb{R}^{n}$, here we consider optimization problems where the optimal solutions have to be restricted in a convex set $\Theta$. This constraint makes convergence analysis much more challenging because the nonlinearity induced by projection (necessary to address set constraints) poses challenges to both optimality analysis and consensus characterization.

Moreover, we propose a novel gradient computation strategy that exploits historical data. This strategy improves learning accuracy and reduces the sensitivity of our algorithm, which is key to ensuring a finite cumulative privacy budget even in the infinite time horizon. This is in sharp contrast to existing DP solutions for distributed online optimization/learning~\cite{DOLA,Zhu2018,Xiong2020,Lu2020,Han2022,Cheng2023,Yuan2023,Lu2023,Zhao2024}, whose cumulative privacy budgets explode to infinity as the number of iterations tends to infinity, implying diminishing privacy protection in the infinite time horizon. In addition, as the data point at any single iteration $t$ might be lost or corrupted, our strategy of using all available data up to time $t$ also enhances the robustness of the learning algorithm. The advantage of this strategy is clearly demonstrated later in experimental results (see Fig.~\ref{mushrooms}-Fig.~\ref{cifar10}) and privacy analysis (see Eq.~\eqref{T5result2}).
\begin{remark}
Note that all existing results on differentially private distributed online optimization follow the approach of patching DP noises with a given existing distributed optimization/learning algorithm (e.g., \cite{DOLA,Zhu2018,Xiong2020,Lu2020,Han2022,Cheng2023,Yuan2023,Lu2023,Zhao2024}), which does not fully exploit the flexibilities in DP design and optimization algorithm design. {\color{blue}In fact, almost all existing distributed optimization algorithms (which are designed without considering privacy) are not robust to DP noises (since directly incorporating DP noises into these optimization algorithms renders them unable to guarantee convergence to the exact optimal solution).} Hence, a direct combination of these existing algorithms with DP designs has to sacrifice either DP strength or convergence accuracy. In contrast, by incorporating a judiciously designed decaying factor $\gamma_{t}$ to gradually attenuate the influence of DP noises, we co-design the optimization algorithm and DP-noise injection mechanism, which enables us to achieve both differential privacy and accurate convergence.
\end{remark}
\begin{remark}
A commonly used approach to enabling privacy protection in distributed optimization/learning is to broadcast $\theta_{t}^{i}+\gamma_{t}\zeta_{t}^{i}$ and make the consensus of optimization variables $\theta_{t}^{i}$ unaffected by the decaying sequence $\gamma_{t}$~\cite{PDOP}. Although this approach reduces the amount of noises injected into the algorithm, and, hence, will make convergence easier to happen, its diminishing noise variance also jeopardizes the strength of privacy protection, leading to an exploding cumulative privacy budget (implying diminishing privacy protection as iteration proceeds) under the stepsize strategy used in our paper.
\vspace{-0.2em}
\end{remark}
\vspace{-0.7em}
\subsection{Algorithm complexity discussion}
In this subsection, we discuss the computational complexity of our strategy that uses historical data in our Algorithm 1. It is intuitive that using all data available at time $t$ can increase execution time of the algorithm compared with traditional online optimization/learning algorithms~\cite{distributedonline1,distributedonline2,distributedonline3,distributedonline4,distributedonline5,distributedonline6} that use only one current data sample. However, here we show that the increased computational complexity can be mitigated by exploiting the characteristics of learning problems. More specifically, if the loss function $l(\theta,\xi^{i})$ is a polynomial function of $\theta$, we can make sure that our strategy of using all historical data has the same order of computational complexity as those only using one data point at each time instant.

We illustrate the idea by using the Ridge regression problem~\cite{complexity1}. In the Ridge regression problem, the loss function is a quadratic function of $\theta$, i.e., $l(\theta,\xi_{k}^{i})=(b^{i}_{k}-a^{i}_{k}\theta)^{T}(b^{i}_{k}-a^{i}_{k}\theta)+\alpha_{t}\theta^{T}\theta.$
The gradient $d_{t}^{i}(\theta_{t}^{i})$ at each time $t$ is given as $d_{t}^{i}(\theta_{t}^{i})=\frac{1}{t+1}\sum_{k=0}^{t}\nabla l(\theta_{t}^{i},\xi_{k}^{i})$ with $\nabla l(\theta,\xi_{k}^{i})\!=\!-2(a_{k}^{i})^{T}(b_{k}^{i}-a_{k}^{i}\theta)+ 2\alpha_{t}\theta.$ Hence, we have
\begin{equation}
d_{t}^{i}(\theta_{t}^{i})=\frac{d_{t-1}^{i}(\theta_{t}^{i})\times t +\nabla l(\theta_{t}^{i},\xi_{t}^{i})}{t+1}.\label{Cp1}
\end{equation}
Using the linear interpolation (two-point interpolation) method, we can obtain $d_{t-1}^{i}(\theta_{t}^{i})$ as follows:
\begin{equation}
d_{t-1}^{i}(\theta_{t}^{i})\!=\!d_{t-1}^{i}(\theta_{t-1}^{i})\frac{\theta_{t}^{i}\!-\!\theta_{t-2}^{i}}{\theta_{t-1}^{i}\!-\!\theta_{t-2}^{i}}+d_{t-1}^{i}(\theta_{t-2}^{i})\frac{\theta_{t}^{i}\!-\!\theta_{t-1}^{i}}{\theta_{t-2}^{i}\!-\!\theta_{t-1}^{i}}.\label{Cp2}
\end{equation}
In the preceding equality, $d_{t-1}^{i}(\theta_{t-2}^{i})$ can be expressed as
\begin{equation}
d_{t-1}^{i}(\theta_{t-2}^{i})=\frac{d_{t-2}^{i}(\theta_{t-2}^{i})\times (t-1) +\nabla l(\theta_{t-2}^{i},\xi_{t-1}^{i})}{t},\label{Cp3}
\end{equation}
where the term $d_{t-2}^{i}(\theta_{t-2}^{i})$ has been calculated at time $t-2$ and $\nabla l(\theta_{t-2}^{i},\xi_{t-1}^{i})$ can be computed at time $t$.

Therefore, by combining~\eqref{Cp1},~\eqref{Cp2}, and \eqref{Cp3}, we can see that the gradient $d_{t}^{i}(\theta_{t}^{i})\!\!=\!\!\frac{1}{t+1}\sum_{k=0}^{t}\!\nabla l(\theta_{t}^{i},\xi_{k}^{i})$ needed at time $t$ can be computed in a recursive manner. By simply storing two gradients computed in the prior two time instants, we can keep the computational complexity invariant with time.

Using a similar argument, we can show that when the loss function is a polynomial function (like in Lasso and polynomial regression) of $\theta$ of order $n$, we can exploit the iterative formulation in~\eqref{Cp1} and the Lagrange interpolation method to control the computational complexity of the gradient to be $\mathcal{O}(n+1)$. 
 
It is worth noting that since every continuous function can be approximated as closely as desired by a polynomial function according to the Weierstrass approximation theorem~\cite{approximation}, the interpolation-based approach can be used in other non-polynomial loss functions to mitigate the computational complexity of our gradient computation strategy. In fact, the sigmoid and logarithmic loss functions in logistic regression have been shown to be easily approximated by polynomials~\cite{complexity2}. Even the cross-entropy and focal losses in neural networks have also been shown to be approximatable efficiently by a series of weighted polynomial bases~\cite{complexity3}.

\section{Tracking accuracy analysis}\label{convergenceanalysis}
In this section, we systematically analyze the learning accuracy of Algorithm~\ref{algorithm1} under both strongly convex and general convex objective functions.
\subsection{Tracking analysis with strongly convex objective functions ($\mu$ in Assumption~\ref{A1} is positive)}
We first analyze the time variation of the optimal parameter:
\begin{lemma}\label{L1}
Denote $\theta_{t}^*$ as the optimal solution to the online optimization problem~\eqref{primalt} at time $t$. Under Assumption~\ref{A1} with $\mu>0$ and Assumption~\ref{A3}, we have
\begin{equation}
\mathbb{E}[\|\theta_{t+1}^*-\theta_{t}^*\|^2]\leq \mathcal{O}\left((t+1)^{-2}\right),\label{L1result}
\end{equation}
which implies $\lim_{t\rightarrow \infty}\mathbb{E}[\|\theta_{t+1}^*-\theta_{t}^*\|^2]=0$.
\end{lemma}
\begin{proof}
See Appendix A.
\end{proof}
\begin{remark}
Lemma~\ref{L1} reveals a key property of our learning problem \eqref{primalt}: as learning progresses, the variation in optimal parameters decreases with time at a rate of $\mathcal{O}((t+1)^{-2})$. This decreasing rate is an intrinsic property of the problem setting in~\eqref{primalt}. Specifically, the objective function is the average of loss functions over a growing number of samples. As more data points are acquired, any single data point's impact on the overall loss becomes progressively smaller. The cumulative moving average acts as a form of memory, which makes the learning process smoother and more stable. 

Notably, only when the data distribution $\mathcal{P}_i$ is time-invariant, the optimal parameter to the problem~\eqref{primalt} could converge to a fixed constant. However, our result in Lemma~\ref{L1} is applicable even when the data distribution $\mathcal{P}_i$ is not time-invariant, or in other words, the optimal parameter does not have to converge to a constant value. For example, if the optimal parameter follows the sequence $\theta_{t}^*=1, 1+\frac{1}{2}, 1+\frac{1}{2}+\frac{1}{3}, \cdots$, it can be verified that the result in Lemma~\ref{L1} still applies, whereas the sequence never converges.
\end{remark} 

We now characterize the expected tracking error of Algorithm~\ref{algorithm1} for strongly convex objective functions.
\begin{theorem}\label{T1}
Under Assumptions~\ref{A1}-\ref{A4} with $\mu>0$, if $0<\gamma_{0}\leq\frac{1}{-3\delta_{m}}$ and $0<\lambda_{0}\leq\frac{-\gamma_{0}\delta_{2}\mu}{\mu^2+8L^2}$ hold, the expected tracking error of Algorithm~\ref{A1} satisfies
\begin{equation}
\mathbb{E}[\|\theta_{t+1}^{i}-\theta_{t+1}^*\|^2]\leq \mathcal{O}(t^{-\beta}),\label{T1result}
\end{equation}
for all $t>0$, where the rate $\beta$ satisfies $\beta=\min\{1-v,2u-2\varsigma-1\}$ {\color{blue} with $\varsigma\triangleq\min_{i\in[m]}\{\varsigma^{i}\}$}.
\end{theorem}
\begin{proof}
See Appendix~C.
\end{proof}
Theorem~\ref{T1} shows that even in the presence of time-increasing DP-noise variances $\varrho_{t}^{i}$ ($\varsigma^{i}>0$), Algorithm 1 can still track time-varying optimal parameters with time, with the expected tracking error diminishing at a rate of $\mathcal{O}(t^{-\beta})$. This proves that Algorithm 1 is capable of preserving learning accuracy even in the presence of large DP noises.

In the following corollary, we quantify the dynamic regret of Algorithm~\ref{A1}, which measures accumulated losses~\cite{distributedonline1} of our algorithm in all $T$ iterations:
\begin{corollary}\label{c1}
	Under the conditions in the statement of Theorem~\ref{T1}, the dynamic regret of Algorithm~\ref{algorithm1} satisfies
	{\color{blue}\begin{equation}
			\sum_{t=1}^{T}\mathbb{E}\left[F_{t}(\theta_{t}^{i})\right]-\sum_{t=1}^{T}\mathbb{E}\left[F_{t}(\theta_{t}^*)\right]\leq \mathcal{O}\left(T^{1-\frac{\beta}{2}}\right),\label{C11result}
		\end{equation}
		for any $i\in[m]$.}
\end{corollary}
\begin{proof}
	According to the definition $f_{j}(\theta)\!=\!\mathbb{E}[l(\theta,\xi^{j})]$, we have
	{\color{blue}
		\begin{equation}
			\begin{aligned}
				&\mathbb{E}\left[F_{t}(\theta_{t}^{i})\right]-\mathbb{E}\left[F_{t}(\theta_{t}^*)\right]\\
				&= \frac{1}{m}\sum_{j=1}^{m}\mathbb{E}\left[\frac{1}{t+1}\sum_{k=0}^{t}l(\theta_{t}^{i},\xi_{k}^{j})-\frac{1}{t+1}\sum_{k=0}^{t}l(\theta_{t}^{*},\xi_{k}^{j})\right]\\
				&=\frac{1}{m}\sum_{j=1}^{m}(f_{j}(\theta_{t}^{i})-f_{j}(\theta_{t}^*))=\frac{1}{m}\sum_{j=1}^{m}\mathbb{E}\left[\nabla f_{j}(\varphi_{t}^{ij})^{T}(\theta_{t}^{i}-\theta_{t}^*)\right]\\
				&\leq D\mathbb{E}\left[\|\theta_{t}^{i}-\theta_{t}^{*}\|\right]\!\leq\!\mathcal{O}(t^{-\frac{\beta}{2}}),\label{C11}
			\end{aligned}
		\end{equation}
		with $\varphi_{t}^{ij}\triangleq q_{j}\theta_{t}^{i}+(1-q_{j})\theta_{t}^*$ for any $q_{j}\in(0,1)$. Here, we have used the mean value theorem in the third equality, Assumption~\ref{A1}-(iii) in the first inequality, and relationship~\eqref{T1result} in the last inequality.}
	
	By using~\eqref{C11} and the relation $\sum_{t=2}^{T}t^{-\alpha}\leq\int_{t=1}^{T}\frac{1}{x^{\alpha}}dx\leq\frac{1}{1-\alpha}T^{1-\alpha}$ valid for any $\alpha\in(0,1)$, we arrive at
	\begin{equation}
		\begin{aligned}
			&\sum_{t=1}^{T}\mathbb{E}\left[F_{t}(\theta_{t}^{i})\right]-\sum_{t=1}^{T}\mathbb{E}\left[F_{t}(\theta_{t}^*)\right]\\
			&\leq \mathcal{O}\left(T^{1-\frac{\beta}{2}}\right)+D\mathbb{E}\left[\|\theta_{1}^{i}-\theta_{1}^{*}\|\right]\leq \mathcal{O}\left(T^{1-\frac{\beta}{2}}\right),\label{C12}
		\end{aligned}
	\end{equation}
	where we have omitted the constant $\frac{2D}{2-\beta}$ in the first inequality and $\mathcal{O}(1)$ in the last inequality.
\end{proof}

Corollary~\ref{c1} proves that Algorithm~\ref{algorithm1} can achieve a sublinear dynamic regret even under LDP constraints. This result is consistent with the dynamic regret result in~\cite{distributedonline1}, which shows that the dynamic regret is bounded by the path length of an online optimization problem. In fact, under our ERM formulation in~\eqref{primalt}, the path length can be quantitatively bounded by $\mathbb{E}[\|\theta_{t+1}^{*}-\theta_{t}^{*}\|]\leq \mathcal{O}((t+1)^{-1})$, as established in Lemma~\ref{L1} (see~Eq.~\eqref{L1result}). Moreover, this upper bound has been incorporated into our convergence result in Theorem~\ref{T1} (see~Eq.~\eqref{D12} for details). Therefore, we can derive a sublinear dynamic regret in Corollary~\ref{c1} based on~Theorem~\ref{T1}.
\vspace{-0.7em}
\subsection{Tracking analysis with convex objective functions ($\mu$ in Assumption~\ref{A1} is nonnegative)}
In this section, we examine the tracking performance of Algorithm 1 for general convex objective functions.
\begin{theorem}\label{T2}
Under Assumptions~\ref{A1}-\ref{A4} with $\mu\geq0$, if~$\frac{2}{3}<\frac{1+2u}{3}<v<1$, $0<\gamma_{0}\leq\frac{1}{-3\delta_{m}}$, and $0<\lambda_{0}\leq \frac{-\delta_{2}\gamma_{0}}{2(L^2+\kappa^2+D^2)}$ hold, the expected instantaneous regret of Algorithm 1 satisfies
\begin{equation}
\mathbb{E}\left[F_{t}(\theta_{t}^{i})-F_{t}(\theta_{t}^{*})\right]\leq \mathcal{O}(t^{-\beta}),\label{T2result}
\end{equation}
for all $t>0$, where the rate $\beta$ satisfies $\beta=\frac{1-v}{2}$.
\vspace{-0.7em}
\end{theorem}
\begin{proof}
See Appendix~D.
\end{proof}
Theorem~\ref{T2} presents the expected instantaneous regret of Algorithm 1 when the objective functions are convex. However, analyzing parameter tracking errors is challenging for convex objective functions due to the possible existence of multiple optimal solutions with identical gradients. In such cases, the gradient's change does not provide sufficient information to establish an upper bound on the parameter tracking error.
\section{Extension: stepsize selection without global parameters}\label{extension}
In Theorem~\ref{T1} and Theorem~\ref{T2}, the design of the stepsize sequence $\lambda_{t}$ and the decaying sequence $\gamma_{t}$ for Algorithm~1 requires knowledge of global parameters, such as the eigenvalues of the matrix $W$, the Lipschitz constant $L$, and the strongly convex coefficient $\mu$ of the objective function. Obtaining these global parameters might be challenging for individual learners in practical distributed implementations. Therefore, in this section, we discuss the tracking performance of Algorithm 1 when the stepsize and decaying sequences are designed without any knowledge of global parameters.

More specifically, we establish the following theorems for strongly convex and convex objective functions, respectively.
\begin{theorem}\label{T3}
Under Assumptions~\ref{A1}-\ref{A4} with $\mu>0$, if $\frac{1}{2}<u<v<1$ holds, then for any positive constants $\lambda_{0}$ and $\gamma_{0}$, the expected tracking error of Algorithm~\ref{algorithm1} satisfies
\begin{equation}
\mathbb{E}[\|\theta_{t+1}^{i}-\theta_{t+1}^*\|^2]\leq\mathcal{O}\left((t-t_{0})^{-\beta}\right),\label{T3result}
\end{equation}	
for all $t>t_{0}$, where the rate $\beta$ satisfies $\beta=\min\{1-v,2u-2\varsigma-1\}$ {\color{blue} with $\varsigma\triangleq\min_{i\in[m]}\{\varsigma^{i}\}$} and the positive constant $t_{0}$ is given by
\begin{equation*}
 t_{0}=\left\lceil\max\left\{\left(-3\delta_{m}\gamma_{0}\right)^{\frac{1}{u}}-1,\left(\frac{(\mu^2+8L^2)\lambda_{0}}{-\delta_{2}\mu\gamma_{0}}\right)^{\frac{1}{v-u}}-1\right\}\right\rceil.
\end{equation*}
\end{theorem}
\begin{proof}
See Appendix E.
\vspace{-0.5em}
\end{proof}
\begin{theorem}\label{T4}
Under Assumptions~\ref{A1}-\ref{A4} with $\mu\geq0$, if $\frac{2}{3}<\frac{2u+1}{3}<v<1$ holds, then for any positive constants $\lambda_{0}$ and $\gamma_{0}$, the expected instantaneous regret of Algorithm~\ref{algorithm1} satisfies
\begin{equation}
\begin{aligned}
\mathbb{E}\left[F_{t}(\theta_{t}^{i})-F_{t}(\theta_{t}^{*})\right]\leq \mathcal{O}\left(t^{-\beta}\right), \label{T4result}
\end{aligned}
\end{equation}
for all $t> t'_{0}$, where the rate $\beta$ satisfies $\beta=\frac{1-v}{2}$ and the positive constant $t'_{0}$ is given by
\begin{equation*}
\begin{aligned}	
t'_{0}&=\left\lceil\max\left\{\left(-3\delta_{m}\gamma_{0}\right)^{\frac{1}{u}}-1,\right.\right.\\
&\left.\left.\quad\left(\frac{2(L^2+\kappa^2+D^2)\lambda_{0}}{-\delta_{2}\gamma_{0}}\right)^{\frac{2}{3v-2u-1}}-1\right\}\right\rceil.
\end{aligned}
\end{equation*}
\end{theorem}
\begin{proof}
See Appendix F.
\vspace{-0.5em}
\end{proof}

The compactness of the parameter set $\Theta$ in Algorithm~\ref{algorithm1} ensures that both the expected tracking error and the expected instantaneous regret are bounded before time instant $t_{0}$ in Theorem~\ref{T3} (or $t'_{0}$ in Theorem~\ref{T4}).
\begin{remark}
The convergence results in Theorems~\ref{T1} and~\ref{T2} need global information, such as the eigenvalues $\delta_{2}$ and $\delta_{m}$ of the matrix $W$, the Lipschitz constant $L$, and the strongly convex coefficient $\mu$, to determine the values of $\lambda_{0}$ and $\gamma_{0}$. To the contrary, the results in Theorems~\ref{T3} and~\ref{T4} hold for any positive constants $\lambda_{0}$ and $\gamma_{0}$, and, hence, are applicable even when global information, such as the eigenvalues of the matrix $W$, the Lipschitz constant $L$, and the strongly convex coefficient $\mu$, are inaccessible.
\end{remark}
\begin{remark}
The decaying sequence $\gamma_{t}$ leads to a decaying coupling strength. However, we prove in Theorems~\ref{T1} through~\ref{T4} that this decaying coupling strength is still sufficient to ensure that all learners converge to the global optimal solution. Of course, the decaying coupling strength will reduce the convergence rate. We use the convergence result in Theorem~\ref{T2} as an example to illustrate this tradeoff. It is clear that the convergence rate $\mathcal{O}(t^{-\frac{1-v}{2}})$ in Theorem~\ref{T2} decreases with an increase in the decaying rate $v$ of the stepsize $\lambda_{t}$. Given the condition $\frac{2}{3}<\frac{1+2u}{3}<v<1$ presented in the statement of Theorem~\ref{T2}, we can see that an increase in the parameter $u$ (corresponding to a faster decaying sequence $\gamma_{t}$) corresponds to an increase in the parameter $v$, resulting in a decreased convergence rate $\mathcal{O}(t^{-\frac{1-v}{2}})$ from Theorem~\ref{T2}.
\end{remark}
\section{Local differential-privacy analysis}\label{privacy}
In this section, we prove that besides accurate convergence, Algorithm~\ref{A1} can simultaneously ensure rigorous $\epsilon_{i}$-LDP for each learner, with the cumulative privacy budget guaranteed to be finite even when the number of iterations $T$ tends to infinity. To this end, we first provide a definition for the sensitivity of Learner $i$'s implementation $\mathcal{A}_i$ of Algorithm~\ref{A1}:
\begin{definition}\label{Definition5}
(Sensitivity) The sensitivity of Learner $i$'s implementation $\mathcal{A}_{i}$ at each time instant $t$ is defined as
\begin{equation}
	\vspace{-0.2em}
\Delta_{t}^{i}=\max_{\text{Adj}(\mathcal{D}_{t}^{i},{\mathcal{D}_{t}^{i}}')}\|\mathcal{A}_{i}(\mathcal{D}_{t}^{i},\theta_{t}^{-i})-\mathcal{A}_{i}({\mathcal{D}_{t}^{i}}',\theta_{t}^{-i})\|_{1},\label{sensitive}
\end{equation}
where $\mathcal{D}_{t}^{i}$ represents Learner $i$'s dataset and $\theta_{t}^{-i}$ represents all messages received by Learner $i$ at time instant $t$. 
\end{definition}
With the defined sensitivity, we have the following lemma:
\begin{lemma}\label{L2}
{\color{blue}For any given $T\in{\mathbb{N}^{+}}$ or $T=\infty$}, if Learner $i$ injects to each of its transmitted messages at each time $t\in\{1,\cdots, T\}$ a noise vector $\zeta_{t}^{i}$ consisting of $n$ independent Laplace noises with parameter $\varrho_{t}^{i}$, then Learner $i$'s implementation $\mathcal{A}_{i}$ is $\epsilon_{i}$ locally differentiable private with the cumulative privacy budget from time $t=1$ to $t=T$ upper bounded by $\sum_{t=1}^{T}\frac{\Delta_{t}^{i}}{\varrho_{t}^{i}}$.
\end{lemma}
\begin{proof}
See Appendix G.
\end{proof}
For our privacy analysis, we also utilize the ensuing result:
\begin{lemma}\label{L3}
(\!\!\cite{L3}) Let $\{v_{t}\}$ denote a nonnegative sequence, and $\{\alpha_{t}\}$ and $\{\beta_{t}\}$ be positive non-increasing sequences satisfying $\sum_{t=0}^{\infty}\alpha_{t}=\infty$, $\lim_{t\rightarrow \infty}\alpha_{t}=0$, and $\lim_{t\rightarrow \infty}\frac{\beta_{t}}{\alpha_{t}}=0$. If there exists a $T\geq 0$ such that $v_{t+1}\leq (1-\alpha_{t})v_{t}+\beta_{t}$ holds for all $t\geq T$, and then we always have $v_{t}\leq c\frac{\beta_{t}}{\alpha_{t}}$ for all $t$, where $c$ is some positive constant.
\vspace{-0.5em}
\end{lemma}

For the convenience of privacy analysis, we represent the different data points between two adjacent datasets $\mathcal{D}_{t}^{i}$ and ${\mathcal{D}_{t}^{i}}'$ as $k$-th one, i.e., $\xi_{k}^{i}$ in $\mathcal{D}_{t}^{i}$ and ${\xi_{k}^{i}}'$ in ${\mathcal{D}_{t}^{i}}'$, without loss of generality. We further denote $\theta_{t}^{i}$ and ${\theta_{t}^{i}}'$ as the parameters generated by Algorithm 1 based on $\mathcal{D}_{t}^{i}$ and ${\mathcal{D}_{t}^{i}}'$, respectively. We also use the following assumption, which is standard in existing DP analysis for distributed optimization/learning (see e.g.,~\cite{offlineDP5}):

\begin{assumption}\label{A5}
For any data $\xi$ and $\xi'$, there exists some constant $C$ such that $\text{sup}_{\theta\in\Theta}\|\nabla l(\theta,\xi)-\nabla l(\theta,\xi')\|_{2}\!\leq\!C$ holds.
\end{assumption}
\begin{remark}
Assumption~\ref{A5} is standard for privacy analysis~\cite{offlineDP5}. It relaxes the bounded-gradient assumption in~\cite{explicity2,explicity3,explicity4,PDOP,offlineDP1,offlineDP3} because if one has $\|\nabla l(\theta,\xi)\|_{2}\leq C$, then one always has $\|\nabla l(\theta,\xi)-\nabla l(\theta,\xi')\|_{2} \leq 2C$. In general, Assumption~\ref{A5} can be satisfied under our problem setting since the optimization variable is restricted in a compact set $\Theta$. For example, under the commonly used loss function $l(\theta,\xi)=\theta^{T}Q\theta+\xi^{T}\theta$ for given data $\xi$ and $Q>0$, we can easily obtain $\|\nabla l(\theta,\xi)-\nabla l(\theta,\xi')\|_{2} \leq\|\xi-\xi'\|_{2}$, and, hence, the boundedness of gradient differences in Assumption~\ref{A5}. In addition, in many machine learning applications, gradient clipping is used to make the norm of the gradient vector be at most $C$~\cite{clipping1,clipping2}. In this case, we can easily obtain the upper bound in Assumption~\ref{A5} by using the inequality $\|\nabla l(\theta,\xi)-\nabla l(\theta,\xi')\|_{2} \leq 2\|\nabla l(\theta,\xi)\|_{2}\leq 2C$.
\end{remark}

\begin{theorem}\label{T5}
Under Assumptions~\ref{A1}-\ref{A5}, if nonnegative sequences $\lambda_{t}$ and $\gamma_{t}$ satisfy the conditions in the statement of Theorem~\ref{T1}, and each element of $\zeta_{t}^{i}$ independently follows a Laplace distribution $\text{Lap}(\varrho_{t}^{i})$ satisfying Assumption 4, then the tracking error of Algorithm 1 will converge in mean square to zero. Furthermore,

(i) For any finite number of iterations $T$, under Algorithm~\ref{algorithm1}, Learner $i$ is ensured to be $\epsilon_{i}$ locally differentially private with the cumulative privacy budget bounded by $ \sum_{t=1}^{T}\frac{\sqrt{2n}C\tau_{t}}{\sigma^{i}(t+1)^{\varsigma^{i}}}$. Here, $C$ is given in the statement of Assumption~\ref{A5} and parameter $\tau_{t}$ is given by $\tau_{t}=\sum_{p=1}^{t-1}\left(\prod_{q=p}^{t-1}(1-\bar{w}\gamma_{q}+\lambda_{q}L)\right)\lambda_{p-1}+\lambda_{t-1}$.

(ii) The cumulative privacy budget is finite for $T\rightarrow \infty$.
\end{theorem}
\begin{proof}
Since the Laplace DP noise satisfies Assumption~\ref{A4}, the tracking result follows naturally from Theorem~\ref{T1}.

(i) To prove the statements on privacy, we first analyze the sensitivity of Learner $i$ under Algorithm~\ref{algorithm1}. 

According to the definition of sensitivity in~\eqref{sensitive}, we have $\theta_{t}^{j}+\zeta_{t}^{j}={\theta_{t}^{j}}'+{\zeta_{t}^{j}}'$ for all $t\geq 0$ and $j\in{\mathcal{N}_{i}}$. Since we assume that only the $k$-th data point is different between $\mathcal{D}_{t}^{i}$ and ${\mathcal{D}_{t}^{i}}'$, when $t<k$, we have $\theta_{t}^{i}={\theta_{t}^{i}}'$. However, when $t\geq k$, since the difference in loss functions kicks in at time $k$, i.e., $l(\theta,\xi_{k}^{i})\neq l(\theta,{\xi_{k}^{i}}')$, we have $\theta_{t}^{i}\neq{\theta_{t}^{i}}'$. Hence, for Learner $i$'s implementation of Algorithm~\ref{algorithm1}, we use the projected inequality to obtain
\vspace{-0.2em} 
\begin{equation}
\begin{aligned}
&\|\theta_{t+1}^{i}-({\theta^{i}_{t+1}})'\|_{2}=\|\text{Pro}_{\Theta}(\hat{\theta}_{t+1}^{i})-\text{Pro}_{\Theta}((\hat{\theta}^{i}_{t+1})')\|_{2}\\
&\!\leq\! \|\hat{\theta}_{t+1}^{i}-(\hat{\theta}^{i}_{t+1})'\|_{2}\leq\big{\|}(\!1\!+\!w_{ii}\gamma_{t})(\theta_{t}^{i}-{\theta_{t}^{i}}')\\
&\quad-\frac{\lambda_{t}}{t+1}\sum_{p=k}^{t}(\nabla l(\theta_{t}^{i},\xi_{p}^{i})-\nabla l({\theta_{t}^{i}}',{\xi_{p}^{i}}'))\big{\|}_{2},
\end{aligned}
\vspace{-0.3em}
\end{equation}
for all $t\geq k$ and any $k\geq 0$, where we have used the definition $w_{ii}=-\sum_{j\in{\mathcal{N}_{i}}}w_{ij}$. Letting $\bar{w}\triangleq\min\{|w_{ii}|\},~i\in[m]$ and $\Phi_{t}^{i}\triangleq \|\theta_{t}^{i}-{\theta_{t}^{i}}'\|_{2}$, we obtain
\begin{flalign}
&\Phi_{t+1}^{i}\leq\! (1-\bar{w}\gamma_{t})\Phi_{t}^{i}+\frac{\lambda_{t}}{t+1}\sum_{p=k}^{t}\|\nabla l(\theta_{t}^{i},\xi_{p}^{i})-\nabla l({\theta_{t}^{i}}',{\xi_{p}^{i}}')\|_{2}\nonumber\\
&\leq\! (1\!-\!\bar{w}\gamma_{t})\Phi_{t}^{i}\!+\!\frac{\lambda_{t}}{t+1}\sum_{p=0}^{t}\|\nabla l(\theta_{t}^{i},\xi_{p}^{i})\!-\!\nabla l({\theta_{t}^{i}}',{\xi_{p}^{i}}')\|_{2},\label{5T2}
\end{flalign}
where we have used $\sum_{p=0}^{k-1}\nabla l(\theta_{t}^{i},\xi_{p}^{i})=\sum_{p=0}^{k-1}\nabla l({\theta_{t}^{i}}',{\xi_{p}^{i}}')$ in the second inequality. Since $\|\nabla l(\theta_{t}^{i},\xi_{p}^{i})-\nabla l({\theta_{t}^{i}}',{\xi_{p}^{i}}')\|_{2}=\|\nabla l(\theta_{t}^{i},\xi_{p}^{i})-\nabla l(\theta_{t}^{i},{\xi_{p}^{i}}')+\nabla l(\theta_{t}^{i},{\xi_{p}^{i}}')-\nabla l({\theta_{t}^{i}}',{\xi_{p}^{i}}')\|_{2}\leq C+L\Phi_{t}^{i}$ holds,
the inequality~\eqref{5T2} can be rewritten as
\begin{equation}
\Phi_{t+1}^{i}\leq (1-\bar{w}\gamma_{t}+\lambda_{t}L)\Phi_{t}^{i}+\lambda_{t}C.\label{5T21}
\end{equation}

By iterating~\eqref{5T21} from $0$ to $t$ and using the relationship $\|\theta\|_{1}\leq \sqrt{n}\|\theta\|_{2}$ valid for any $\theta\in \mathbb{R}^{n}$, we obtain
\begin{equation}
\Delta_{t}^{i}\leq \sqrt{n}C\left(\sum_{p=1}^{t-1}\left(\prod_{q=p}^{t-1}(1-\bar{\omega}\gamma_{q}+\lambda_{q}L)\right)\lambda_{p-1}+\lambda_{t-1}\right).\nonumber
\end{equation}

Therefore, for Learner $i$, the cumulative privacy budget for any finite $T$ iterations is bounded by
\begin{equation}
\sum_{t=1}^{T}\frac{\Delta_{t}^{i}}{\varrho_{t}^{i}}\leq  \sum_{t=1}^{T}\frac{\sqrt{2n}C\tau_{t}}{\sigma^{i}(t+1)^{\varsigma^{i}}},\label{T5result1}
\end{equation}
where $\tau_{t}$ is defined in the theorem statement.

(ii) Based on~\eqref{5T2} and $\xi_{p}^{i}={\xi_{p}^{i}}'$ for $p\neq k$, we have
\begin{flalign}
\Phi_{t+1}^{i}&\leq (1-\bar{w}\gamma_{t})\Phi_{t}^{i}+\frac{\lambda_{t}}{t+1}\|\nabla l(\theta_{t}^{i},\xi_{k}^{i})-\nabla l(\theta_{t}^{i},{\xi_{k}^{i}}')\|_{2}\nonumber\\
&\quad+\frac{\lambda_{t}}{t+1}\|\nabla l(\theta_{t}^{i},{\xi_{k}^{i}}')-\nabla l({\theta_{t}^{i}}',{\xi_{k}^{i}}')\|_{2}\nonumber\\
&\quad+\frac{\lambda_{t}}{t+1}\sum_{p=0,~p\neq k}^{t}\|\nabla l(\theta_{t}^{i},\xi_{p}^{i})-\nabla l({\theta_{t}^{i}}',\xi_{p}^{i})\|_{2},\label{5T3}
\end{flalign}
for all $t\geq k$ and any $k \geq 0$. By using Assumption~\ref{A3}-(iii) and Assumption~\ref{A5}, we can rewrite~\eqref{5T3} as follows:
\begin{equation}
\Phi_{t+1}^{i}\leq\left(1-\bar{w}\gamma_{t}+\frac{L\lambda_{t}(t+1)}{t+1}\right)\Phi_{t}^{i}+\frac{\lambda_{t}C}{t+1}. \label{T5P4}
\end{equation}

Recalling the definitions $\gamma_{t}=\frac{\gamma_{0}}{(t+1)^{u}}$ and $\lambda_{t}=\frac{\lambda_{0}}{(t+1)^{v}}$ with $v>u$ from the statement of Theorem~\ref{T1}, there must exist a $T_{0}>0$ and some constant $C_{1}>0$ such that
\begin{flalign}
\bar{w}\gamma_{t}-L\lambda_{t}\!=\!\frac{\bar{w}\gamma_{0}}{(t+1)^{u}}-\frac{L\lambda_{0}}{(t+1)^{v}}\geq \frac{C_{1}}{(t+1)^{u}},\label{T524}
\end{flalign}
holds for all $t\geq T_{0}$. Combining~\eqref{T5P4} and~\eqref{T524} yields $\Phi_{t+1}^{i}\leq\left(1-\frac{C_{1}}{(t+1)^{u}}\right)\Phi_{t}^{i}+\frac{\lambda_{0}C}{(t+1)^{1+v}}$
for all $t\geq T_{0}$. Using Lemma~\ref{L3} yields for some constant $C_{2}$, we have $\Phi_{t}^{i}=\|\theta_{t}^{i}-{\theta_{t}^{i}}'\|_{2}\leq C_{2}\frac{\lambda_{0}C}{C_{1}(t+1)^{1+v-u}}$ for all $t>0$.

Based on the relationship $\|x\|_{1}\leq \sqrt{n}\|x\|_{2}$ valid for any $x\in\mathbb{R}^{n}$, we can prove that the sensitivity $\Delta_{t}^{i}$ satisfies
\begin{equation}
\Delta_{t}^{i}\leq \sqrt{n}	\Phi_{t}^{i}\leq \frac{\sqrt{n}\lambda_{0}CC_{2}}{C_{1}(t+1)^{1+v-u}},\label{sensitivt}
\end{equation}
for all $t>0$. Recalling the Laplace-noise parameter $\varrho_{t}^{i}=\frac{\sigma^{i}(t+1)^{\varsigma^{i}}}{\sqrt{2}}$ from the statement of Assumption 4, we have the cumulative privacy budget bounded by
\begin{equation}
\sum_{t=1}^{\infty}\frac{\Delta_{t}^{i}}{\varrho_{t}^{i}}\leq \sum_{t=1}^{\infty}\frac{\sqrt{2n}\lambda_{0}CC_{2}}{C_{1}\sigma^{i}(t+1)^{1+v-u+\varsigma^{i}}},\label{T5result2}
\end{equation}
according to Lemma~\ref{L2} when $T\rightarrow\infty$. Since $v-u+\varsigma^{i}>0$, the cumulative privacy budget is finite when $T\rightarrow\infty$.
\end{proof}

Theorem~\ref{T5}-(i) implies that for any given cumulative privacy budget $\epsilon_{i}$, Learner $i$'s implementation $\mathcal{A}_{i}$ of Algorithm~\ref{algorithm1} is $\epsilon_{i}$-locally differentially private when the noise parameter satisfies $\sigma^{i}= \sum_{t=1}^{T}\frac{\sqrt{2n}C\tau_{t}}{\epsilon_{i}(t+1)^{\varsigma^{i}}}$ with $\tau_{t}$ defined in Theorem~\ref{T5}. Therefore, each learner can choose its desired privacy budget based on its own practical and \textbf{personalized} need. This differs from existing centralized DP frameworks used in differentially private distributed optimization/learning approaches~\cite{offlineDP1,offlineDP3,offlineDP4,offlineDP5,tailoring,nonconvexwang,wangcentrial}, which, in the absence of a data aggregator/curator, require participating learners to trust each other and cooperatively determine the amount of noises needed to achieve a \textbf{universal} global privacy budget $\epsilon$.

Theorem~\ref{T5}-(ii) proves that in addition to accurate convergence, Algorithm~\ref{algorithm1} can ensure a finite cumulative privacy budget even when the number of iterations tends to infinity. The key reason for our approach to achieve rigorous LDP is the judicious design of the decaying factor $\gamma_{t}$, gradient computation strategy, and the stepsize $\lambda_{t}$. These designs can ensure a fast diminishing sensitivity (see Eq.~\eqref{sensitivt}), which, combined with increasing DP-noise variances, ensures a finite cumulative privacy budget even in the infinite time horizon (see Eq.~\eqref{T5result2}).
\begin{remark}
	\vspace{-0.2em}
	Theorem~\ref{T5} proves that our algorithm can circumvent the tradeoff between privacy and learning accuracy.
	However, this does not mean that our algorithm achieves privacy protection for free. In fact, resolving the tradeoff
	between privacy and learning accuracy comes at the expense
	of sacrificing the convergence rate. Specifically, the rate $\mathcal{O}(t^{-\beta})$ in Theorem~\ref{T1} is determined by the decaying parameter $u$ of the sequence $\gamma_{t}$, the decaying parameter $v$ of the stepsize sequence $\lambda_{t}$, and the noise parameter $\max_{i\in[m]}\{\varsigma^{i}\}$. The condition $\max_{i\in[m]}\{\varsigma^{i}\}\!<\!u\!<\!v<1$ indicates that an increase in noise parameter $\varsigma^{i}$ (corresponding to stronger privacy protection) necessitates an increase in decaying parameter $v$, resulting in a slower convergence rate $\mathcal{O}(t^{-\beta})$ from Theorem~\ref{T1}.
	\vspace{-0.2em}
\end{remark}
\begin{remark}
	The parameters $u$, $v$, and $\varsigma^{i}$ are crucial for our algorithm's performance. More specifically, according to the convergence results in Theorems~\ref{T1} through~\ref{T4}, a smaller $v$, a larger $u$, and a smaller $\varsigma^{i}$ lead to a faster convergence rate. Therefore, for applications requiring fast convergence, a small $v$, a large $u$, and a small $\varsigma^{i}$ are preferable. In addition, according to~\eqref{T5result2} in our privacy analysis, a smaller $v$, a larger $u$, and a smaller $\varsigma^{i}$ result in weaker privacy protection. Hence, for privacy-sensitive applications, a large $v$, a small $u$, and a large $\varsigma^{i}$ are preferable. Therefore, there is a tradeoff between convergence rate and privacy. In applications, we can select these parameters based on practical needs.
	\vspace{-0.5em}
\end{remark}
\begin{table*}[h]
	\centering
	\caption{The number of iterations to achieve $\|\frac{1}{m}\sum_{i=1}^{m}\theta_{t}^{i}-\theta_{t}^{*}\|_{2}\leq 1$ under different cumulative privacy budgets}
	\label{table2}
	\begin{threeparttable}
		\begin{tabular}{|l|c|c|c|c|c|c|c|c|c|c|c|}
			\hline
			Noise level\tnote{a}  & $\times 1$ & $\times 1.5$ & $\times 2$ & $\times 2.5$ & $\times 3$ & $\times 3.5$ & $\times 4$  & $\times 4.5$  & $\times 5$  & $\times 5.5$  & $\times 6$\\ \hline
			Cumulative privacy budget & 23.34 & 16.59 & 12.65 & 11.97 & 11.54 & 10.50 & 9.64 & 8.79 & 7.98 & 7.47 & 7.03  \\ \hline
			Iteration number  & 8 & 11 & 12 & 34 & 127 & 269 & 575 & 934 & 1119 &  2292 & 4999 \\ \hline
		\end{tabular}
		\begin{tablenotes}
			\footnotesize
			\item[a] Considering the Laplace noise $\text{Lap}(0.1(t+1)^{0.1+0.01^{i}})$ as the base level.
		\end{tablenotes}
	\end{threeparttable}
	\vspace{-0.7em}
\end{table*}
\begin{figure*}[h]
	\subfigure[Tracking error]{
		\begin{minipage}[t]{0.32\linewidth}
			\centering
			\includegraphics[width=1.0\linewidth]{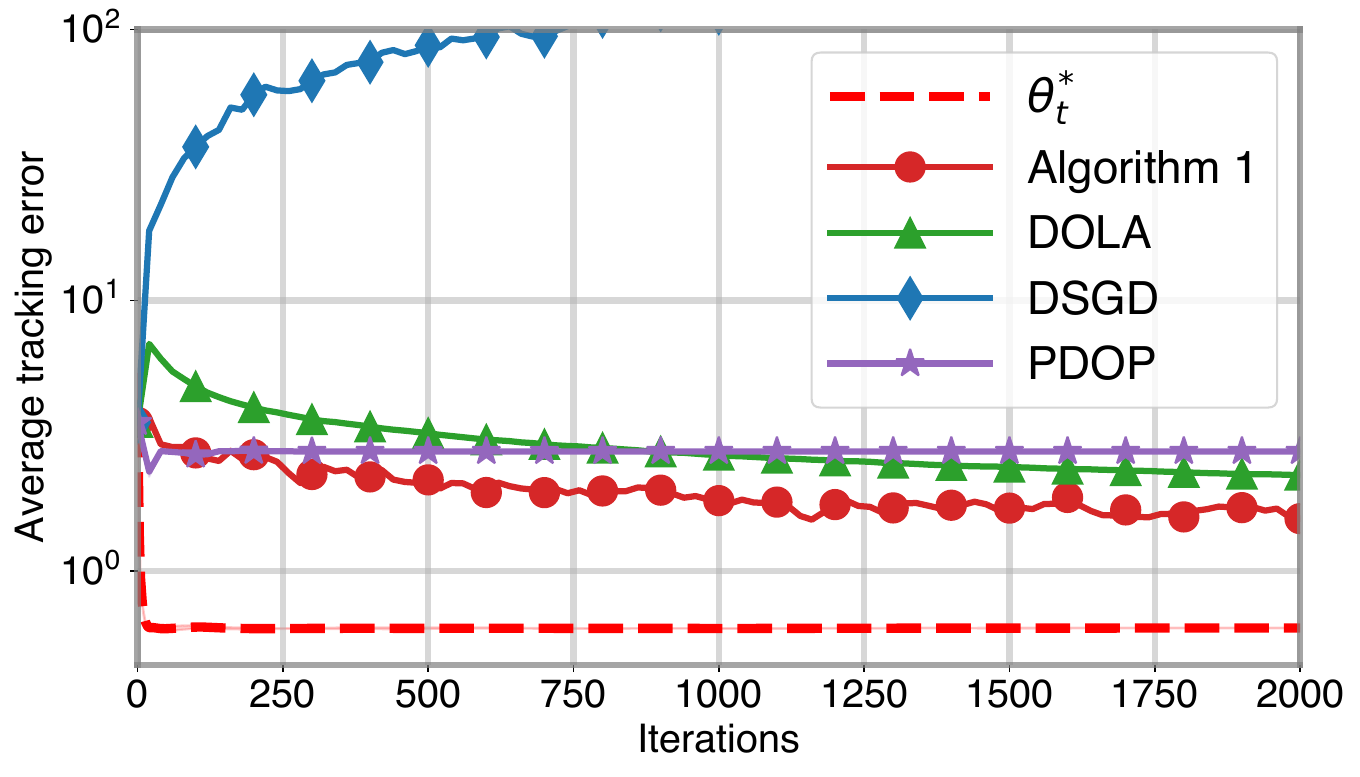}
		\end{minipage}
	}%
	\subfigure[Loss]{
		\begin{minipage}[t]{0.32\linewidth}
			\centering
			\includegraphics[width=1.0\linewidth]{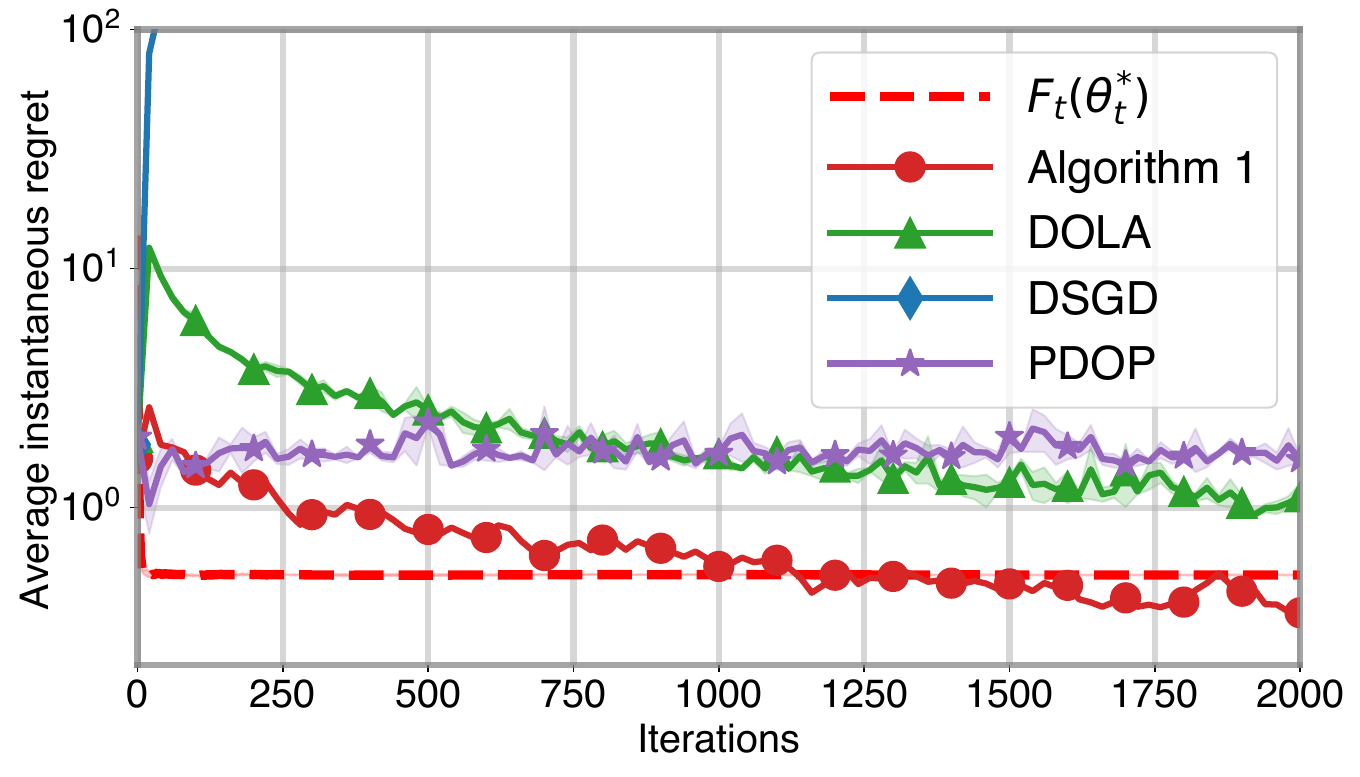}
		\end{minipage}
	}%
	\subfigure[Cumulative privacy budget]{
		\begin{minipage}[t]{0.32\linewidth}
			\centering
			\includegraphics[width=1.0\linewidth]{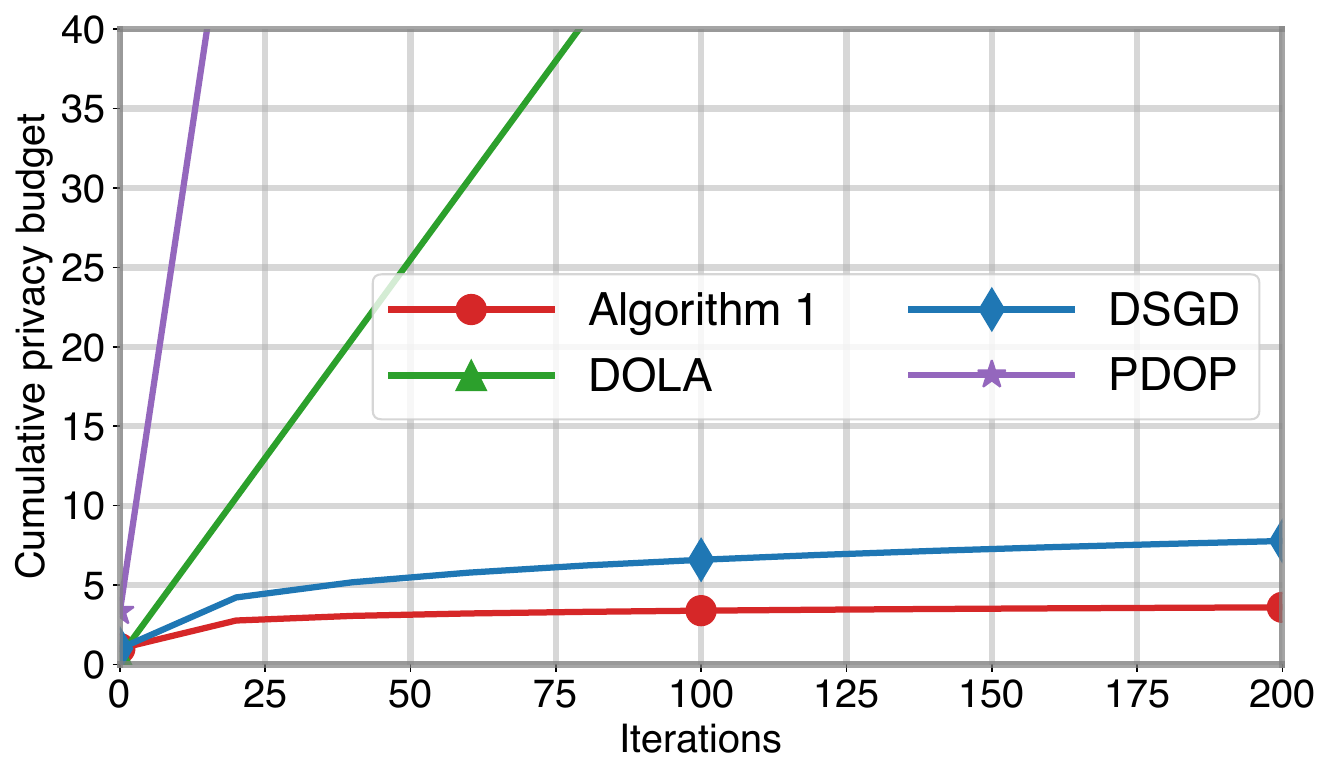}
		\end{minipage}
	}%
	\caption{Comparison of online logistic regression results by using the ``mushrooms" dataset.}
	\label{mushrooms}
	\vspace{-1em}
\end{figure*}
\begin{remark}
	\vspace{-0.7em}
	Our approach can ensure both DP and \textbf{mean square convergence} of the optimization variable to the optimal solution (the variance of the distance between the optimization variable and the optimal solution converges to zero). It is much stronger than~\cite{datalabel} that only characterizes the convergence of the \textbf{expected} value of the optimization variable to the optimal solution in the presence of DP noises (which cannot exclude the possibility that the optimization error can have an arbitrarily large variance). In addition,~\cite{datalabel} only ensures DP of the data (sample) label but does not consider the privacy of the content of data. In contrast, we enable DP for both the label and the content of data.
	\vspace{-0.2em}
\end{remark}
\begin{figure*}[h]
	\subfigure[Training accuracy]{
		\begin{minipage}[t]{0.32\linewidth}
			\centering
			\includegraphics[width=1.0\linewidth]{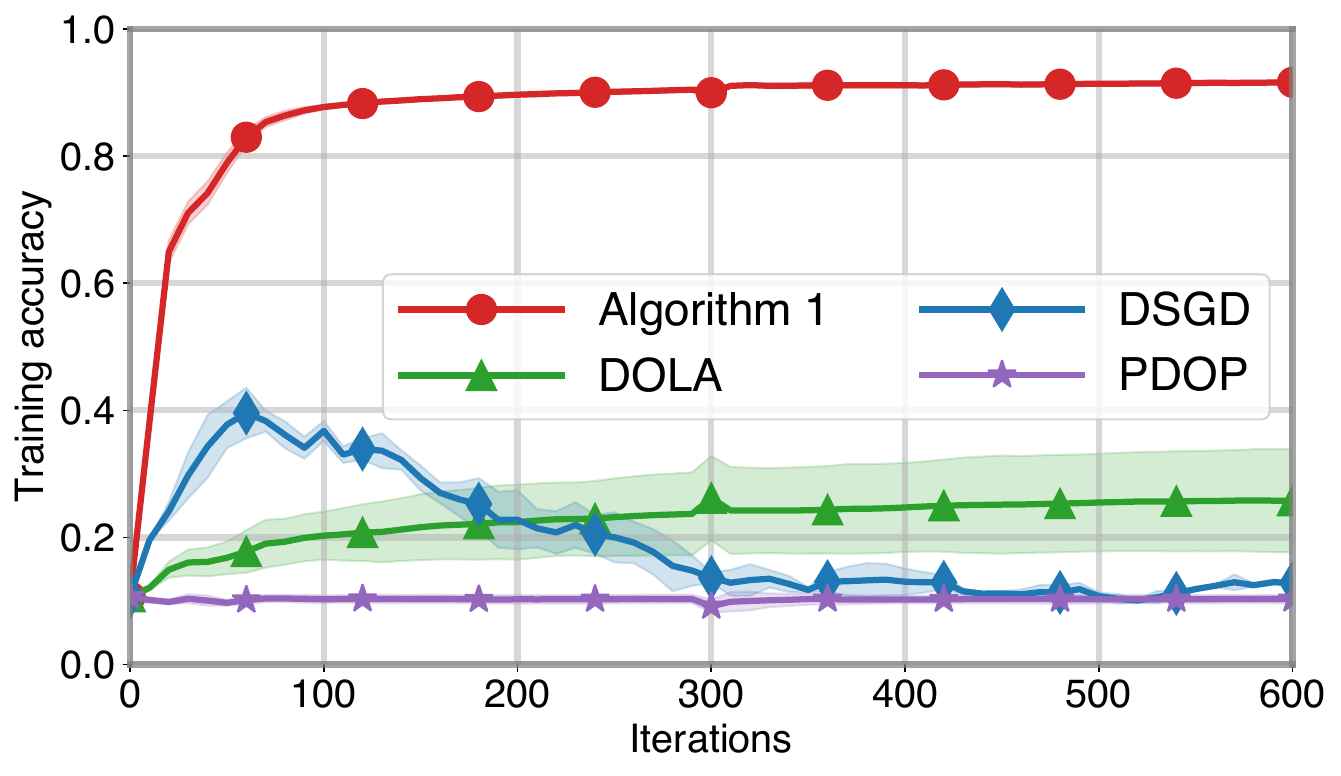}
		\end{minipage}
	}%
	\subfigure[Test accuracy]{
		\begin{minipage}[t]{0.32\linewidth}
			\centering
			\includegraphics[width=1.0\linewidth]{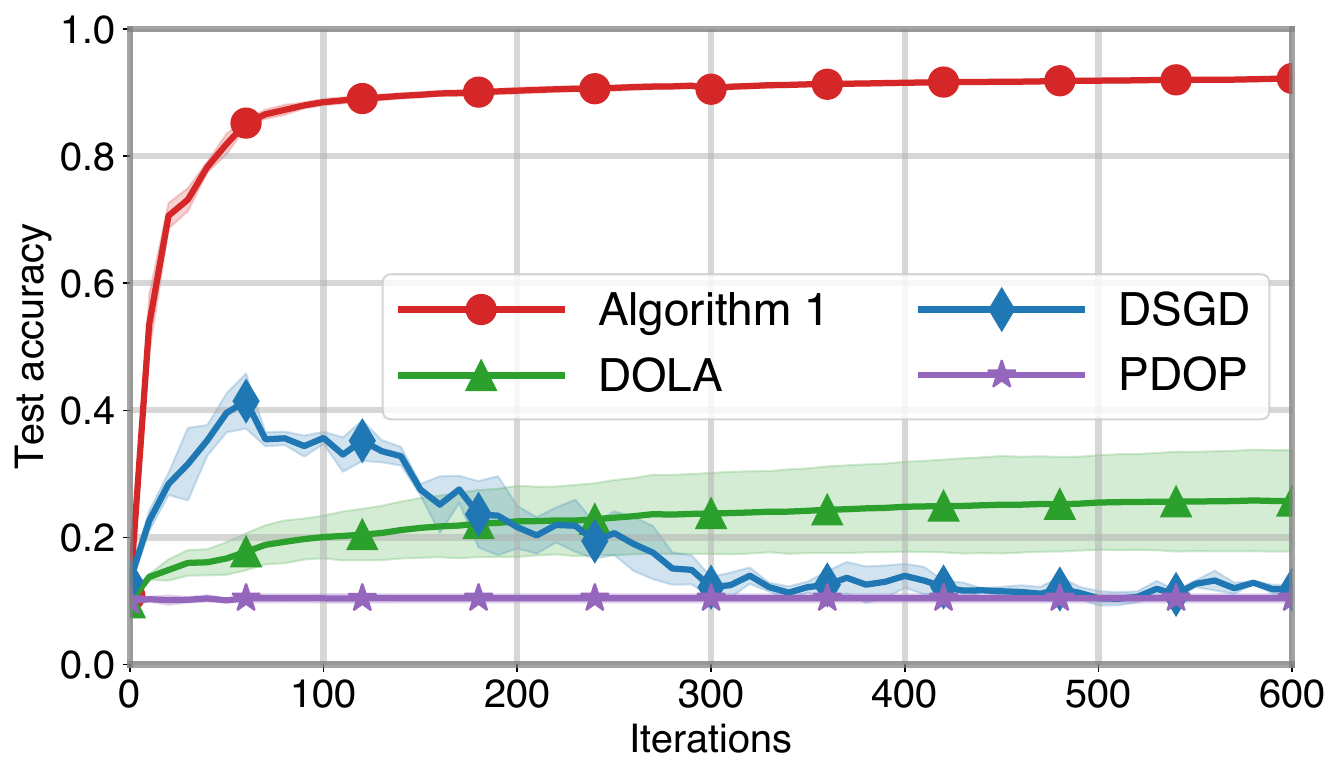}
		\end{minipage}
	}%
	\subfigure[Cumulative privacy budget]{
		\begin{minipage}[t]{0.32\linewidth}
			\centering
			\includegraphics[width=1.0\linewidth]{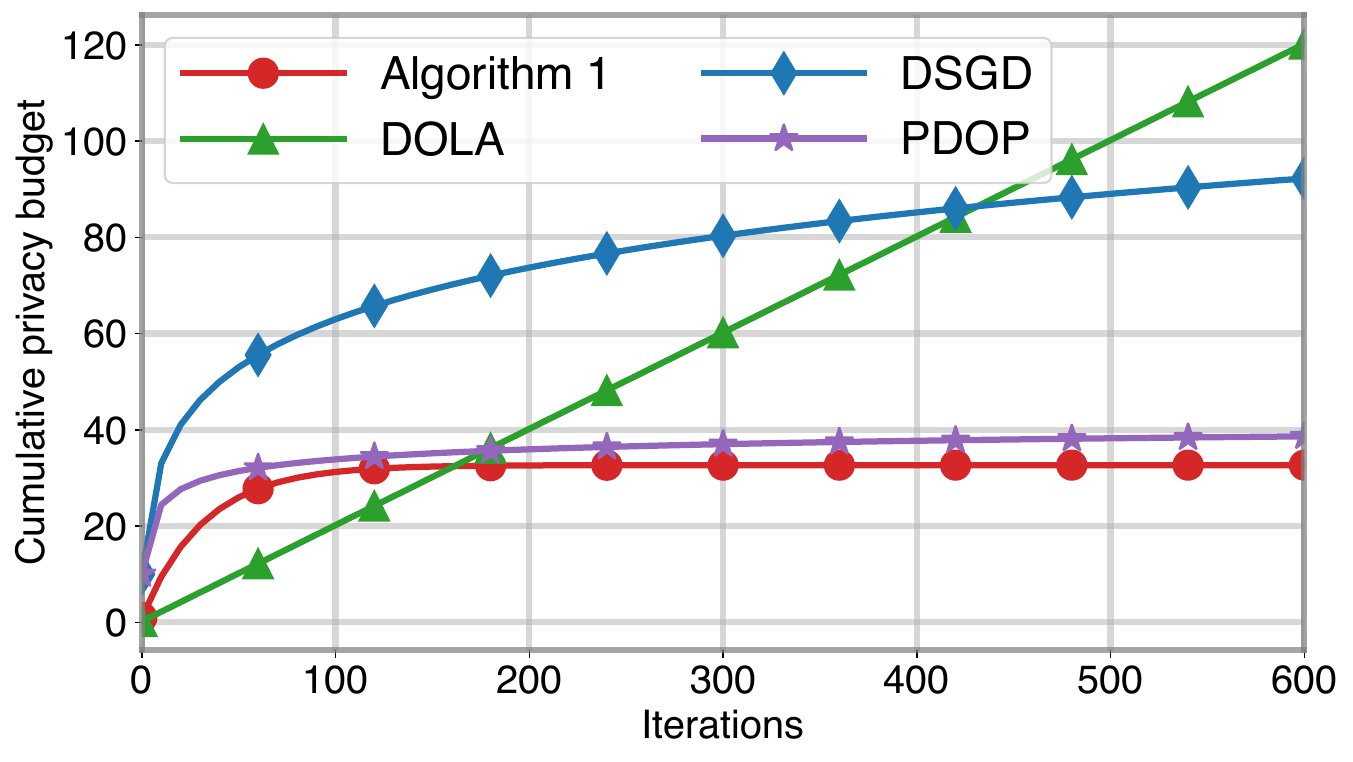}
		\end{minipage}
	}%
	\caption{Comparison of neural network training results by using the ``MNIST" dataset.}
	\label{mnist}
	\vspace{-0.8em}
\end{figure*}
\begin{figure*}[h]
	\subfigure[Training accuracy]{
		\begin{minipage}[t]{0.32\linewidth}
			\centering
			\includegraphics[width=1.0\linewidth]{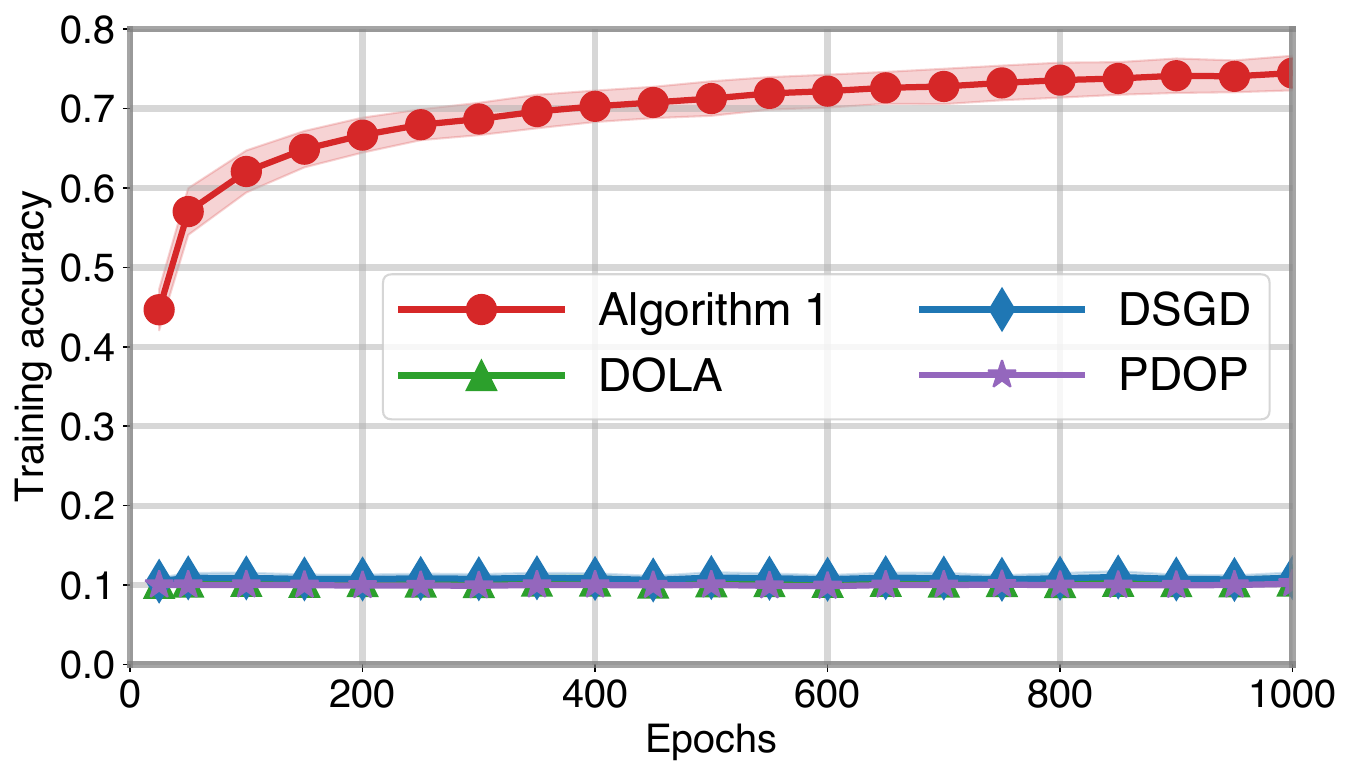}
		\end{minipage}
	}%
	\subfigure[Test accuracy]{
		\begin{minipage}[t]{0.32\linewidth}
			\centering
			\includegraphics[width=1.0\linewidth]{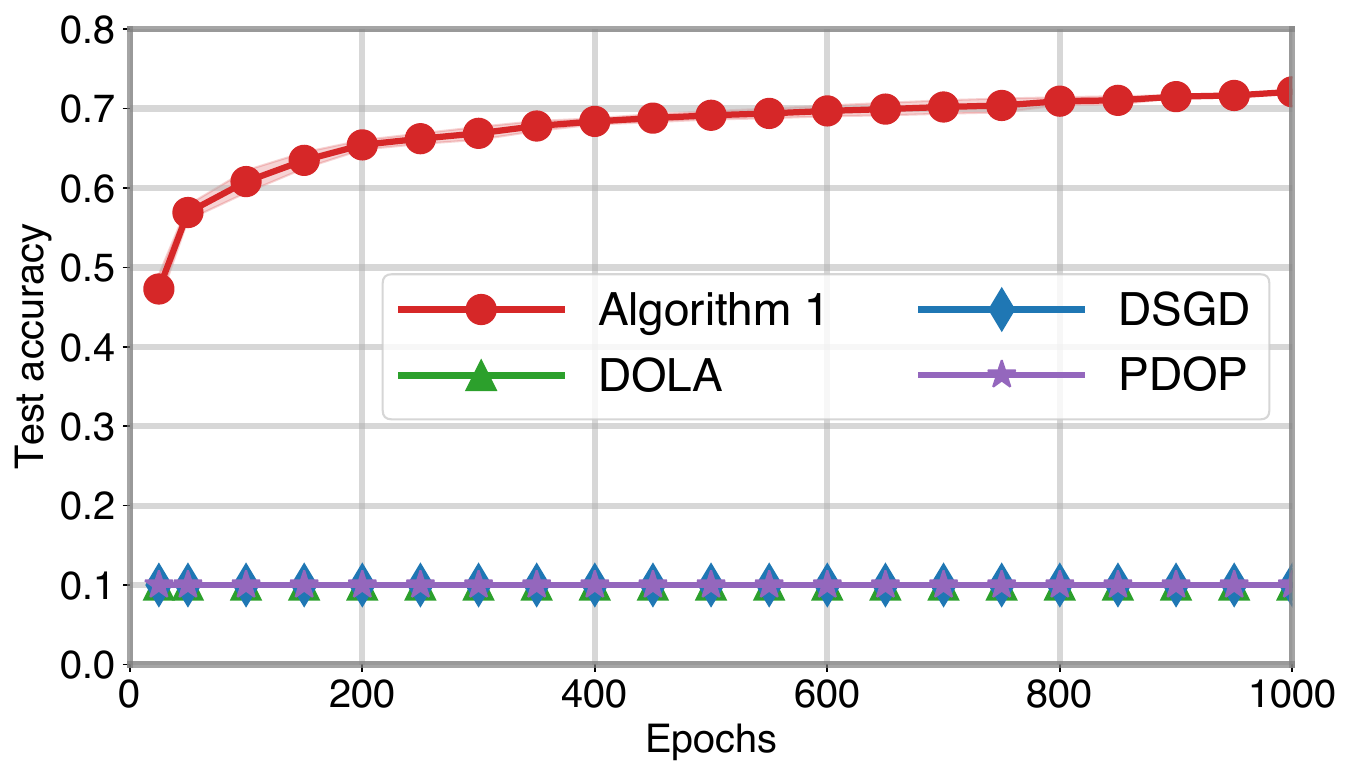}
		\end{minipage}
	}%
	\subfigure[Cumulative privacy budget]{
		\begin{minipage}[t]{0.32\linewidth}
			\centering
			\includegraphics[width=1.0\linewidth]{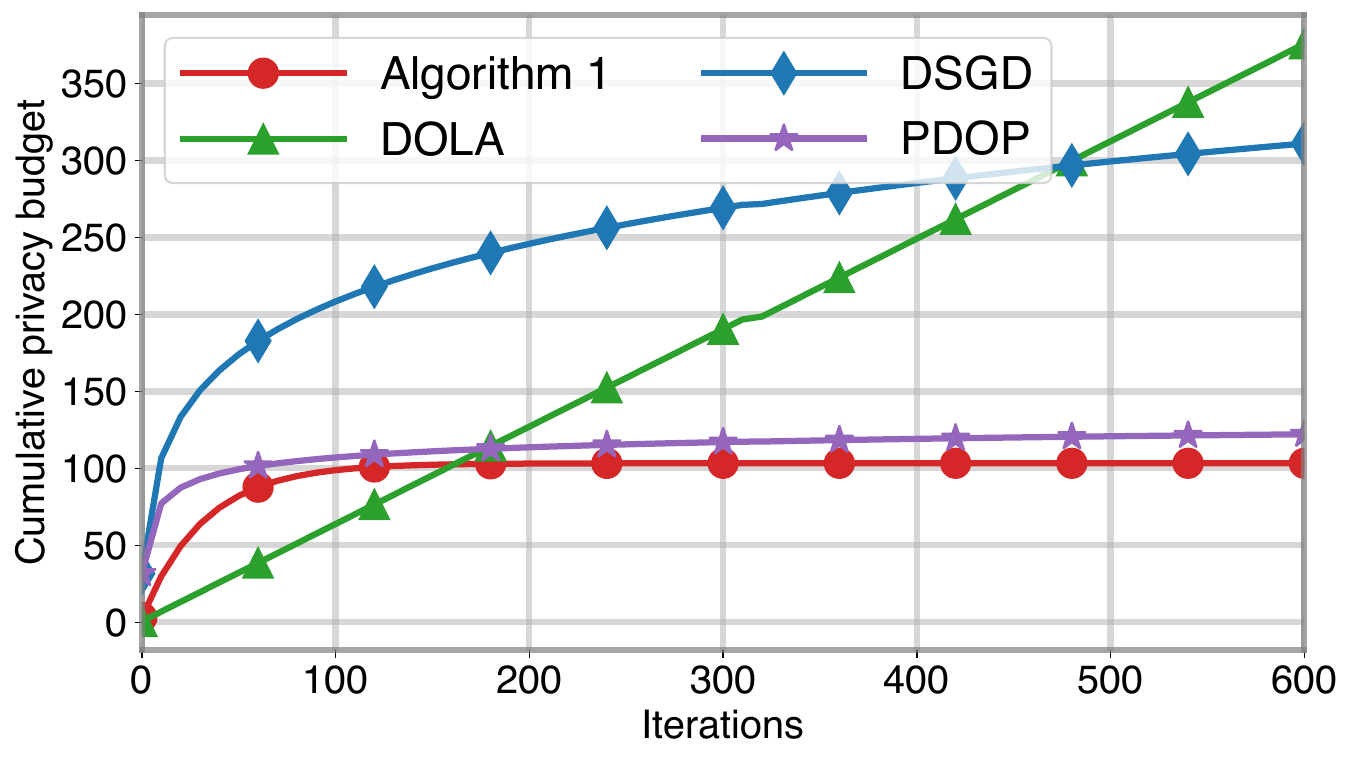}
		\end{minipage}
	}%
	\caption{Comparison of neural network training results by using the ``CIFAR-10" dataset.}
	\label{cifar10}
	\vspace{-0.8em}
\end{figure*}
\vspace{-0.7em}
\section{Numerical Experiments}\label{example}
In this section, we use three numerical experiments to validate our theoretical results. In the first experiment, we consider distributed online training of a logistic regression classifier using the ``mushrooms" dataset~\cite{mushrooms}. In the second experiment, we consider distributed online training of a convolutional neural network (CNN) using the ``MNIST" dataset~\cite{mnist}. In the third experiment, we train a CNN distributively using the ``CIFAR-10" dataset~\cite{cifar10}, which is a more diverse and challenging dataset than ``MNIST". For each test, we considered heterogeneous data distributions, which are particularly likely in distributed learning where data are collected by multiple learners from multiple sources. In all three experiments, we compared Algorithm~\ref{algorithm1} with the distributed stochastic gradient descent algorithm (DSGD) in~\cite{DSGD}, the DP approach for distributed online learning (DOLA) in~\cite{DOLA}, and the DP approach for distributed optimization (PDOP) in~\cite{PDOP}. The convex set was set as $\Theta\!=\!\{\theta\!\in\!\mathbb{R}^{n}|\|\theta\|\!\leq\! 10^{5}\}$. We considered five learners connected in a circle, where each learner can only communicate with its two immediate neighbors. For the matrix $W$, we set $w_{ij}=0.3$ if Learners $i$ and $j$ are neighbors, and $w_{ij}=0$ otherwise.
\subsection{Logistic regression using the ``mushrooms" dataset}
We first evaluated the effectiveness of Algorithm~\ref{algorithm1} by using an $l_2$-logistic regression classification problem on the ``mushrooms" dataset~\cite{mushrooms}. We spread data samples among the learners according to their target values. Specifically, Learners $1$, $2$, and $3$ have samples with the target value of $0$, while Learners $4$ and $5$ have samples with the target value of $1$. All learners cooperatively track the optimal parameter $\theta_{t}^*$ to the online optimization problem~\eqref{primalt}, in which the loss function is given by $l(\theta,\xi^{i})=\frac{1}{N_{i}}\sum_{j=1}^{N_{i}}(1-b_{j}^{i})(a_{j}^{i})^{T}\theta-\log(s((a_{j}^{i})^{T}\theta))+\frac{r_{i}}{2}\|\theta\|^2$. Here, $N_{i}$ represents the number of data points per iteration, $r_{i}>0$ is a regularization parameter proportional to $\frac{1}{N_{i}}$, $\xi^{i}=(a_{j}^{i},b_{j}^{i})$ represents the $j$-th data sample on Learner $i$, and $s(a)=1/(1+e^{-a})$ is the sigmoid function.

In each iteration, we incorporated Laplace DP noises with parameter $\varrho_{t}^{i}=(t+1)^{\varsigma^{i}}$ to all shared messages, where $\varsigma^{i}=0.1+0.01i$. Note that the multiplier $i$ in $\varsigma^{i}$ leads to different noise amplitudes and further different privacy budgets $\epsilon_{i}$ for different learners. We configured the stepsize sequence and diminishing sequence as $\lambda_{t}=\frac{1}{(t+1)^{0.77}}$ and $\gamma_{t}=\frac{1}{(t+1)^{0.65}}$, respectively. All configurations satisfy the conditions in Theorems~\ref{T1}-\ref{T5}. The algorithm was implemented for $2,000$ iterations, during which time-varying optimal parameters $\theta_{t}^*$ were calculated using a noise-free, centralized gradient descent algorithm. 

In the comparison, we selected the near-optimal stepsize sequences for DSGD, DOLA, and PDOP such that doubling the stepsize results in nonconverging behavior. The resulting average tracking error and average instantaneous regret are shown in Fig.~\ref{mushrooms}-(a) and Fig.~\ref{mushrooms}-(b), respectively. It is clear that the proposed approach has a much better learning accuracy under the constraint of local differential privacy. We also plotted the cumulative privacy budgets of all algorithms in Fig.~\ref{mushrooms}-(c), which shows that our algorithm always has a finite cumulative privacy budget whereas the cumulative privacy budgets for DSGD, PDOP, and DOLA all grow with time to infinity as iteration proceeds, implying diminishing privacy protection as iteration proceeds. 

To show that Algorithm~\ref{algorithm1}'s achievement of both rigorous LDP and optimization accuracy comes at the expense of sacrificing convergence rate, we compared the number of iterations needed to achieve a certain optimization accuracy under different cumulative privacy budgets. The results, summarized in Table~\ref{table2}, clearly show that a smaller cumulative privacy budget (i.e., stronger privacy protection) corresponds to a greater number of iterations (i.e., a slower convergence rate).
\vspace{-0.5em}
\subsection{Neural network training using the ``MNIST" dataset}
In the second experiment, we assessed Algorithm~\ref{algorithm1}'s performance through distributed online training of a convolutional neural network (CNN) using the ``MNIST" dataset~\cite{mnist}. We assigned $40\%$ of the data from the $i$-th class to Learner $i$, while splitting the remaining $60\%$ evenly among the other learners. The training process spanned $600$ iterations.

In this experiment, we utilize Laplace DP noises with parameter $\varrho_{t}^{i}\!=\!\sqrt{2}(t+1)^{\varsigma^{i}}$ for all shared messages, where $\varsigma^{i}\!=\!0.1\!+\!0.02i$. We set the stepsize sequence and decaying sequence to $\lambda_{t}\!=\!\frac{1}{(t+1)^{0.71}}$ and $\gamma_{t}\!=\!\frac{0.01}{(t+1)^{0.7}}$, respectively.

We compared our algorithm with the algorithm DSGD in~\cite{DSGD} by training the same CNN, utilizing the same stepsize sequence and the same Laplace DP noise. Additionally, we implemented existing DP methods, DOLA in~\cite{DOLA} and PDOP in~\cite{PDOP}, using DP noises with decaying and homogeneous parameters $\varrho_{t}\!=\!\frac{0.07}{0.2(t+1)}$ and $\varrho_{t}\!=\!0.5(0.98)^{t}$ for DOLA and PDOP, respectively). The stepsize sequences for DOLA and PDOP followed their default parameters suggested in~\cite{DOLA} and~\cite{PDOP}, respectively. Fig.~\ref{mnist}-(a) and Fig.~\ref{mnist}-(b) illustrate the training and test accuracies, respectively.

The results reveal that under the given DP noise, the DSGD algorithm falls short in training the CNN model. Besides, both the DOLA and PDOP algorithms are incapable of effectively training the CNN model (Note that when the test data led to an exploding loss function --- all happened when the training accuracy stalled around $0.1$ in existing algorithms ---, we used the initial parameter for validation, which always gave a test accuracy of $0.1$). These results confirm the advantages of our proposed algorithm.

To compare the strength of enabled privacy protection, we ran the DLG attack model proposed in~\cite{DLG}, which is a powerful inference algorithm capable of reconstructing raw data from shared gradient/model updates. The training/test accuracies and the DLG attacker's inference errors under different levels of DP noise are summarized in Table~\ref{table3}. It can be seen that stronger privacy protection (i.e., a larger DLG inference error) leads to lower training/test accuracies under a fixed number of $3,000$ iterations (implying a slower convergence rate).
\vspace{-0.7em}
\subsection{Neural network training using the ``CIFAR-10" dataset}
In our third experiment, we appraised Algorithm 1's performance via distributed online training of a CNN on the ``CIFAR-10" dataset, which is one of the most widely used datasets for machine learning research (it is also more difficult to train than the ``MNIST" dataset). In this experiment, the CNN architecture and all parameter designs are identical to those employed in the previous ``MNIST" dataset experiment. 

The results in Fig.~\ref{cifar10} once again confirms the effectiveness of our distributed online learning algorithm for training the complex CNN model under the constraint of LDP. 
\begin{table}[h]
	\vspace{-0.8em}
	\setlength\tabcolsep{5pt}
	\centering
	\caption{Training/test accuracies and DLG attacker's inference errors under different levels of DP noise}
	\label{table3}
	\begin{threeparttable}
		\begin{tabular}{|c|c|c|c|c|}
			\hline
			Noise level\tnote{a}  & $\times 0.5$ & $\times 1$ & $\times 1.5$ & $\times 2$  \\ \hline
			Training accuracy & 0.9402 & 0.9350 & 0.8862 & 0.8180  \\ \hline
			Test accuracy & 0.9449 & 0.9380 & 0.8964 & 0.8259  \\ \hline
			DLG inference error & 0.2696 & 0.2786 & 0.2898 & 0.3110 \\ \hline
		\end{tabular}
		\begin{tablenotes}
			\footnotesize
			\item[a] Considering Laplace noise $\text{Lap}(0.05(t+1)^{0.1+0.01^{i}})$ as the base level.
		\end{tablenotes}
	\end{threeparttable}
	\vspace{-0.8em}
\end{table}
\section{Conclusion}\label{conclusion}
In this study, we have introduced a differentially private distributed online learning algorithm that successfully circumvents the tradeoff between privacy and learning accuracy. More specifically, our proposed approach ensures a finite cumulative privacy budget in the infinite time horizon. This is in sharp contrast to existing DP methods for distributed online learning/optimization, which allow the privacy budget to grow to infinity, implying losing privacy protection when time tends to infinity. In addition, our approach also guarantees the convergence of expected instantaneous regret to zero. To the best of our knowledge, our approach is the first to achieve both rigorous local differential privacy and provable convergence in distributed online learning. Our numerical experiments on benchmark datasets confirm the advantages of the proposed approach over existing counterparts.
\section*{Appendix}
\subsection{Proof of Lemma~\ref{L1}}\label{AppendixA}
Since $\theta_{t}^{*}$ is the optimal solution to~\eqref{primalt} and $\Theta$ is a convex and compact set, the Euclidean projection property implies
\begin{equation}
\text{Pro}[\theta_{t}^*-c_{1}\nabla F_{t}(\theta_{t}^*)]=\theta_{t}^*,~\forall c_{1}>0,\nonumber
\end{equation}
which further leads to 
\begin{flalign}
&\|\theta_{t+1}^*-\theta_{t}^*\|^2\nonumber\\
&=\|\text{Pro}_{\Theta}[\theta_{t+1}^*-c_{1}\nabla F_{t+1}(\theta_{t+1}^*)]-\text{Pro}_{\Theta}[\theta_{t}^*-c_{1}\nabla F_{t}(\theta_{t}^*)]\|^2\nonumber\\
&=\|\theta_{t+1}^*-c_{1}\nabla F_{t+1}(\theta_{t+1}^*)-\theta_{t}^*+c_{1}\nabla F_{t}(\theta_{t}^*)\|^2\nonumber\\
&\leq -2c_{1}\left\langle\theta_{t+1}^*-\theta_{t}^*,\nabla F_{t+1}(\theta_{t+1}^*)-\nabla F_{t}(\theta_{t}^*)\right\rangle\nonumber\\
&\quad\quad+\|\theta_{t+1}^*\!-\!\theta_{t}^*\|^2+c_{1}^2\|\nabla F_{t+1}(\theta_{t+1}^*)\!-\!\nabla F_{t}(\theta_{t}^*)\|^2.\label{1L1}
\end{flalign}

Next, we estimate each term on the right hand side of~\eqref{1L1}:
\begin{equation}
\begin{aligned}
&2c_{1}\left\langle\theta_{t+1}^*-\theta_{t}^*,\nabla F_{t+1}(\theta_{t+1}^*)-\nabla F_{t}(\theta_{t}^*)\right\rangle\\
&=2c_{1}\left\langle\theta_{t+1}^*-\theta_{t}^*,\nabla F_{t+1}(\theta_{t+1}^*)-\nabla F_{t+1}(\theta_{t}^*)\right\rangle\\
&\quad\quad+2c_{1}\left\langle\theta_{t+1}^*-\theta_{t}^*,\nabla F_{t+1}(\theta_{t}^*)-\nabla F_{t}(\theta_{t}^*)\right\rangle.\label{1L2}
\end{aligned}
\end{equation}

Assumption~\ref{A1}-(ii) with $\mu>0$ implies that the first term on the right hand side of~\eqref{1L2} satisfies
\begin{flalign}
&2c_{1}\mathbb{E}\left[\langle\theta_{t+1}^*-\theta_{t}^*,\nabla F_{t+1}(\theta_{t+1}^*)-\nabla F_{t}(\theta_{t}^*)\rangle\right]\nonumber\\
&\geq 2c_{1}\mu\mathbb{E}\left[\|\theta_{t+1}^*-\theta_{t}^*\|^2\right].\label{1L7}
\end{flalign}

By using the Young's inequality, we have the following relationship for the second term on the right hand side of~\eqref{1L2}:
\begin{flalign}
&\left|\left\langle\theta_{t+1}^*-\theta_{t}^*,\nabla F_{t+1}(\theta_{t}^*)-\nabla F_{t}(\theta_{t}^*)\right\rangle\right|\nonumber\\
&\leq \frac{\mu}{4}\|\theta_{t+1}^*-\theta_{t}^*\|^2+\frac{1}{\mu}\|\nabla F_{t+1}(\theta_{t}^*)-\nabla F_{t}(\theta_{t}^*)\|^2.\label{1L3}
\end{flalign}
The last term on the right hand side of~\eqref{1L3} satisfies
\begin{equation}
\begin{aligned}
&\nabla F_{t+1}(\theta_{t}^*)-\nabla F_{t}(\theta_{t}^*)\\
&=\frac{1}{m}\sum_{i=1}^{m}\!\left(\!\frac{1}{t+2}\nabla l(\theta_{t}^*,\xi_{t+1}^{i})\!-\!\frac{1}{(t\!+\!1)(t\!+\!2)}\sum_{k=0}^{t}\nabla l(\theta_{t}^*,\xi_{k}^{i})\!\right)\\
&\leq \frac{1}{m}\sum_{i=1}^{m}\left(2\left\|\frac{\nabla l(\theta_{t}^*,\xi_{t+1}^{i})}{t+2}\right\|^2+2\left\|\frac{\sum_{k=0}^{t}\nabla l(\theta_{t}^*,\xi_{k}^{i})}{(t+1)(t+2)}\right\|^2\right).\label{1L4}
\end{aligned}
\end{equation}
By using Assumption~\ref{A1}-(iii) and Assumption~\ref{A3}-(ii), we have $\mathbb{E}[\|\nabla l(\theta_{t}^*,\xi_{t}^{i})\|^2]\leq 2(\kappa^2+D^2)$. By taking the expectation on both sides of~\eqref{1L4}, one obtains
\begin{equation}
\mathbb{E}[\|\nabla F_{t+1}(\theta_{t}^*)-\nabla F_{t}(\theta_{t}^*)\|^2] \leq \frac{8(\kappa^2+D^2)}{(t+1)^2}.\label{1L5}
\end{equation}
Substituting~\eqref{1L5} into~\eqref{1L3} yields
\begin{equation}
\begin{aligned}
&2c_{1}\mathbb{E}\left[\left\langle\theta_{t+1}^*-\theta_{t}^*,\nabla F_{t+1}(\theta_{t}^*)-\nabla F_{t}(\theta_{t}^*)\right\rangle\right]\\
&\geq -2c_{1}\left(\frac{\mu}{4}\mathbb{E}\left[\|\theta_{t+1}^*-\theta_{t}^*\|^2\right]+\frac{8(\kappa^2+D^2)}{\mu(t+1)^2}\right).\label{1L6}
\end{aligned}
\end{equation}

Plugging~\eqref{1L7} and~\eqref{1L6}  into~\eqref{1L2} leads to
\begin{equation}
\begin{aligned}
&2c_{1}\mathbb{E}\left[\left\langle\theta_{t+1}^*-\theta_{t}^*,\nabla F_{t+1}(\theta_{t+1}^*)-\nabla F_{t}(\theta_{t}^*)\right\rangle\right]\\
&\geq \frac{3c_{1}\mu}{2}\mathbb{E}\left[\|\theta_{t+1}^*-\theta_{t}^*\|^2\right]-\frac{16c_{1}(\kappa^2+D^2)}{\mu(t+1)^2}.\label{1L8}
\end{aligned}
\end{equation}

The Lipschitz property in Assumption~\ref{A3}-(iii) and~\eqref{1L5} imply
\begin{equation}
\begin{aligned}
&c_{1}^2\mathbb{E}\left[\|\nabla F_{t+1}(\theta_{t+1}^*)-\nabla F_{t}(\theta_{t}^*)\|^2\right]\\
&\leq 2c_{1}^2L^2\mathbb{E}\left[\|\theta_{t+1}^*-\theta_{t}^*\|^2\right]+\frac{16c_{1}^2(\kappa^2+D^2)}{(t+1)^2}.\label{1L9}
\end{aligned}
\end{equation}

Plugging~\eqref{1L8} and~\eqref{1L9} into~\eqref{1L1} leads to
\begin{equation}
\begin{aligned}
&\mathbb{E}\left[\|\theta_{t+1}^*-\theta_{t}^*\|^2\right]\\
&\leq \mathbb{E}\left[\|\theta_{t+1}^*-\theta_{t}^*\|^2\right]-\left(\frac{3c_{1}\mu}{2}-2c_{1}^2L^2\right)\mathbb{E}\left[\|\theta_{t+1}^*-\theta_{t}^*\|^2\right]\\
&\quad+\left(\frac{1}{\mu}+c_{1}\right)\frac{16c_{1}(\kappa^2+D^2)}{(t+1)^2}.\nonumber
\end{aligned}
\end{equation}
We choose $c_{1}=\frac{\mu}{2L^2}$ (i.e., $c_{1}\mu=2c_{1}^2L^2$) to obtain
\begin{equation*}
\frac{c_{1}\mu}{2}\mathbb{E}\left[\|\theta_{t+1}^*-\theta_{t}^*\|^2\right]\!\leq\! \left(\frac{1}{\mu}+c_{1}\right)\frac{16c_{1}(\kappa^2+D^2)}{(t+1)^2},
\end{equation*}
which further implies the inequality~\eqref{L1result} in Lemma~\ref{L1}.

\subsection{Technical Lemmas}\label{AppendixB}
In this subsection, we introduce two auxiliary lemmas. For the sake of notational simplicity, we add an overbar to a letter to denote the average of all learners, e.g., $\bar{\theta}_{t}=\frac{1}{m}\sum_{i=1}^{m}\theta_{t}^{i}$. We also use bold font to represent the stacked vectors of all learners, e.g., $\boldsymbol{\theta}_{t}=\text{col}(\theta_{t}^{1},\cdots,\theta_{t}^{m})$. We also denote  $\boldsymbol{\tilde{\theta}}_{t}\triangleq\boldsymbol{\theta}_{t}-\boldsymbol{\theta}_{t}^*$,  $\boldsymbol{\check{\theta}}_{t}\triangleq\boldsymbol{\theta}_{t}-\boldsymbol{\bar{\theta}}_{t}$, $\boldsymbol{d}_{t}(\boldsymbol{\theta}_{t})\triangleq\text{col}(d_{t}^{1}(\theta_{t}^{1}),\cdots,d_{t}^{m}(\theta_{t}^{m}))$, $\zeta_{t}^{wi}\triangleq\sum_{j\in{\mathcal{N}_{i}}}w_{ij}\zeta_{t}^{j}$, $\sigma_{t}^{i}\triangleq\sigma^{i}(t+1)^{\varsigma^{i}},$ and $\sigma^{+}\triangleq\max_{i\in[m]}\{\sigma^{i}\}$.

\begin{lemma}\label{C1}
Under the conditions in the statement of Theorem~\ref{T1}, the following inequality always holds:
\begin{flalign}
&\mathbb{E}\left[\|\boldsymbol{\tilde{\theta}}_{t+1}\|^2\right]\leq \left(1+\frac{\mu\lambda_{t}}{8}\right)\left[\left(1-\frac{\lambda_{t}\mu}{4}\right)\mathbb{E}[\|\boldsymbol{\tilde{\theta}}_{t}\|^2]\right.\nonumber\\
&\left.\quad+2\gamma_{t}\mathbb{E}[\boldsymbol{\tilde{\theta}}_{t}^{T}(W\otimes I_{n})\boldsymbol{\tilde{\theta}}_{t}]+3\gamma_{t}^2\mathbb{E}[\|(W\otimes I_{n})\boldsymbol{\tilde{\theta}}_{t}\|^2]\right.\nonumber\\
&\left.\quad+\lambda_{t}\left(\mu+\frac{8L^2}{\mu}\right)\mathbb{E}[\|\boldsymbol{\check{\theta}}_{t}\|^2]+3\gamma_{t}^2\|\boldsymbol{\sigma}_{t}\|^2+3\lambda_{t}^2\mathbb{E}[\|\boldsymbol{d}_{t}(\boldsymbol{\theta}_{t})\|^2]\right.\nonumber\\
&\left.\quad+\frac{12m\kappa^2\lambda_{t}}{\mu(t+1)}\right]+\left(1+\frac{8}{\lambda_{t}\mu}\right)\mathbb{E}\left[\|\boldsymbol{\theta}_{t+1}^*-\boldsymbol{\theta}_{t}^*\|^2\right],\label{C1result}
\end{flalign}
\end{lemma}
\begin{proof}
Line 7 in Algorithm~\ref{algorithm1} can be written into the following compact form:
\begin{equation}
\boldsymbol{\hat{\theta}}_{t+1}=(I+\gamma_{t}(W\otimes I_{n}))\boldsymbol{\theta}_{t}+\gamma_{t}\boldsymbol{\zeta}_{t}^{w}-\lambda_{t}\boldsymbol{d}_{t}(\boldsymbol{\theta}_{t}).\nonumber
\end{equation}
By using the projection inequality, we obtain
\begin{equation}
\begin{aligned}
&\|\boldsymbol{\theta}_{t+1}-\boldsymbol{\boldsymbol{\theta}}_{t}^*\|^2=\|\text{Pro}_{\Theta}[\boldsymbol{\hat{\theta}}_{t+1}]-\boldsymbol{\theta}_{t}^*\|^2\leq \|\boldsymbol{\hat{\theta}}_{t+1}-\boldsymbol{\theta}_{t}^*\|^2\\
&=\|\boldsymbol{\tilde{\theta}}_{t}\|^2+\|\gamma_{t}(W\otimes I_{n})\boldsymbol{\theta}_{t}+\gamma_{t}\boldsymbol{\zeta}_{t}^{w}-\lambda_{t}\boldsymbol{d}_{t}(\boldsymbol{\theta}_{t})\|^2\\
&\quad+2\langle\boldsymbol{\tilde{\theta}}_{t}, \gamma_{t}(W\otimes I_{n})\boldsymbol{\theta}_{t}+\gamma_{t}\boldsymbol{\zeta}_{t}^{w}-\lambda_{t}\boldsymbol{d}_{t}(\boldsymbol{\theta}_{t})\rangle.\label{1C1}
\end{aligned}
\end{equation}
	
1) Using the Cauchy-Schwartz inequality, the second term on the right hand side of~\eqref{1C1} satisfies
\begin{equation}
\begin{aligned}
&\mathbb{E}[\|\gamma_{t}(W\otimes I_{n})\boldsymbol{\tilde{\theta}}_{t}+\gamma_{t}\boldsymbol{\zeta}_{t}^{w}-\lambda_{t}\boldsymbol{d}_{t}(\boldsymbol{\theta}_{t})\|^2]\\
&\leq 3\gamma_{t}^2\mathbb{E}[\|(W\otimes I_{n})\boldsymbol{\tilde{\theta}}_{t}\|^2]+3\gamma_{t}^2\|\boldsymbol{\sigma}_{t}\|^2+3\lambda_{t}^2\mathbb{E}[\|\boldsymbol{d}_{t}(\boldsymbol{\theta}_{t})\|^2].\label{1C4}
\end{aligned}
\end{equation}
where in the derivation we have used the relationship $\mathbb{E}[\|\boldsymbol{\zeta}_{t}^{w}\|^2]=2\|\boldsymbol{\varrho}_{t}\|^2=\|\boldsymbol{\sigma}_{t}\|^2$ with $\boldsymbol{\varrho}_{t}\triangleq\text{col}(\frac{\sigma_{t}^{1}}{\sqrt{2}},\cdots,\frac{\sigma_{t}^{m}}{\sqrt{2}})$ and $\boldsymbol{\sigma}_{t}\triangleq\text{col}(\sigma_{t}^{1},\cdots,\sigma_{t}^{m})$.
	
Moreover, Assumption~\ref{A4} implies that the third term on the right hand side of~\eqref{1C1} satisfies
\begin{equation}
\begin{aligned}	&\mathbb{E}[\langle\boldsymbol{\tilde{\theta}}_{t}, \gamma_{t}(W\otimes I_{n})\boldsymbol{\tilde{\theta}}_{t}+\gamma_{t}\boldsymbol{\zeta}_{t}^{w}-\lambda_{t}\boldsymbol{d}_{t}(\boldsymbol{\theta}_{t})\rangle]\\	&=\gamma_{t}\mathbb{E}[\boldsymbol{\tilde{\theta}}_{t}^{T}(W\otimes I_{n})\boldsymbol{\tilde{\theta}}_{t}]-\mathbb{E}[\langle\boldsymbol{\tilde{\theta}}_{t},\lambda_{t}\boldsymbol{d}_{t}(\boldsymbol{\theta}_{t})\rangle].\label{1C3}
\end{aligned}
\end{equation}

Substituting~\eqref{1C4} and~\eqref{1C3}  into~\eqref{1C1} leads to
\begin{equation}
\begin{aligned}
&\mathbb{E}\left[\|\boldsymbol{\theta}_{t+1}-\boldsymbol{\boldsymbol{\theta}}_{t}^*\|^2\right]\\
&\leq \mathbb{E}[\|\boldsymbol{\tilde{\theta}}_{t}\|^2]+2\gamma_{t}\mathbb{E}[\boldsymbol{\tilde{\theta}}_{t}^{T}(W\otimes I_{n})\boldsymbol{\tilde{\theta}}_{t}]\\
&\quad-2\lambda_{t}\mathbb{E}[\langle\boldsymbol{\tilde{\theta}}_{t},\boldsymbol{d}_{t}(\boldsymbol{\theta}_{t})\rangle]+3\gamma_{t}^2\mathbb{E}[\|(W\otimes I_{n})\boldsymbol{\tilde{\theta}}_{t}\|^2]\\
&\quad+3\gamma_{t}^2\|\boldsymbol{\sigma}_{t}\|^2+3\lambda_{t}^2\mathbb{E}[\|\boldsymbol{d}_{t}(\boldsymbol{\theta}_{t})\|^2].\label{1C5}
\end{aligned}
\end{equation}

2) We proceed to characterize the third term on the right hand side of~\eqref{1C5} by using the following decomposition:
\begin{flalign}
\mathbb{E}[\langle\boldsymbol{\tilde{\theta}}_{t},\boldsymbol{d}_{t}(\boldsymbol{\theta}_{t})\rangle]&=\mathbb{E}[\langle\boldsymbol{\tilde{\theta}}_{t},\nabla \boldsymbol{f}(\bar{\theta}_{t})\rangle]+\mathbb{E}[\langle\boldsymbol{\tilde{\theta}}_{t}, \boldsymbol{d}_{t}(\boldsymbol{\theta}_{t})\!-\!\boldsymbol{d}_{t}(\bar{\theta}_{t})\rangle]\nonumber\\
&\quad+\mathbb{E}[\langle\boldsymbol{\tilde{\theta}}_{t},\boldsymbol{d}_{t}(\bar{\theta}_{t})-\nabla \boldsymbol{f}(\bar{\theta}_{t})\rangle],\label{1C6}
\end{flalign}
with $\nabla \boldsymbol{f}(\bar{\theta}_{t})\triangleq \text{col}(\nabla f_{1}(\bar{\theta}_{t}),\cdots,\nabla f_{m}(\bar{\theta}_{t}))$ and $\boldsymbol{d}_{t}(\bar{\theta}_{t})\triangleq\text{col}(d_{t}^{1}(\bar{\theta}_{t}),\cdots,d_{t}^{m}(\bar{\theta}_{t}))$.

The first term on the right hand side of~\eqref{1C6} satisfies
\begin{equation}
\begin{aligned}
&\mathbb{E}[\langle\boldsymbol{\tilde{\theta}}_{t},\nabla \boldsymbol{f}(\bar{\theta}_{t})\rangle]=\mathbb{E}\left[\left\langle m\bar{\theta}_{t}-m\theta_{t}^*,\sum_{i=1}^{m}\nabla f_{i}(\bar{\theta}_{t})\right\rangle\right]\\
&\geq \mathbb{E}[mF(\bar{\theta}_{t})-mF(\theta_{t}^*)]+\frac{\mu}{4}\mathbb{E}[\|\boldsymbol{\tilde{\theta}}_{t}\|^2]-\frac{\mu}{2}\mathbb{E}[\|\boldsymbol{\check{\theta}}_{t}\|^2],\label{1C7}
\end{aligned}
\end{equation}
where we have used Assumption~\ref{A1}-(ii) with $\mu>0$ and the Cauchy-Schiwartz inequality in the last inequality.
		
We further characterize the first term on the right hand side of~\eqref{1C7}. Since the relationship $F_{t}({\theta}^*_{t})\leq F_t(\theta^*)$ holds, we have
\begin{equation}
\begin{aligned}
& \mathbb{E}[F(\theta_{t}^*)-F(\theta^*)]\\
&\leq  \mathbb{E}[(F(\theta_{t}^*)-F_t(\theta_{t}^*))-(F(\theta^*)-F_t(\theta^*))]\\
&=\mathbb{E}[\langle \nabla F(\Xi) -\nabla F_t(\Xi),\theta^*_t-\theta^*\rangle],
\label{1C8}
\end{aligned}
\end{equation}
where in the derivation we have used the mean value theorem with $\Xi=\alpha\theta^*_t+(1-\alpha)\theta^*$ valid for some $\alpha\in(0,1)$.

Assumption~\ref{A3}-(ii) implies $\mathbb{E}[\|\nabla F(\Xi) -\nabla F_t(\Xi)\|^2]\leq\frac{\kappa^2}{t+1}$ for any $\Xi \in\Theta$. Then, we can use the Cauchy-Bunyakovski inequality to rewrite~\eqref{1C8} as
\begin{equation}
\mathbb{E}[F(\theta_{t}^*)-F(\theta^*)]
\leq \frac{\kappa}{\sqrt{t+1}}\sqrt{\mathbb{E}[\|\theta^*_t-\theta^*\|^2]}.
\label{1C9}
\end{equation}
Assumption~\ref{A1}-(ii) implies 
\begin{equation}
\frac{\mu}{2}\mathbb{E}[\|\theta^*_t-\theta^*\|^2]\leq \mathbb{E}[F(\theta_{t}^*)-F(\theta^*)].
\label{1C10}
\end{equation}
Substituting~\eqref{1C10} into the inequality~\eqref{1C9} yields $\mathbb{E}[F(\theta_{t}^*)-F(\theta^*)]\leq\frac{2\kappa^2}{\mu(t+1)},$ which implies
\begin{equation}
\mathbb{E}[F(\bar{\theta}_{t})-F(\theta_{t}^*)]\geq-\frac{2\kappa^2}{\mu(t+1)}. \label{1C11}
\end{equation}

By substituting~\eqref{1C11} into~\eqref{1C7}, we can obtain the following inequality for the first term on the right hand side of~\eqref{1C6}:
\begin{flalign}
\mathbb{E}[\langle\boldsymbol{\tilde{\theta}}_{t},\nabla \boldsymbol{f}(\bar{\theta}_{t})\rangle]\geq -\frac{2m\kappa^2}{\mu(t+1)}+\frac{\mu}{4}\mathbb{E}[\|\boldsymbol{\tilde{\theta}}_{t}\|^2]-\frac{\mu}{2}\mathbb{E}[\|\boldsymbol{\check{\theta}}_{t}\|^2].\label{1C71}
\end{flalign}

The second term on the right hand side of~\eqref{1C6} satisfies
\begin{equation}
\begin{aligned}
\mathbb{E}[\langle\boldsymbol{\tilde{\theta}}_{t}, \boldsymbol{d}_{t}(\boldsymbol{\theta}_{t})\!-\!\boldsymbol{d}_{t}(\bar{\theta}_{t})\rangle]&\geq -\frac{\mu}{16}\mathbb{E}[\|\boldsymbol{\tilde{\theta}}_{t}\|^2]\\
&\quad-\frac{4}{\mu}\mathbb{E}[\|\boldsymbol{d}_{t}(\boldsymbol{\theta}_{t})-\boldsymbol{d}_{t}(\bar{\theta}_{t})\|^2].\label{1C12}
\end{aligned}
\end{equation}
Based on Assumption~\ref{A3}-(iii), we have
\begin{flalign}
&\mathbb{E}[\|\boldsymbol{d}_{t}(\boldsymbol{\theta}_{t})-\boldsymbol{d}_{t}(\bar{\theta}_{t})\|^2]\nonumber\\
&=\mathbb{E}\left[\sum_{i=1}^{m}\left\|\frac{1}{t+1}\sum_{k=0}^{t}\nabla l(\theta_{t}^{i},\xi_{k}^{i})-\frac{1}{t+1}\sum_{k=0}^{t}\nabla l(\bar{\theta}_{t},\xi_{k}^{i})\right\|^2\right]\nonumber\\
&\leq L^2\mathbb{E}\left[\|\boldsymbol{\theta}_{t}-\boldsymbol{1}_{m}\otimes\bar{\theta}_{t}\|^2\right]=L^2\mathbb{E}[\|\boldsymbol{\check{\theta}}_{t}\|^2].\label{1c121}
\end{flalign}
Substituting~\eqref{1c121} into~\eqref{1C12} yields
\begin{equation}
\mathbb{E}[\langle\boldsymbol{\tilde{\theta}}_{t}, \boldsymbol{d}_{t}(\boldsymbol{\theta}_{t})\!-\!\boldsymbol{d}_{t}(\bar{\theta}_{t})\rangle]\geq -\frac{\mu}{16}\mathbb{E}[\|\boldsymbol{\tilde{\theta}}_{t}\|^2]-\frac{4L^2}{\mu}\mathbb{E}[\|\boldsymbol{\check{\theta}}_{t}\|^2].\label{1C122}
\end{equation} 

Since Assumption~\ref{A3}-(ii) implies $\mathbb{E}[\|\boldsymbol{d}_{t}(\theta)-\nabla\boldsymbol{f}(\theta)\|^2]\leq\frac{m\kappa^2}{t+1}$, the third term on the right hand side of~\eqref{1C6} satisfies 
\begin{flalign}
\left|\mathbb{E}\left[\langle\boldsymbol{\tilde{\theta}}_{t}, \boldsymbol{d}_{t}(\bar{\theta}_{t})\!-\!\nabla \boldsymbol{f}(\bar{\theta}_{t})\rangle\right]\right|\!\leq\! \frac{\mu}{16}\mathbb{E}[\|\boldsymbol{\tilde{\theta}}_{t}\|^2]+\frac{4m\kappa^2}{\mu(t+1)}.\label{1C13}
\end{flalign}

Substituting~\eqref{1C71},~\eqref{1C122}, and~\eqref{1C13} into~\eqref{1C6} leads to
\begin{equation}
\begin{aligned}
&2\lambda_{t}\mathbb{E}[\langle\boldsymbol{\tilde{\theta}}_{t},\boldsymbol{d}_{t}(\boldsymbol{\theta}_{t})\rangle]\geq \frac{\lambda_{t}\mu}{4}\mathbb{E}[\|\boldsymbol{\tilde{\theta}}_{t}\|^2]\\
&\quad-\lambda_{t}\left(\mu+\frac{8L^2}{\mu}\right)\mathbb{E}[\|\boldsymbol{\check{\theta}}_{t}\|^2]-\frac{12m\kappa^2\lambda_{t}}{\mu(t+1)}.\label{1C14}
\end{aligned}
\end{equation}
	
3) Substituting~\eqref{1C14} into~\eqref{1C5}, we obtain
\begin{equation}
\begin{aligned}
&\mathbb{E}[\|\boldsymbol{\theta}_{t+1}-\boldsymbol{\boldsymbol{\theta}}_{t}^*\|^2]\leq \left(1-\frac{\lambda_{t}\mu}{4}\right)\mathbb{E}[\|\boldsymbol{\tilde{\theta}}_{t}\|^2]+\frac{12m\kappa^2\lambda_{t}}{\mu(t+1)}\\
&\quad+2\gamma_{t}\mathbb{E}[\boldsymbol{\tilde{\theta}}_{t}^{T}(W\otimes I_{n})\boldsymbol{\tilde{\theta}}_{t}]+3\gamma_{t}^2\mathbb{E}[\|(W\otimes I_{n})\boldsymbol{\tilde{\theta}}_{t}\|^2]\\
&\quad+\lambda_{t}\left(\mu\!+\!\frac{8L^2}{\mu}\right)\mathbb{E}[\|\boldsymbol{\check{\theta}}_{t}\|^2]\!+\!3\gamma_{t}^2\|\boldsymbol{\sigma}_{t}\|^2\!+\!3\lambda_{t}^2\mathbb{E}[\|\boldsymbol{d}_{t}(\boldsymbol{\theta}_{t})\|^2].\label{1C15}
\end{aligned}
\end{equation}
By using the relationship $\|\boldsymbol{\theta}_{t+1}-\boldsymbol{\theta}_{t}\|^2\leq a_{t}\|\boldsymbol{\theta}_{t+1}-\boldsymbol{\boldsymbol{\theta}}_{t}^*\|^2+b_{t}\|\boldsymbol{\theta}_{t+1}^*-\boldsymbol{\theta}_{t}^*\|^2$ valid for any $a_{t},b_{t}>1$, where $(a_{t}-1)(b_{t}-1)=1$, we arrive at~\eqref{C1result} in Lemma~\ref{C1}.
\end{proof}

For the convenience of analysis, we introduce $s\in[0,t]$ and denote $\boldsymbol{\tilde{\theta}}_{t-s}\triangleq\boldsymbol{\theta}_{t-s}-\boldsymbol{\theta}_{t+1}^*$ and $\boldsymbol{\tilde{\theta}}_{t+1-s}\triangleq\boldsymbol{\theta}_{t+1-s}-\boldsymbol{\theta}_{t+1}^*$.
\begin{lemma}\label{C2}
Under the conditions in the statement of Theorem~\ref{T2}, the following inequality always holds:
\begin{flalign}
&\mathbb{E}[\|\boldsymbol{\tilde{\theta}}_{t+1-s}\|^2]\leq (1+\lambda_{t-s}(\eta_s+\eta_{t-s}))\mathbb{E}[\|\boldsymbol{\tilde{\theta}}_{t-s}\|^2]\nonumber\\
&-2m\lambda_{t-s}\mathbb{E}[F_{t+1}(\theta_{t-s}^{i})-F_{t+1}(\theta_{t+1}^*)]+3\gamma_{t-s}^2\|\boldsymbol{\sigma}_{t-s}\|^2\nonumber\\
&+3\lambda_{t-s}^2\mathbb{E}[\|\boldsymbol{d}_{t-s}(\boldsymbol{\theta}_{t-s})\|^2]\!+\!2\gamma_{t-s}\mathbb{E}[\boldsymbol{\tilde{\theta}}_{t-s}^{T}(W\otimes I_{n})\boldsymbol{\tilde{\theta}}_{t-s}]\nonumber\\
&+3\gamma_{t-s}^2\mathbb{E}[\|(W\otimes I_{n})\boldsymbol{\tilde{\theta}}_{t-s}\|^2]+m{\lambda_{t-s}}{\eta_{t-s}}\nonumber\\
&+\frac{8\lambda_{t-s}(\kappa^2+D^2)(s+1)}{\eta_{s}(t+2)(t-s+1)}+\frac{4m\kappa^2\lambda_{t-s}}{\eta_{t-s}(t+1)}+\frac{2m\kappa R\lambda_{t-s}}{\sqrt{t+1}}\nonumber\\
&+\frac{2(L^2+\kappa^2+D^2)\lambda_{t-s}\mathbb{E}[\|\boldsymbol{\check{\theta}}_{t-s}\|^2]}{\eta_{t-s}},\label{2Cresult}
\end{flalign}
where the sequence $\eta_{t}$ is given by $\eta_{t}=\frac{1}{(t+1)^{r}}$ with $r=\frac{1-v}{2}$.
\end{lemma}
\begin{proof}
By delivering an argument similar to that of~\eqref{1C5}, we
can obtain
\begin{equation}
\begin{aligned}
&\mathbb{E}[\|\boldsymbol{\theta}_{t+1-s}-\boldsymbol{\boldsymbol{\theta}}_{t+1}^*\|^2]\\
&\leq \mathbb{E}[\|\boldsymbol{\tilde{\theta}}_{t-s}\|^2]+2\gamma_{t-s}\mathbb{E}[\boldsymbol{\tilde{\theta}}_{t-s}^{T}(W\otimes I_{n})\boldsymbol{\tilde{\theta}}_{t-s}]\\
&\quad-2\lambda_{t-s}\mathbb{E}[\langle\boldsymbol{\tilde{\theta}}_{t-s},\boldsymbol{d}_{t-s}(\boldsymbol{\theta}_{t-s})\rangle]\\
&\quad+3\gamma_{t-s}^2\mathbb{E}[\|(W\otimes I_{n})\boldsymbol{\tilde{\theta}}_{t-s}\|^2]\\
&\quad+3\gamma_{t-s}^2\|\boldsymbol{\sigma}_{t-s}\|^2+3\lambda_{t-s}^2\mathbb{E}[\|\boldsymbol{d}_{t-s}(\boldsymbol{\theta}_{t-s})\|^2].\label{2C1}
\end{aligned}
\end{equation}

The third term on the right hand side of~\eqref{2C1} can be decomposed into
\begin{equation}
\begin{aligned}
&2\lambda_{t-s}\mathbb{E}[\langle\boldsymbol{\tilde{\theta}}_{t-s},\boldsymbol{d}_{t-s}(\boldsymbol{\theta}_{t-s})\rangle]\\
&=2\lambda_{t-s}\mathbb{E}[\langle\boldsymbol{\tilde{\theta}}_{t-s},\boldsymbol{d}_{t+1}(\boldsymbol{\theta}_{t-s})\rangle]\\
&\quad-2\lambda_{t-s}\mathbb{E}[\langle\boldsymbol{\tilde{\theta}}_{t-s},\boldsymbol{d}_{t+1}(\boldsymbol{\theta}_{t-s})\!-\!\boldsymbol{d}_{t-s}(\boldsymbol{\theta}_{t-s})\rangle].\label{2C2}
\end{aligned}
\end{equation}

1) We fist decompose the first term on the right hand side of~\eqref{2C2} to obtain
\begin{equation}
\begin{aligned}
&2\lambda_{t-s}\mathbb{E}[\langle\boldsymbol{\tilde{\theta}}_{t-s},\boldsymbol{d}_{t+1}(\boldsymbol{\theta}_{t-s})\rangle]\\
&=2\lambda_{t-s}\mathbb{E}[\langle\boldsymbol{\tilde{\theta}}_{t-s},\nabla \boldsymbol{f}(\bar{\theta}_{t-s})\rangle]\\
&\quad+2\lambda_{t-s}\mathbb{E}[\langle\boldsymbol{\tilde{\theta}}_{t-s},\boldsymbol{d}_{t+1}(\boldsymbol{\theta}_{t-s})-\boldsymbol{f}(\bar{\theta}_{t-s})\rangle],\label{2C3}
\end{aligned}
\end{equation}
with $\nabla \boldsymbol{f}(\bar{\theta}_{t-s})\triangleq \text{col}(\nabla f_{1}(\bar{\theta}_{t-s}),\cdots,\nabla f_{m}(\bar{\theta}_{t-s}))$ and  $\boldsymbol{d}_{t+1}(\bar{\theta}_{t-s})\triangleq\text{col}(d_{t+1}^{1}(\bar{\theta}_{t-s}),\cdots,d_{t+1}^{m}(\bar{\theta}_{t-s}))$.

By using Assumption~\ref{A1}-(ii) with $\mu=0$ and the relation $\sum_{i=1}^{m}\theta_{t-s}^{i}=m\bar\theta_{t-s}$, we have
\begin{equation}
\begin{aligned}
&2\lambda_{t-s}\mathbb{E}[\langle\boldsymbol{\tilde{\theta}}_{t-s},\nabla \boldsymbol{f}(\bar{\theta}_{t-s})\rangle]\\
&=\mathbb{E}\left[2\lambda_{t-s}\sum_{i=1}^{m}\left\langle\theta_{t-s}^{i}-\theta_{t+1}^*,\nabla f_{i}(\bar\theta_{t-s})\right\rangle\right]\\
&\geq  2\lambda_{t-s}\mathbb{E}[mF(\bar\theta_{t-s})-mF(\theta_{t+1}^*)].\label{2C31}
\end{aligned}
\end{equation}
		
According to the mean value theorem, we have that for any $x,y\in\Theta$, the following inequality always holds:
\begin{equation}
\begin{aligned}
& |\mathbb{E}[(F(x)-F(y))-(F_t(x)-F_t(y))]|\\
&=|\mathbb{E}[\langle \nabla F(\Xi) -\nabla F_t(\Xi),x-y\rangle]|\\
&\leq \sqrt{\mathbb{E}[\|\nabla F(\Xi)-\nabla F_t(\Xi) \|^2]}\mathbb{E}[\|x-y\|]\leq \frac{\kappa R}{\sqrt{t+1}},
\label{2C32}
\end{aligned}
\end{equation}
where in the last inequality we have used that Assumption~\ref{A3}-(ii) implies $\mathbb{E}[\|\nabla F(\Xi) -\nabla F_t(\Xi)\|^2]\leq\frac{\kappa^2}{t+1}$ for any $\Xi \in\Theta$ and Assumption~\ref{A1}-(i) implies $\|x-y\|\leq R$ for any $x,y\in\Theta$.

Based on~\eqref{2C32}, we rewrite the inequality~\eqref{2C31} as follows:
\begin{equation}
\begin{aligned}
&2\lambda_{t-s}\mathbb{E}[\langle\boldsymbol{\tilde{\theta}}_{t-s},\nabla \boldsymbol{f}(\bar{\theta}_{t-s})\rangle]\geq -\frac{2\lambda_{t-s}m\kappa R}{\sqrt{t+1}}\\
&\quad+  2\lambda_{t-s}m\mathbb{E}[F_{t+1}(\bar\theta_{t-s})-mF_{t+1}(\theta_{t+1}^*)].\label{2C33}
\end{aligned}
\end{equation}

We estimate the second term on the right hand side of~\eqref{2C3}:
\begin{flalign}
&\mathbb{E}[\langle\boldsymbol{\tilde{\theta}}_{t-s},\boldsymbol{d}_{t+1}(\boldsymbol{\theta}_{t-s})-\nabla \boldsymbol{f}(\bar{\theta}_{t-s})\rangle]\nonumber\\
&\geq-\frac{\eta_{t-s}}{2}\mathbb{E}[\|\boldsymbol{\tilde{\theta}}_{t-s}\|^2]-\frac{1}
{\eta_{t-s}}\mathbb{E}[\|\boldsymbol{d}_{t+1}(\boldsymbol{\theta}_{t-s})-\boldsymbol{d}_{t+1}(\bar{\theta}_{t-s})\|^2]\nonumber\\
&\quad-\frac{1}
{\eta_{t-s}}\mathbb{E}[\|\boldsymbol{d}_{t+1}(\bar{\theta}_{t-s})-\nabla \boldsymbol{f}(\bar{\theta}_{t-s})\|^2],\label{2C3ii}
\end{flalign}
for any $\eta_{t-s}>0$. Assumption~\ref{A3}-(ii) and~\ref{A3}-(iii) imply
\begin{flalign}
&\mathbb{E}[\langle\boldsymbol{\tilde{\theta}}_{t-s},\boldsymbol{d}_{t+1}(\boldsymbol{\theta}_{t-s})-\nabla \boldsymbol{f}(\bar{\theta}_{t-s})\rangle]\nonumber\\
&\geq-\left(\frac{\eta_{t-s}}{2}\mathbb{E}[\|\boldsymbol{\tilde{\theta}}_{t-s}\|^2]+\frac{L^2}
{\eta_{t-s}}\mathbb{E}[\|\boldsymbol{\breve{\theta}}_{t-s}\|^2]+\frac{2m\kappa^2}{\eta_{t-s}(t+1)}\right).\label{2C3iii}
\end{flalign}

Substituting~\eqref{2C33} and~\eqref{2C3iii} into~\eqref{2C3} leads to
\begin{equation}
\begin{aligned}
&2\lambda_{t-s}\mathbb{E}[\langle\boldsymbol{\tilde{\theta}}_{t-s},\boldsymbol{d}_{t+1}(\boldsymbol{\theta}_{t-s})\rangle]\\
&\geq 2\lambda_{t-s}m\mathbb{E}[F_{t+1}(\bar\theta_{t-s})-mF_{t+1}(\theta_{t+1}^*)]\\
&\quad-2\lambda_{t-s}\left(\frac{\eta_{t-s}}{2}\mathbb{E}[\|\boldsymbol{\tilde{\theta}}_{t-s}\|^2]+\frac{L^2}
{\eta_{t-s}}\mathbb{E}[\|\boldsymbol{\breve{\theta}}_{t-s}\|^2]\right.\\
&\left.\quad+\frac{2m\kappa^2}{\eta_{t-s}(t+1)}\right)-\frac{2\lambda_{t-s}m\kappa R}{\sqrt{t+1}}.\label{2C4}
\end{aligned}
\end{equation}

2) Next, we characterize the second term on the right hand side of~\eqref{2C2}:
\begin{equation}
\begin{aligned}
&-2\lambda_{t-s}\mathbb{E}[\langle\boldsymbol{\tilde{\theta}}_{t-s},\boldsymbol{d}_{t+1}(\boldsymbol{\theta}_{t-s})\!-\!\boldsymbol{d}_{t-s}(\boldsymbol{\theta}_{t-s})\rangle]\\
&\geq\!-\!\lambda_{t-s}\eta_{s}\mathbb{E}[\|\boldsymbol{\tilde{\theta}}_{t-s}\|^2]\!-\!\frac{\lambda_{t-s}}{2\eta_{s}}\mathbb{E}[\|\boldsymbol{d}_{t+1}(\boldsymbol{\theta}_{t-s})\!-\!\boldsymbol{d}_{t-s}(\boldsymbol{\theta}_{t-s})\|^2].\label{2C5}
\end{aligned}
\end{equation}
We define an auxiliary random variable $\psi(\theta_{t-s}^{i},\xi^{i})\triangleq\nabla l(\theta_{t-s}^{i},\xi^{i})-\mathbb{E}\left[\nabla l(\theta_{t-s}^{i},\xi^{i})\right],$ which implies $\mathbb{E}[\psi(\theta_{t-s}^{i},\xi^{i})]=0$. Furthermore, the independent and identically distributed data always ensure 
\begin{equation}
\mathbb{E}\left[\sum_{k=0}^{t-s}\psi(\theta_{t-s}^{i},\xi_{k}^{i})\!\times\!\sum_{k=t-s+1}^{t+1}\psi(\theta_{t-s}^{i},\xi_{k}^{i})\right]\!=\!0.
\end{equation}
By using the relationship $\frac{1}{t+2}\sum_{k=0}^{t+1}\mathbb{E}\left[\nabla l(\theta_{t-s}^{i},\xi_{k}^{i})\right]\!=\!\frac{1}{t-s+1}\sum_{k=0}^{t-s}\mathbb{E}\left[\nabla l(\theta_{t-s}^{i},\xi_{k}^{i})\right]$, we can obtain
\begin{flalign}
&\mathbb{E}\left[\|d_{t+1}^{i}(\theta_{t-s}^{i})-d_{t-s}^{i}(\theta_{t-s}^{i})\|^2\right]\nonumber\\
&=\mathbb{E}\left[\left\|\frac{1}{t+2}\sum_{k=0}^{t+1}\psi(\theta_{t-s}^{i},\xi_{k}^{i})-\frac{1}{t-s+1}\sum_{k=0}^{t-s}\psi(\theta_{t-s}^{i},\xi_{k}^{i})\right\|^2\right]\nonumber\\
&= \mathbb{E}\left[\left\|(\frac{1}{t+2}-\frac{1}{t-s+1})\sum_{k=0}^{t-s}\psi(\theta_{t-s}^{i},\xi_{k}^{i})\right\|^2\right]\nonumber\\
&\quad+\mathbb{E}\left[\left\|\frac{1}{t+2}\sum_{k=t-s+1}^{t+1}\psi(\theta_{t-s}^{i},\xi_{k}^{i})\right\|^2\right].\label{2C5i}
\end{flalign}
We use Assumption~\ref{A1}-(iii) and~Assumption~\ref{A3}-(ii) to ensure $\mathbb{E}\left[\|\nabla l(\theta_{t-s}^{i},\xi_{k}^{i})\|\right]\leq 2(\kappa^2+D^2)$, which further leads to
\begin{equation}
\begin{aligned}
&\mathbb{E}\left[\left\|\sum_{k=0}^{t-s}\psi(\theta_{t-s}^{i},\xi_{k}^{i})\right\|^2\right]\leq 8(\kappa^2+D^2)(t-s+1).\label{2C5ii}
\end{aligned}
\end{equation}
Similarly, we have $$\mathbb{E}\left[\sum_{k=t-s+1}^{t+1}\|\psi(\theta_{t-s}^{i},\xi_{k}^{i})\|^2\right]\leq 8(\kappa^2+D^2)(s+1).$$ Then, the equality \eqref{2C5i} can be rewritten as
\begin{equation}
\mathbb{E}\left[\|d_{t+1}^{i}(\theta_{t-s}^{i})-d_{t-s}^{i}(\theta_{t-s}^{i})\|^2\right]\leq \frac{8(\kappa^2+D^2)(s+1)}{(t+2)(t-s+1)}.\label{2C5iii}
\end{equation}

Substituting~\eqref{2C5iii} into \eqref{2C5}, we arrive at
\begin{equation}
\begin{aligned}
&-2\lambda_{t-s}\mathbb{E}[\langle\boldsymbol{\tilde{\theta}}_{t-s},\boldsymbol{d}_{t+1}(\boldsymbol{\theta}_{t-s})-\boldsymbol{d}_{t-s}(\boldsymbol{\theta}_{t-s})\rangle]\\
&\geq-\lambda_{t-s}\eta_{s}\mathbb{E}[\|\boldsymbol{\tilde{\theta}}_{t-s}\|^2]-\frac{\lambda_{t-s}}{\eta_{s}}\!\frac{8m(\kappa^2+D^2)(s+1)}{(t+2)(t-s+1)}.\label{2C6}
\end{aligned}
\end{equation}

3) Combining~\eqref{2C4} and \eqref{2C6} with~\eqref{2C2}, we obtain
\begin{flalign}
&2\lambda_{t-s}\mathbb{E}[\langle\boldsymbol{\tilde{\theta}}_{t-s},\boldsymbol{d}_{t-s}(\boldsymbol{\theta}_{t-s})\rangle]\nonumber\\
&\geq-\frac{\lambda_{t-s}}{\eta_{s}}\frac{8m(\kappa^2+D^2)(s+1)}{(t+2)(t-s+1)}-\frac{2\lambda_{t-s}L^2}{\eta_{t-s}}\mathbb{E}\left[\|\boldsymbol{\check{\theta}}_{t-s}\|^2\right]\nonumber\\
&\quad+2\lambda_{t-s}m\mathbb{E}[F_{t+1}(\bar\theta_{t-s})-mF_{t+1}(\theta_{t+1}^*)]\nonumber\\
&\quad-\lambda_{t-s}(\eta_{s}+\eta_{t-s})\mathbb{E}\left[\|\boldsymbol{\tilde{\theta}}_{t-s}\|^2\right]\nonumber\\
&\quad-2\lambda_{t-s}\left(\frac{2m\kappa^2}{\eta_{t-s}(t+1)}+\frac{m\kappa R}{\sqrt{t+1}}\right).\label{2C7}
\end{flalign}
Assumption~\ref{A1}-(ii) with $\mu=0$ and the Young's inequality imply
\begin{equation}
\begin{aligned}
&2m\lambda_{t-s}\mathbb{E}\left[F_{t+1}(\bar{\theta}_{t-s})-F_{t+1}(\theta_{t+1}^*)\right]\\
&\geq 2m\lambda_{t-s}\mathbb{E}\left[F_{t+1}(\theta_{t-s}^{i})-F_{t+1}(\theta_{t+1}^*)\right]\\
&\quad-{m\lambda_{t-s}}{\eta_{t-s}}-2(\kappa^2+D^2)\frac{\lambda_{t-s}}{\eta_{t-s}}\mathbb{E}\left[\|\boldsymbol{\check{\theta}}_{t-s}\|^2\right].\label{2C8}
\end{aligned}
\end{equation}
Substituting~\eqref{2C8} into~\eqref{2C7} leads to
\begin{equation}
\begin{aligned}
&2\lambda_{t-s}\mathbb{E}[\langle\boldsymbol{\tilde{\theta}}_{t-s},\boldsymbol{d}_{t-s}(\boldsymbol{\theta}_{t-s})\rangle]\\
&\geq-\frac{\lambda_{t-s}}{\eta_{s}}\frac{8m(\kappa^2+D^2)(s+1)}{(t+2)(t-s+1)}\\
&\quad-2\left(\kappa^2+D^2+L^2\right)\frac{\lambda_{t-s}}{\eta_{t-s}}\mathbb{E}\left[\|\boldsymbol{\check{\theta}}_{t-s}\|^2\right]\\
&\quad+2m\lambda_{t-s}\mathbb{E}\left[F_{t+1}(\theta_{t-s}^{i})-F_{t+1}(\theta_{t+1}^*)\right]\\
&\quad-\lambda_{t-s}(\eta_{s}+\eta_{t-s})\mathbb{E}\left[\|\boldsymbol{\tilde{\theta}}_{t-s}\|^2\right]-{m\lambda_{t-s}}{\eta_{t-s}}\\
&\quad-2\lambda_{t-s}\left(\frac{2m\kappa^2}{\eta_{t-s}(t+1)}+\frac{m\kappa R}{\sqrt{t+1}}\right).\label{2C9}
\end{aligned}
\end{equation}

Incorporating~\eqref{2C9} into~\eqref{2C1}, we arrive at~\eqref{2Cresult} in Lemma~\ref{C2}.
\end{proof}

\subsection{Proof of Theorem~\ref{T1}}\label{AppendixC}
The proof is divided into three steps: in Step 1), we simplify the result in Lemma~\ref{C1} to obtain~\eqref{1T6}; in Step 2), we iterate~\eqref{1T6} from $0$ to $t$ to derive~\eqref{1T8}; in Step 3), we estimate an upper bound on each term on the right hand side of~\eqref{1T8} and obtain~\eqref{1T15}.

1) To simplify the result in Lemma~\ref{C1}, we first prove that the sum of the following three terms in~\eqref{C1result} is negative:
\begin{equation}
\begin{aligned}
&2\gamma_{t}\mathbb{E}[\boldsymbol{\tilde{\theta}}_{t}^{T}(W\otimes I_{n})\boldsymbol{\tilde{\theta}}_{t}]+3\gamma_{t}^2\mathbb{E}[\|(W\otimes I_{n})\boldsymbol{\tilde{\theta}}_{t}\|^2]\\
&\quad+\lambda_{t}\left(\mu+\frac{8L^2}{\mu}\right)\mathbb{E}[\|\boldsymbol{\check{\theta}}_{t}\|^2]\leq 0. \label{1T1}
\end{aligned}
\end{equation}

Given $\gamma_{t}\leq\gamma_{0}\leq\frac{1}{-3\delta_{m}}$ in the statement of Theorem~\ref{T1}, we have $\gamma_{t}\delta_{i}+3\gamma_{t}^2\delta_{i}^2\leq 0,~\forall i\in[m],$
which implies
\begin{equation}
\gamma_{t}\mathbb{E}[\boldsymbol{\tilde{\theta}}_{t}^{T}(W\otimes  I_{n})\boldsymbol{\tilde{\theta}}_{t}]+3\gamma_{t}^2\mathbb{E}[\|(W\otimes I_{n})\boldsymbol{\tilde{\theta}}_{t}\|^2]\leq 0.\label{1T3}
\end{equation}

By using the relation $\boldsymbol{\check{\theta}}_{t}=\boldsymbol{\theta}_{t}-\boldsymbol{\theta}_{t}^*-(\boldsymbol{\bar{\theta}}_{t}-\boldsymbol{\theta}_{t}^*)$, we obtain
\begin{equation}
\boldsymbol{\check{\theta}}_{t}
=\boldsymbol{\tilde{\theta}}_{t}-\left(\left(\frac{\boldsymbol{1}_{m}\boldsymbol{1}_{m}^{T}}{m}\otimes I_{n}\right)\boldsymbol{\theta}_{t}-\boldsymbol{1}_{m}\left(\frac{\boldsymbol{1}_{m}^{T}\boldsymbol{1}_{m}}{m}\right)\otimes \theta_{t}^*\right).\nonumber
\end{equation}
Given $\boldsymbol{1}_{m}\left(\frac{\boldsymbol{1}_{m}^{T}\boldsymbol{1}_{m}}{m}\right)\otimes \theta_{t}^*=\left(\frac{\boldsymbol{1}_{m}\boldsymbol{1}_{m}^{T}}{m}\otimes I_{n}\right)(\boldsymbol{1}_{m}\otimes \theta^*_{t})$, we have
\begin{equation*}
\boldsymbol{\check{\theta}}_{t}
=\left(I_{mn}-\left(\frac{\boldsymbol{1}_{m}\boldsymbol{1}_{m}^{T}}{m}\otimes I_{n}\right)\right)\boldsymbol{\tilde{\theta}}_{t},
\end{equation*}
which further leads to
\begin{equation} \boldsymbol{\tilde{\theta}}_{t}=\boldsymbol{\check{\theta}}_{t}+\left(\frac{\boldsymbol{1}_{m}\boldsymbol{1}_{m}^{T}}{m}\otimes I_{n}\right)\boldsymbol{\tilde{\theta}}_{t}.\nonumber
\vspace{-0.5em}
\end{equation}
By using  $\boldsymbol{\check{\theta}}_{t}^{T}(W\otimes I_{n})\boldsymbol{\check{\theta}}_{t}\leq \delta_{2}\|\boldsymbol{\check{\theta}}_{t}\|^2$, $\boldsymbol{1}^{T}W=\boldsymbol{0}^{T}$, and $W\boldsymbol{1}=\boldsymbol{0}$, we obtain
\begin{equation}
\gamma_{t}\boldsymbol{\tilde{\theta}}_{t}^{T}(W\otimes I_{n})\boldsymbol{\tilde{\theta}}_{t}=\gamma_{t}\boldsymbol{\check{\theta}}_{t}^{T}(W\otimes I_{n})\boldsymbol{\check{\theta}}_{t}\leq \gamma_{t}\delta_{2}\|\boldsymbol{\check{\theta}}_{t}\|^2.\label{checkee}
\end{equation}
Noting the relationships $\lambda_{t}\leq\lambda_{0}\leq\frac{-\gamma_{0}\delta_{2}\mu}{\mu^2+8L^2}$ and $\gamma_{t}\leq \gamma_{0}\leq\frac{1}{-3\delta_{m}}$ with $v>u$ from the statement of Theorem~\ref{T1}, we have
\begin{equation}
\gamma_{t}\mathbb{E}[\boldsymbol{\tilde{\theta}}_{t}^{T}(W\otimes I_{n})\boldsymbol{\tilde{\theta}}_{t}]+\lambda_{t}\left(\mu+\frac{8L^2}{\mu}\right)\mathbb{E}[\|\boldsymbol{\check{\theta}}_{t}\|^2]\leq 0.\label{1T4}
\end{equation}
Combining~\eqref{1T3} with~\eqref{1T4} yields~\eqref{1T1}.

By using the inequality~\eqref{1T4}, we can rewrite \eqref{C1result} as follows:
\begin{equation}
\mathbb{E}[\|\boldsymbol{\tilde{\theta}}_{t+1}\|^2]\leq \left(1-\frac{\mu\lambda_{t}}{8}\right)\mathbb{E}[\|\boldsymbol{\tilde{\theta}}_{t}\|^2]+\Delta_{t},\label{1T6}
\end{equation}
where the term $\Delta_{t}$ in~\eqref{1T6} is given by
\begin{equation}
\begin{aligned}
\Delta_{t}&\triangleq\left(1+\frac{\mu\lambda_{t}}{8}\right)\frac{12m\kappa^2\lambda_{t}}{\mu(t+1)}\\
&\quad+\left(1+\frac{8}{\lambda_{t}\mu}\right)\mathbb{E}[\|\boldsymbol{\theta}_{t+1}^*-\boldsymbol{\theta}_{t}^*\|^2]\\
&\quad+\left(3+\frac{3\mu\lambda_{t}}{8}\right)\left(\lambda_{t}^2\mathbb{E}[\|\boldsymbol{d}_{t}(\boldsymbol{\theta}_{t})\|^2]+\gamma_{t}^2\|\boldsymbol{\sigma}_{t}\|^2\right).\label{D1}
\end{aligned}
\end{equation}

2) Iterating~\eqref{1T6} from $0$ to $t$, we arrive at
\begin{equation}
\begin{aligned}
\mathbb{E}[\|\boldsymbol{\tilde{\theta}}_{t+1}\|^2]&\leq \prod_{p=0}^{t}\left(1-\frac{\mu\lambda_{p}}{8}\right)\mathbb{E}[\|\boldsymbol{\tilde{\theta}}_{0}\|^2]\\
&\quad+\sum_{p=1}^{t}\prod_{q=p}^{t}\left(1-\frac{\mu\lambda_{q}}{8}\right)\Delta_{p-1}+\Delta_{t}.\label{1T7}
\end{aligned}
\end{equation}
Since $\text{ln}(1-u)\leq -u$ holds for all $u>0$, and, hence, we have
$\prod_{p=0}^{t}\left(1-\frac{\mu\lambda_{p}}{8}\right)\leq e^{-\frac{1}{8}\mu\sum_{p=0}^{t}\lambda_{p}}.$
Then, the inequality~\eqref{1T7} can be rewritten as follows:
\begin{equation}
\begin{aligned}
\mathbb{E}[\|\boldsymbol{\tilde{\theta}}_{t+1}\|^2]&\leq e^{-\frac{1}{8}\mu\sum_{p=0}^{t}\lambda_{p}}\mathbb{E}[\|\boldsymbol{\tilde{\theta}}_{0}\|^2]\\
&\quad+\sum_{p=1}^{t}\Delta_{p-1}e^{-\frac{1}{8}\mu\sum_{q=p}^{t}\lambda_{q}}+\Delta_{t},\label{1T8}
\end{aligned}
\end{equation}
where the term $\Delta_{t}$ is given in~\eqref{D1}.

3) We proceed to estimate an upper bound on the right hand side of~\eqref{1T8}. 

By using the relationships $e^{-\frac{\mu}{8}\sum_{q=p}^{t}\lambda_{q}} \leq e^{-\frac{\mu}{8}\sum_{q=\lceil\frac{t}{2} \rceil}^{t}\lambda_{q}}$ valid for all $p\in[1,\lceil\frac{t}{2} \rceil]$ and $e^{-\frac{\mu}{8}\sum_{q=\lceil\frac{t}{2} \rceil+1}^{t}\lambda_{q}}<1$, the last two terms on the right hand side of~\eqref{1T8} satisfies
\begin{equation}
\begin{aligned}
&\sum_{p=1}^{t}\Delta_{p-1}e^{-\frac{\mu}{8}\sum_{q=p}^{t}\lambda_{q}}+\Delta_{t}\\
&\leq \sum_{p=1}^{\lceil\frac{t}{2} \rceil}\Delta_{p-1}e^{-\frac{\mu}{8}\sum_{q=\lceil\frac{t}{2} \rceil}^{t}\lambda_{q}}+\sum_{p=\lceil\frac{t}{2} \rceil+1}^{t}\Delta_{p-1}+\Delta_{t}\\
&\leq \sum_{p=0}^{\lfloor\frac{t}{2} \rfloor}\Delta_{p}e^{-\frac{\mu}{8}\sum_{q=\lceil\frac{t}{2} \rceil}^{t}\lambda_{q}}+\sum_{p=\lceil\frac{t}{2} \rceil}^{t}\Delta_{p}.\label{1T9}
\end{aligned}
\end{equation}

Next, we estimate an upper bound on $\Delta_{t}$ in~\eqref{D1}:

(a) Since $\lambda_{t}\leq \lambda_{0}$ always holds, we have 
\begin{equation}
\left(1+\frac{\mu\lambda_{t}}{8}\right)\frac{12m\kappa^2\lambda_{t}}{\mu(t+1)}\leq c_{1}\frac{\lambda_{t}}{t+1}, \label{D11}
\end{equation}
where $c_{1}$ is given by $c_{1}=\frac{12m\kappa^2}{\mu}(1+\frac{\mu\lambda_{0}}{8})$.

(b) By using~\eqref{L1result} in Lemma~\ref{L1}, we have
\begin{equation}
\left(1+\frac{8}{\lambda_{t}\mu}\right)\mathbb{E}\left[\|\boldsymbol{\theta}_{t+1}^*-\boldsymbol{\theta}_{t}^*\|^2\right]\leq \frac{c_{2}}{\lambda_{t}(t+1)^2},\label{D12}
\end{equation}
where $c_{2}$ is given by $c_{2}=\frac{16m(\kappa^2+D^2)(\lambda_{0}+8)}{\mu}(\frac{2}{\mu^2}+\frac{1}{L^2})$.

(c) Assumption~\ref{A1}-(iii) and Assumption~\ref{A3}-(ii) imply $\mathbb{E}[\|\nabla l(\theta_{t}^{i},\xi^{i})\|^2] \leq 2(\kappa^2+D^2)$, which further leads to 
\begin{equation}
\left(3+\frac{3\mu\lambda_{t}}{8}\right)\lambda_{t}^2\mathbb{E}\left[\|\boldsymbol{d}_{t}(\boldsymbol{\theta}_{t})\|^2\right]\leq c_{3}\lambda_{t}^2,\label{D13}
\end{equation}
where $c_{3}$ is given by $c_{3}=6m(1+\frac{\lambda_{0}\mu}{8})(\kappa^2+D^2).$

(d) We denote $\varsigma\triangleq\max_{i\in[m]}\{\varsigma^{i}\}$ and $\sigma^{+}\triangleq\max_{i\in[m]}\{\sigma^{i}\}$. Then, we have
\begin{equation}
\left(3+\frac{3\mu\lambda_{t}}{8}\right)\gamma_{t}^2\|\boldsymbol{\sigma}_{t}\|^2\leq c_{4}\gamma_{t}^2(t+1)^{2\varsigma},\label{D14}
\end{equation}
where $c_{4}$ is given by $c_{4}=m(\sigma^{+})^2(3+\frac{3\mu\lambda_{0}}{8}).$

By substituting~\eqref{D11}-\eqref{D14} into~\eqref{D1}, we obtain
\begin{equation}
\Delta_{t}\leq \frac{c_{1}\lambda_{t}}{t+1}+\frac{c_{2}}{\lambda_{t}(t+1)^2}+c_{3}\lambda_{t}^2+c_{4}\gamma_{t}^2(t+1)^{2\varsigma}.\label{D1upper}
\end{equation}

Substituting~\eqref{D1upper} into the second term on the right hand side of~\eqref{1T9}, one yields
\begin{equation}
\begin{aligned}
\sum_{p=\lceil\frac{t}{2} \rceil}^{t}\Delta_{p}
&\leq \sum_{p=\lceil\frac{t}{2} \rceil}^{\infty}\frac{c_{1}\lambda_{p}}{p+1}+\sum_{p=\lceil\frac{t}{2} \rceil}^{\infty}\frac{c_{2}}{\lambda_{p}(p+1)^2}\\
&\quad+\sum_{p=\lceil\frac{t}{2} \rceil}^{\infty}c_{3}\lambda_{p}^2+\sum_{p=\lceil\frac{t}{2} \rceil}^{\infty}c_{4}\gamma_{p}^2(p+1)^{2\varsigma}.\label{1T10}
\end{aligned}
\end{equation}

Recalling the definition $\lambda_{p}=\frac{\lambda_{0}}{(p+1)^{v}}$, we have
\begin{equation}
\sum_{p=\lceil\frac{t}{2} \rceil}^{\infty}\frac{c_{1}\lambda_{p}}{p+1}\leq c_{1}\lambda_{0}\int_{\lceil\frac{t}{2} \rceil}^{\infty}\frac{1}{x^{1+v}}dx\leq \frac{c_{1}\lambda_{0}2^{v}}{vt^{v}}.\label{1T10i}
\end{equation}
Following an argument similar to that of~\eqref{1T10i}, we can derive that the following inequalities always hold:
\begin{equation}
\left\{\begin{aligned}
&\sum_{p=\lceil\frac{t}{2} \rceil}^{\infty}\frac{c_{2}}{\lambda_{p}(p+1)^2}\leq \frac{c_{2}}{\lambda_{0}}\int_{\lceil\frac{t}{2} \rceil}^{\infty}\frac{1}{x^{2-v}}dx\leq \frac{c_{2}2^{1-v}}{(1-v)\lambda_{0}t^{1-v}},\\
&\sum_{p=\lceil\frac{t}{2} \rceil}^{\infty}c_{3}\lambda_{p}^2\leq c_{3}\lambda_{0}^2\int_{\lceil\frac{t}{2} \rceil}^{\infty}\frac{1}{x^{2v}}dx\leq \frac{c_{3}\lambda_{0}^22^{2v-1}}{(2v-1)t^{2v-1}},\\
&\sum_{p=\lceil\frac{t}{2} \rceil}^{\infty}c_{4}\gamma_{p}^2(p+1)^{2\varsigma}\leq \frac{c_{4}\gamma_{0}^22^{2u-2\varsigma-1}}{(2u-2\varsigma-1)t^{2u-2\varsigma-1}},\label{1T10ii}
\end{aligned}
\right.
\end{equation}
where we have used $\sum_{p=\lceil\frac{t}{2} \rceil}^{\infty}\gamma_{p}^2(p+1)^{2\varsigma}\leq \gamma_{0}^2\int_{\lceil\frac{t}{2} \rceil}^{\infty}\frac{1}{x^{2u-2\varsigma}}dx$ and $\max_{i\in[m]}\{\varsigma^{i}\}+\frac{1}{2}<u$ in the last inequality.

Substituting~\eqref{1T10i}-\eqref{1T10ii} into~\eqref{1T10}, we have
\begin{equation}
\begin{aligned}
&\sum_{p=\lceil\frac{t}{2} \rceil}^{t}\Delta_{p}\leq \sum_{p=\lceil\frac{t}{2} \rceil}^{\infty}\Delta_{p}\leq \frac{c_{1}\lambda_{0}2^{v}}{vt^{v}}+\frac{c_{2}2^{1-v}}{(1-v)\lambda_{0}t^{1-v}}\\
&\quad +\frac{c_{3}\lambda_{0}^22^{2v-1}}{(2v-1)t^{2v-1}}+\frac{c_{4}\gamma_{0}^22^{2u-2\varsigma-1}}{(2u-2\varsigma-1)t^{2u-2\varsigma-1}},\label{1T10upper}
\end{aligned}
\end{equation}
where the positive constants $c_{1}$, $c_{2}$, $c_{3}$, and $c_{4}$ are given in~\eqref{D11},~\eqref{D12},~\eqref{D13}, and~\eqref{D14}, respectively.

Based on the result in~\eqref{1T10upper}, we proceed to characterize the first term on the right hand side of~\eqref{1T9}. To this end, we first characterize the term $e^{-\frac{\mu}{8}\sum_{q=\lceil\frac{t}{2} \rceil}^{t}\lambda_{q}}$. By using the inequality $(t+1)^{v}\leq 3(t-1)^{v}$ valid for all $t\geq 2$ and the fact $t-\lceil \frac{t}{2} \rceil=\lfloor \frac{t}{2} \rfloor$, we can obtain 
\vspace{-0.3em}
\begin{flalign}
\sum_{q=\lceil\frac{t}{2} \rceil}^{t}\lambda_{q}&\geq \int_{\lceil\frac{t}{2} \rceil}^{t}\frac{\lambda_{0}}{(x+1)^{v}}dx=\frac{\lambda_{0}}{1-v}(x+1)^{1-v}\Big{|}_{\lceil\frac{t}{2} \rceil}^{t}\nonumber\\
&\geq \frac{\lambda_{0}}{(1-v)^2}\times \xi^{-v}\times \left\lfloor\frac{t}{2} \right\rfloor\geq\frac{\lambda_{0}(t-1)^{1-v}}{6(1-v)^2},\label{1T11}
\end{flalign}
with $\xi\in\left(\left\lceil\frac{t}{2} \right\rceil+1,t+1\right)$.
By combining the relations $e^{-x}<\frac{1}{x}$ valid for all $x>0$, $(t-1)^{v-1}\!\leq\!2^{1-v}t^{v-1}$ valid for all $t\geq 2$, and the inequality~\eqref{1T11}, we have
\begin{equation}
e^{-\frac{\mu}{8}\sum_{q=\lceil\frac{t}{2} \rceil}^{t}\lambda_{q}}<\frac{48(1-v)^22^{1-v}t^{v-1}}{\mu\lambda_{0}},\label{1T122}
\end{equation}
for all $t\geq 2$. Moreover, when we consider the case $t=1$ (i.e., $q=1$), the relationship $e^{-\frac{\mu}{8}\sum_{q=1}^{1}\lambda_{q}}<1$ always holds. Combining $e^{-\frac{\mu}{8}\sum_{q=1}^{1}\lambda_{q}}<1$ and~\eqref{1T122}, we obtain the following inequality for all $t\geq 1$:
\begin{equation}
e^{-\frac{\mu}{8}\sum_{q=\lceil\frac{t}{2} \rceil}^{t}\lambda_{q}}<\frac{c_{5}}{t^{1-v}},\label{1T12}
\end{equation}
where $c_{5}$ is given by $c_{5}=\max\{1,\frac{48(1-v)^22^{1-v}}{\mu\lambda_{0}}\}$.

Substituting~\eqref{1T12} into the first term on the right hand side of~\eqref{1T9} yields
\begin{equation}
\sum_{p=0}^{\lfloor\frac{t}{2} \rfloor}\Delta_{p}e^{-\frac{\mu}{8}\sum_{q=\lceil\frac{t}{2} \rceil}^{t}\lambda_{q}}<\left(\Delta_{0}+\sum_{p=1}^{\infty}\Delta_{p}\right)\frac{c_{5}}{t^{1-v}}.\label{1T13}
\end{equation}
By using the relationship $\Delta_{0}\leq c_{1}\lambda_{0}+\frac{c_{2}}{\lambda_{0}}+c_{3}\lambda_{0}^2+c_{4}\gamma_{0}^2$ derived from~\eqref{D1upper} and the upper bound obtained in~\eqref{1T10upper}, we can rewrite~\eqref{1T13} as
\begin{equation}
\begin{aligned}
\sum_{p=0}^{\lfloor\frac{t}{2} \rfloor}\Delta_{p}e^{-\frac{\mu}{8}\sum_{q=\lceil\frac{t}{2} \rceil}^{t}\lambda_{q}}<\frac{c_{6}}{t^{1-v}},\label{1T131}
\end{aligned}
\end{equation}
where the positive constant $c_{6}$ is given by $c_{6}=c_{5}[c_{1}\lambda_{0}+\frac{c_{2}}{\lambda_{0}}+c_{3}\lambda_{0}^2+c_{4}\gamma_{0}^2+\frac{c_{1}\lambda_{0}2^{v}}{v}+\frac{c_{2}2^{1-v}}{(1-v)\lambda_{0}}+\frac{c_{3}\lambda_{0}^22^{2v-1}}{(2v-1)}+\frac{c_{4}\gamma_{0}^22^{2u-2\varsigma-1}}{(2u-2\varsigma-1)}]$.

By substituting~\eqref{1T10upper} and \eqref{1T131} into~\eqref{1T9}, we have
\begin{equation}
\begin{aligned}
&\sum_{p=1}^{t}\Delta_{p-1}e^{-\frac{\mu}{8}\sum_{q=p}^{t}\lambda_{q}}+\Delta_{t}\\
&\leq \frac{c_{1}\lambda_{0}2^{v}}{vt^{v}}+\frac{c_{2}2^{1-v}}{(1-v)\lambda_{0}t^{1-v}}+\frac{c_{3}\lambda_{0}^22^{2v-1}}{(2v-1)t^{2v-1}}\\
& \quad+\frac{c_{4}\gamma_{0}^22^{2u-2\varsigma-1}}{(2u-2\varsigma-1)t^{2u-2\varsigma-1}}+\frac{c_{6}}{t^{1-v}}.\label{1T141}
\end{aligned}
\end{equation}
We further incorporate~\eqref{1T141} into~\eqref{1T8} to arrive at
\begin{flalign}
&\mathbb{E}[\|\boldsymbol{\tilde{\theta}}_{t+1}\|^2]\!\leq\! e^{-\frac{1}{8}\mu\sum_{p=0}^{t}\lambda_{p}}\mathbb{E}[\|\boldsymbol{\tilde{\theta}}_{0}\|^2]\!+\!\frac{c_{1}\lambda_{0}2^{v}}{vt^{v}}\!+\!\frac{c_{2}2^{1-v}}{(1-v)\lambda_{0}t^{1-v}}\nonumber\\
&\quad+\frac{c_{3}\lambda_{0}^22^{2v-1}}{(2v-1)t^{2v-1}}\!+\!\frac{c_{4}\gamma_{0}^22^{2u-2\varsigma-1}}{(2u-2\varsigma-1)t^{2u-2\varsigma-1}} \!+\!\frac{c_{6}}{t^{1-v}}.\label{1T14}
\end{flalign}
Using the relation $(t+1)^{v}\leq 2t^{v}$ valid for all $t>0$, we have $$\sum_{p=0}^{t}\lambda_{p}\geq \int_{0}^{t}\frac{\lambda_{0}}{(x+1)^{v}}dx>\frac{\lambda_{0}t}{(1-v)^2(t+1)^{v}}\geq\frac{\lambda_{0}t^{1-v}}{2(1-v)^2},$$ which further leads to 
\begin{equation}
e^{-\frac{\mu}{8}\sum_{p=0}^{t}\lambda_{p}}< \frac{1}{\frac{\mu}{8}\sum_{p=0}^{t}\lambda_{p}}<\frac{16(1-v)^2}{\mu\lambda_{0}t^{1-v}}.\label{yaosi}
\end{equation}

Incorporating~\eqref{yaosi} into~\eqref{1T14}, we arrive at
\begin{flalign}
&\mathbb{E}[\|\boldsymbol{\tilde{\theta}}_{t+1}\|^2]\leq \frac{16(1-v)^2\mathbb{E}[\|\boldsymbol{\tilde{\theta}}_{0}\|^2]}{\mu\lambda_{0}t^{1-v}}\!+\!\frac{c_{1}\lambda_{0}2^{v}}{vt^{v}}\!+\!\frac{c_{2}2^{1-v}}{(1-v)\lambda_{0}t^{1-v}}\nonumber\\
&\quad+\frac{c_{3}\lambda_{0}^22^{2v-1}}{(2v-1)t^{2v-1}} \!+\!\frac{c_{4}\gamma_{0}^22^{2u-2\varsigma-1}}{(2u-2\varsigma-1)t^{2u-2\varsigma-1}}\!+\!\frac{c_{6}}{t^{1-v}},\label{1T15}
\end{flalign}
which implies~\eqref{T1result} in Theorem~\ref{T1} since $\min\{1-v,v,2v-1,2u-2\varsigma-1\}=\min\{1-v,2u-2\varsigma-1\}$ always holds.

\subsection{Proof of Theorem~\ref{T2}}\label{AppendixD}
The proof is divided into three steps: in Step 1), we simplify the result in Lemma~\ref{C2} to obtain~\eqref{2T5}; in Step 2), we estimate an upper bound on the item on the right hand side of~\eqref{2T5} and obtain~\eqref{2T8}; in Step 3), we characterize~\eqref{2T8} to arrive at~\eqref{2t2}.

1) Given $\gamma_{t-s}\leq\frac{1}{-3\delta_{m}}$, $\lambda_{t-s}\leq \frac{-\delta_{2}\gamma_{0}}{2(L^2+\kappa^2+D^2)}$, and $v>\frac{1+2u}{3}$ in the statement of Theorem~\ref{T2}, the sum of the following three terms in~\eqref{2Cresult} is negative, whose proof is similar to that of~\eqref{1T1} and thus is omitted here.
\begin{flalign}
&2\gamma_{t-s}\mathbb{E}[\tilde{\boldsymbol{\theta}}_{t-s}^{T}(W\otimes I_{n})\boldsymbol{\tilde{\theta}}_{t-s}]+3\gamma_{t-s}^2\mathbb{E}[\|(W\otimes I_{n})\boldsymbol{\tilde{\theta}}_{t-s}\|^2]\nonumber\\
&\quad+\frac{\lambda_{t-s}}{\eta_{t-s}}2(L^2+\kappa^2+D^2)\mathbb{E}[\|\boldsymbol{\check{\theta}}_{t-s}\|^2]\leq 0.\label{2T1}
\end{flalign}

Substituting the inequality~\eqref{2T1} into~\eqref{2Cresult} and further summing both sides of~\eqref{2Cresult} from $s=0$ to $s=t$, we obtain
\begin{flalign}
&\sum_{s=0}^{t}\mathbb{E}[\|\boldsymbol{\tilde{\theta}}_{t+1-s}\|]\nonumber\\
&\leq-2m\sum_{s=0}^{t}\lambda_{t-s}\mathbb{E}[F_{t+1}(\theta_{t-s}^{i})-F_{t+1}(\theta_{t+1}^*)]\nonumber\\
&+\sum_{s=0}^{t}\left(1+\frac{\lambda_{t-s}}{(s+1)^{r}}+\frac{\lambda_{t-s}}{(t-s+1)^{r}}\right)\mathbb{E}[\|\boldsymbol{\tilde{\theta}}_{t-s}\|^2]\nonumber\\
&+3m(\sigma^{+})^2\sum_{s=0}^{t}\gamma_{t-s}^2(t-s+1)^{2\varsigma}+m\sum_{s=0}^{t}\frac{\lambda_{t-s}}{(t-s+1)^{r}}\nonumber\\
&+8m(\kappa^2\!+\!D^2)\sum_{s=0}^{t}\frac{\lambda_{t-s}(s+1)^{r+1}}{(t+2)(t-s+1)}\nonumber\\
&+4m\kappa^2\sum_{s=0}^{t}\frac{\lambda_{t-s}(t-s+1)^{r}}{t+1}+2m\kappa R\sum_{s=0}^{t}\frac{\lambda_{t-s}}{\sqrt{t+1}},\label{2T3}
\end{flalign}
where in the derivation we have used the definition $\eta_{s}=\frac{1}{(s+1)^{r}}$ with $r=\frac{1-v}{2}$ from the statement of Lemma~\ref{C2}.

We proceed to characterize the second term on the right hand side of~\eqref{2T3}. To this end, we make the following decomposition:
\begin{equation}	
\sum_{s=0}^{t}\mathbb{E}[\|\boldsymbol{\tilde{\theta}}_{t+1-s}\|^2]\!=\!\sum_{s=0}^{t}\mathbb{E}[\|\boldsymbol{\tilde{\theta}}_{t-s}\|^2]\!+\!\mathbb{E}[\|\boldsymbol{\tilde{\theta}}_{t+1}\|^2]-\mathbb{E}[\|\boldsymbol{\tilde{\theta}}_{0}\|^2].\label{shabi}
\end{equation}
By using~\eqref{shabi}, we have
\begin{equation}
\begin{aligned}
&\sum_{s=0}^{t}\left(1+\frac{\lambda_{t-s}}{(s+1)^{r}}+\frac{\lambda_{t-s}}{(t+1-s)^{r}}\right)\mathbb{E}[\|\boldsymbol{\tilde{\theta}}_{t-s}\|^2]\\
&\quad-\sum_{s=0}^{t}\mathbb{E}[\|\boldsymbol{\tilde{\theta}}_{t+1-s}\|^2]\\
&\leq \sum_{s=0}^{t}\left(\frac{\lambda_{t-s}}{(s+1)^{r}}+\frac{\lambda_{t-s}}{(t+1-s)^{r}}\right)\mathbb{E}[\|\boldsymbol{\tilde{\theta}}_{t-s}\|^2]\\
&\quad-\mathbb{E}[\|\boldsymbol{\tilde{\theta}}_{t+1}\|^2]+\mathbb{E}[\|\boldsymbol{\tilde{\theta}}_{0}\|^2]\\
&\leq \sum_{s=1}^{t}\left(\frac{\lambda_{t-s}}{(s+1)^r}+\frac{\lambda_{t-s}}{(t+1-s)^{r}}\right)\mathbb{E}[\|\boldsymbol{\tilde{\theta}}_{t+1-s}\|^2]\\
&\quad-\mathbb{E}[\|\boldsymbol{\tilde{\theta}}_{t+1}\|^2]+\left(1+\frac{\lambda_{0}}{(t+1)^r}+\lambda_{0}\right)\mathbb{E}[\|\boldsymbol{\tilde{\theta}}_{0}\|^2],\label{2T41}
\end{aligned}
\end{equation}
where we have used the relation $\sum_{s=0}^{t}\eta_{t-s}=\sum_{s=1}^{t}\eta_{t+1-s}+\eta_{0}$ valid for any sequence $\{\eta_{p}\},p\in{\mathbb{N}}$ in the last inequality.

Furthermore, since $\theta_{t}^{i}$ is restricted in a compact set $\Theta$, we have 
$\mathbb{E}[\|\boldsymbol{\theta}_{p}-\boldsymbol{\boldsymbol{\theta}}_{t+1}^*\|^2]\leq mR^2$ valid for all $p\geq 0$. By using the relationship $\mathbb{E}[F_{t+1}(\theta_{t-s}^{i})-F_{t+1}(\theta_{t+1}^*)]\geq \mathbb{E}[F_{t+1}(\theta_{t+1}^{i})-F_{t+1}(\theta_{t+1}^*)]$, we can substitute~\eqref{2T41} into~\eqref{2T3} and omit the negative term $-\mathbb{E}[\|\tilde{\boldsymbol{\theta}}_{t+1}\|^2]$ to obtain 
\begin{equation}
2\sum_{s=0}^{t}\lambda_{t-s}\mathbb{E}[F_{t+1}(\theta_{t+1}^{i})-F_{t+1}(\theta_{t+1}^*)]\leq \Delta_{t},\label{2T5}
\end{equation}
where 
\begin{flalign}
&\Delta_{t}=\sum_{s=1}^{t}\left(\frac{\lambda_{t-s}}{(s+1)^r}+\frac{\lambda_{t-s}}{(t+1-s)^{r}}\right)R^2\nonumber\\
&\quad+\left(1+\frac{\lambda_{0}}{(t+1)^r}+\lambda_{0}\right)R^2+3(\sigma^{+})^2\sum_{s=0}^{t}\gamma_{t-s}^2(t-s+1)^{2\varsigma}\nonumber\\
&\quad+8(\kappa^2+D^2)\sum_{s=0}^{t}\frac{\lambda_{t-s}(s+1)^{r+1}}{(t+2)(t-s+1)}\!+\!\sum_{s=0}^{t}\frac{\lambda_{t-s}}{(t+1-s)^r}\nonumber\\
&\quad+4\kappa^2\sum_{s=0}^{t}\frac{\lambda_{t-s}(t-s+1)^{r}}{t+1}+2\kappa R\sum_{s=0}^{t}\frac{\lambda_{t-s}}{\sqrt{t+1}}.\label{2T6}
\end{flalign}

2) Using the relation $\sum_{s=0}^{t}\lambda_{t-s}=\sum_{s=0}^{t}\lambda_{s}$, the first term on the left hand side of~\eqref{2T5} satisfies
\begin{equation}
2\sum_{s=0}^{t}\lambda_{t-s}\geq 2\int_{0}^{t+1}\frac{\lambda_{0}}{(x+1)^{v}}dx=\frac{2\lambda_{0}((t+2)^{1-v}-1)}{1-v}.\label{2T71}
\end{equation}

We proceed with the calculation of the upper bound on~$\Delta_{t}$.

(a) The rearrangement inequality states 
\begin{equation}
\begin{aligned}
x_{1}y_{n}+\cdots+x_{n}y_{1}&\leq x_{1}y_{\sigma(1)}+\cdots+x_{n}y_{\sigma(n)}\\
&\leq x_{1}y_{1}+\cdots+x_{n}y_{n},\nonumber
\end{aligned}
\end{equation}
for all real numbers satisfying $x_{1} \leq \cdots \leq x_{n}$ and $y_{1}\leq \cdots\leq y_{n}$ and for all permutations $y_{\sigma(1)},\cdots,y_{\sigma(1)}$ of $y_{1},\cdots,y_{n}$. 

Therefore, for two decreasing sequences $\lambda_{s-1}$ and $\frac{1}{(s+1)^{r}}$ with $s=1,\cdots,t$, we have
\begin{equation}
\begin{aligned}
&\sum_{s=1}^{t}\frac{\lambda_{t-s}}{({s+1})^{r}}=\sum_{s=1}^{t}\frac{\lambda_{s-1}}{({s+1})^{r}}\leq\sum_{s=1}^{t}\frac{2\lambda_{s}}{({s+1})^{r}}\\
&\leq \int_{1}^{t+1}\frac{2\lambda_{0}}{x^{r+v}}dx\leq \frac{2\lambda_{0}}{1-r-v}((t+1)^{1-r-v}-1).\label{2T6i}
\end{aligned}
\end{equation}

(b) Given the fact $\sum_{s=1}^{t}\frac{\lambda_{t-s}}{({t-s+1})^{r}}\!\!=\!\sum_{s=0}^{t-1}\frac{\lambda_{s}}{({s+1})^{r}}$, we obtain
\begin{equation}
\begin{aligned}
&\sum_{s=1}^{t}\frac{\lambda_{t-s}}{({t-s+1})^{r}}=\sum_{s=0}^{t-1}\frac{\lambda_{s}}{({s+1})^{r}}\leq\sum_{s=0}^{t}\frac{\lambda_{s}}{({s+1})^{r}}\\
& \leq \lambda_{0}+\int_{1}^{t+1}\frac{\lambda_{0}}{x^{r+v}}dx\leq\frac{\lambda_{0}}{1-r-v}(t+1)^{1-r-v}.\label{2T6ii}
\end{aligned}
\end{equation}

(c) By using Assumption~\ref{A4} with $u>\varsigma+\frac{1}{2}$, we have
\begin{equation}
\begin{aligned}
&\sum_{s=0}^{t}\gamma_{t-s}^2(t-s+1)^{2\varsigma}=\sum_{s=0}^{t}\gamma_{s}^2(s+1)^{2\varsigma}\\
&\leq \gamma_{0}^2+\int_{1}^{\infty}\frac{\gamma_{0}^2}{x^{2u-2\varsigma}}dx=\frac{\gamma_{0}^2(2u-2\varsigma)}{2u-2\varsigma-1}.\label{2T6iii}
\end{aligned}
\end{equation}

(d) Given $\lambda_{t}=\frac{\lambda_{0}}{(t+1)^{v}}$ with $v\in(\frac{1}{2})$, we have
\begin{equation}
\begin{aligned}
&\sum_{s=0}^{t}\frac{\lambda_{t-s}(s+1)^{r+1}}{(t+2)(t-s+1)}=\frac{\lambda_{0}}{t+2}\sum_{s=0}^{t}\frac{(s+1)^{r+1}}{(t-s+1)^{1+v}}\\
&=\frac{\lambda_{0}}{t+2}\left(\sum_{s=0}^{t-1}\frac{(s+1)^{r+1}}{(t-s+1)^{1+v}}+(t+1)^{r+1}\right),\label{2T6iv}
\end{aligned}
\end{equation}
which can be further simplified by using the following inequality:
\begin{equation}
\begin{aligned}
&\sum_{s=0}^{t-1}\frac{(s+1)^{r+1}}{(t-s+1)^{1+v}}\leq\int_{0}^t\frac{(x+1)^{r+1}}{(t-x+1)^{1+v}}dx\\
&=\frac{1}{v}\int_{0}^t(x+1)^{r+1}d(t-x+1)^{-v}\\
&=\frac{1}{v}((t+1)^{r+1}-(t+1)^{-v})\\
&\quad\quad-\frac{1}{v}(1+r)\int_{0}^{t}(t-x+1)^{-v}(x+1)^rdx\\
&\leq \frac{1}{v}((t+1)^{r+1}-(t+1)^{-v}).\label{2T6iv1}
\end{aligned}
\end{equation}
We substitute~\eqref{2T6iv1} into~\eqref{2T6iv} to obtain
\begin{equation}
\sum_{s=0}^{t}\frac{\lambda_{t-s}(s+1)^{r+1}}{(t+2)(t-s+1)}\leq \lambda_{0}\left(\frac{1}{v}+1\right)(t+1)^{r}.\label{2T6iv2}
\end{equation}

(e) Using the relation $\sum_{s=0}^{t}\lambda_{t-s}(t-s+1)^{r}=\sum_{s=1}^{t+1}\frac{1}{s^{v-r}}$, we have
\begin{flalign}
&\sum_{s=0}^{t}\frac{\lambda_{t-s}(t-s+1)^{r}}{t+1}=\frac{1}{t+1}\sum_{s=1}^{t+1}\frac{1}{s^{v-r}}\nonumber\\
&\leq \frac{1}{t+1}\int_{0}^{t+1}\frac{1}{x^{v-r}}dx\leq \frac{1}{(1+r-v)(t+1)^{v-r}}.\label{2T6vi}
\end{flalign}

(f) Following an argument similar to that of~\eqref{2T6vi}, we have
\begin{equation}
\sum_{s=0}^{t}\frac{\lambda_{t-s}}{\sqrt{t+1}}=\frac{1}{\sqrt{t+1}}\sum_{s=1}^{t+1}\frac{1}{s^{v}}\leq\frac{1}{(1-v)(t+1)^{v-\frac{1}{2}}}.\label{2T6vii}
\end{equation}

Incorporating~\eqref{2T71}-\eqref{2T6vii} into~\eqref{2T6} and further multiplying both sides of~\eqref{2T5} by $\frac{1-v}{4\lambda_{0}((t+2)^{1-v}-1)}$ yield
\begin{equation}
\begin{aligned}
&\mathbb{E}\left[F_{t+1}(\theta_{t+1}^{i})-F_{t+1}(\theta_{t+1}^*)\right]\\
&\leq\frac{c_{1}(1-v)}{4\lambda_{0}((t+2)^{1-v}-1)}+\frac{R^2(1-v)}{4(t+1)^{r}((t+2)^{1-v}-1)}\\
&\quad+\frac{(3R^2+1)(1-v)}{4(1-r-v)(t+1)^{r+v-1}((t+2)^{1-v}-1)}\\
&\quad+\frac{2(\kappa^2+D^2)(1+v)(1-v)}{v(t+1)^{-r}((t+2)^{1-v}-1)}\\
&\quad+\frac{\kappa R}{2\lambda_{0}(t+1)^{v-\frac{1}{2}}((t+2)^{1-v}-1)}\\
&\quad+\frac{\kappa^2(1-v)}{\lambda_{0}(1+r-v)(t+1)^{v-r}((t+2)^{1-v}-1)},\label{2T8}
\end{aligned}
\end{equation}
with $c_{1}=(1+\lambda_{0})R^2+\frac{3(\sigma^{+})^2\gamma_{0}^2(2u-2\varsigma)}{2u-2\varsigma-1}.$

3) Using that the relation $(t+2)^{1-v}(t+1)^{r}\geq 2^{1-v}(t+1)^{r}$ implies $(t+1)^{r}\leq \frac{1}{2^{1-v}}(t+2)^{1-v}(t+1)^{r}$, we obtain
\begin{equation}
(t+1)^{r}((t+2)^{1-v}-1)\geq \left(1-\frac{1}{2^{1-v}}\right)(t+1)^{1-v+r}.\nonumber
\end{equation}

By using a similar argument for each item on the right hand side of~\eqref{2T8} and substituting $r=\frac{1-v}{2}$ given in the statement of Lemma~\ref{C2} into~\eqref{2T8}, we can arrive at
\begin{equation}
\begin{aligned}
&\mathbb{E}\left[F_{t+1}(\theta_{t+1}^{i})-F_{t+1}(\theta_{t+1}^*)\right]\\
&\leq \frac{c_{2}}{(t+1)^{1-v}}+\frac{c_{3}}{(t+1)^{\frac{3(1-v)}{2}}}+\frac{c_{4}+c_{5}}{(t+1)^{\frac{1-v}{2}}}\\
&\quad+\frac{c_{6}}{(t+1)^{\frac{1}{2}}}+\frac{c_{7}}{(t+1)^{\frac{1+v}{2}}}=\mathcal{O}((t+1)^{-\beta}),\label{2t2}
\end{aligned}
\end{equation}
for any $t\geq 0$, where $\beta$ is given by  $\beta=\frac{1-v}{2}$ and the constants $c_{2}$ to $c_{6}$ are given by $c_{2}=\frac{c_{1}(1-v)2^{1-v}}{4\lambda_{0}(2^{1-v}-1)}$, $c_{3}=\frac{R^2(1-v)2^{1-v}}{4(2^{1-v}-1)}$, $c_{4}=\frac{(3R^2+1)(1-v)2^{1-v}}{4(1-r-v)(2^{1-v}-1)}$, $c_{5}=\frac{2(\kappa^2+D^2)(1-v^2)2^{1-v}}{v(2^{1-v}-1)}$, $c_{6}=\frac{\kappa R2^{1-v}}{2\lambda_{0}(2^{1-v}-1)}$, and $c_{7}=\frac{\kappa^2(1-v)2^{1-v}}{\lambda_{0}(1+r-v)(2^{1-v}-1)}$, respectively. The inequality~\eqref{2t2} implies~\eqref{T2result} in Theorem~\ref{T2}.

\subsection{Proof of Theorem~\ref{T3}}\label{AppendixE}
Note that Lemma~\ref{C1} remains valid under the conditions in Theorem~\ref{T3}. Therefore, we proceed with the proof by utilizing \eqref{C1result} in Lemma~\ref{C1}. The proof is divided into three steps: in Step 1), we simplify the result in Lemma~\ref{C1} to obtain~\eqref{3T6}; in Step 2), we iterate~\eqref{3T6} from $t_{0}$ to $t$ to derive~\eqref{3T8}; and in Step 3), we estimate an upper bound on each term on the right hand side of \eqref{3T8} and obtain~\eqref{3T17}.

1) We first show that the sum of the following three terms in~\eqref{C1result} is negative:
\begin{equation}
\begin{aligned}
&2\gamma_{t}\mathbb{E}[\boldsymbol{\tilde{\theta}}_{t}^{T}(W\otimes I_{n})\boldsymbol{\tilde{\theta}}_{t}]+3\gamma_{t}^2\mathbb{E}[\|(W\otimes I_{n})\boldsymbol{\tilde{\theta}}_{t}\|^2]\\
&\quad+\lambda_{t}\left(\mu+\frac{8L^2}{\mu}\right)\mathbb{E}[\|\boldsymbol{\check{\theta}}_{t}\|^2]\leq 0. \label{3T1}
\end{aligned}
\end{equation}
Given $t_{0}\geq\sqrt[u]{-3\delta_{m}\gamma_{0}}-1$ from the statement of Theorem~\ref{T3}, we have $\gamma_{t_{0}}=\frac{\gamma_{0}}{(t_{0}+1)^{u}}\leq-\frac{1}{3\delta_{m}},$
which further leads to $\gamma_{t_{0}}\delta_{i}+3\gamma_{t_{0}}^2\delta_{i}^2\leq 0,~\forall i\in[m].$ Since $\gamma_{t}$ is a decaying sequence, the following inequality always holds:
\begin{equation}
\gamma_{t}\delta_{i}+3\gamma_{t}^2\delta_{i}^2\leq 0,~i\in[m],\label{tt}
\end{equation} 
for all $t>t_{0}$. Incorporating~\eqref{tt} into the first and second terms of~\eqref{3T1} yields
\begin{equation}
\gamma_{t}\mathbb{E}[\boldsymbol{\tilde{\theta}}_{t}^{T}(W\otimes I_{n})\boldsymbol{\tilde{\theta}}_{t}]+3\gamma_{t}^2\mathbb{E}[\|(W\otimes I_{n})\boldsymbol{\tilde{\theta}}_{t}\|^2]\leq 0.\label{3T2}
\end{equation}
Following an argument similar to that of~\eqref{checkee}, we have
\begin{equation}
\gamma_{t}\boldsymbol{\tilde{\theta}}_{t}^{T}(W\otimes I_{n})\boldsymbol{\tilde{\theta}}_{t}=\gamma_{t}\boldsymbol{\check{\theta}}_{t}^{T}(W\otimes I_{n})\boldsymbol{\check{\theta}}_{t}\leq \gamma_{t}\delta_{2}\|\boldsymbol{\check{\theta}}_{t}\|^2.\label{3T3}
\end{equation}
Noting the relationship $t_{0}\geq\sqrt[v-u]{\frac{(\mu^2+8L^2)\lambda_{0}}{-\delta_{2}\mu\gamma_{0}}}-1$ with $v>u$ given in the statement of Theorem~\ref{T3}, we obtain
\begin{equation}
\frac{\gamma_{t_{0}}}{\lambda_{t_{0}}}=\frac{\gamma_{0}}{\lambda_{0}}(t_{0}+1)^{v-u}\geq \frac{\mu^2+8L^2}{-\delta_{2}\mu},\nonumber
\end{equation}
which further leads to $\gamma_{t}\delta_{2}+\lambda_{t}(\mu+\frac{8L^2}{\mu})\leq 0$ and
\begin{equation}
\gamma_{t}\mathbb{E}[\boldsymbol{\tilde{\theta}}_{t}^{T}(W\otimes I_{n})\boldsymbol{\tilde{\theta}}_{t}]+\lambda_{t}\left(\mu+\frac{8L^2}{\mu}\right)\mathbb{E}[\|\boldsymbol{\check{\theta}}_{t}\|^2]\leq 0.\label{3T4}
\end{equation}
Combining~\eqref{3T3} and~\eqref{3T4} yields~\eqref{3T1}. 

Based on~\eqref{3T1}, the inequality~\eqref{C1result} can be rewritten as follows:
\begin{equation}
\mathbb{E}[\|\boldsymbol{\tilde{\theta}}_{t+1}\|^2]\leq \left(1-\frac{\mu\lambda_{t}}{8}\right)\mathbb{E}[\|\boldsymbol{\tilde{\theta}}_{t}\|^2]+\Delta_{t},\label{3T6}
\end{equation}
for all $t>t_{0}$, where the term $\Delta_{t}$ in~\eqref{3T6} is given by
\begin{flalign}
\Delta_{t}&\triangleq\left(1+\frac{\mu\lambda_{t}}{8}\right)\frac{12m\kappa^2\lambda_{t}}{\mu(t+1)}+\left(1+\frac{8}{\lambda_{t}\mu}\right)\mathbb{E}[\|\boldsymbol{\theta}_{t+1}^*-\boldsymbol{\theta}_{t}^*\|^2]\nonumber\\
&\quad+\left(3+\frac{3\mu\lambda_{t}}{8}\right)\left(\lambda_{t}^2\mathbb{E}[\|\boldsymbol{d}_{t}(\boldsymbol{\theta}_{t})\|^2]+\gamma_{t}^2\|\boldsymbol{\sigma}_{t}\|^2\right).\label{D3}
\end{flalign}

2) Iterating~\eqref{3T6} from $t_{0}$ to $t$, we arrive at
\begin{equation}
\begin{aligned}
&\mathbb{E}[\|\boldsymbol{\tilde{\theta}}_{t+1}\|^2]\leq\prod_{p=t_0}^{t}\left(1-\frac{\mu\lambda_{p}}{8}\right)\mathbb{E}[\|\boldsymbol{\tilde{\theta}}_{t_{0}}\|^2]\\
&\quad+\sum_{p=t_{0}+1}^{t}\prod_{q=p}^{t}\left(1-\frac{\mu\lambda_{q}}{8}\right)\Delta_{p-1}+\Delta_{t}.\label{3T7}
\end{aligned}
\end{equation}

Since $\text{ln}(1-u)\leq-u$ holds for all $u>0$, we have
\begin{equation} \prod_{p=t_{0}}^{t}\left(1-\frac{\mu\lambda_{p}}{8}\right)\leq e^{-\frac{1}{8}\mu\sum_{p=t_{0}}^{t}\lambda_{p}}.\nonumber
\end{equation}
Then, the inequality~\eqref{3T7} can be rewritten as follows:
\begin{equation}
\begin{aligned}
\mathbb{E}[\|\boldsymbol{\tilde{\theta}}_{t+1}\|^2]&\leq e^{-\frac{\mu}{8}\sum_{p=t_{0}}^{t}\lambda_{p}}\mathbb{E}[\|\boldsymbol{\tilde{\theta}}_{t_{0}}\|^2]\\
&\quad\quad+\sum_{p=t_{0}+1}^{t}\Delta_{p-1}e^{-\frac{\mu}{8}\mu\sum_{q=p}^{t}\lambda_{q}}+\Delta_{t},\label{3T8}
\end{aligned}
\end{equation}
where the term $\Delta_{t}$ is given in~\eqref{D3}.

3) We proceed to estimate an upper bound on the right hand side of~\eqref{3T8}.

By using the relationships $e^{-\frac{\mu}{8}\sum_{q=p}^{t}\lambda_{q}}\leq e^{-\frac{\mu}{8}\sum_{q=\lceil\frac{t+t_0}{2}\rceil}^{t}\lambda_{q}}$ valid for any $p\in[t_{0}+1,\lceil\frac{t+t_0}{2}\rceil]$ and $e^{-\frac{\mu}{8}\sum_{q=p}^{t}\lambda_{q}}<1$, the last two terms on the right hand side of~\eqref{3T8} satisfies
\begin{flalign}
&\sum_{p=t_{0}+1}^{t}\Delta_{p-1}e^{-\frac{\mu}{8}\sum_{q=p}^{t}\lambda_{q}}+\Delta_{t}\nonumber\\
&\leq\sum_{p=t_{0}+1}^{\lceil\frac{t+t_0}{2}\rceil}\Delta_{p-1}e^{-\frac{\mu}{8}\sum_{q=\lceil\frac{t+t_{0}}{2} \rceil}^{t}\lambda_{q}}+\sum_{p=\lceil\frac{t+t_0}{2}\rceil+1}^{t}\Delta_{p-1}+\Delta_{t}\nonumber\\
&\leq \sum_{p=t_{0}}^{\left\lfloor \frac{t+t_{0}}{2} \right\rfloor}\Delta_{p}e^{-\frac{\mu}{8}\sum_{q=\lceil\frac{t+t_{0}}{2} \rceil}^{t}\lambda_{q}}+\sum_{p=\left\lceil\frac{t+t_0}{2}\right\rceil}^{t}\Delta_{p}.\label{3T9}
\end{flalign}

Following an argument similar to that of~\eqref{D1}, we have
\begin{equation}
\Delta_{t}\leq \frac{c_{1}\lambda_{t}}{t+1}+\frac{c_{2}}{\lambda_{t}(t+1)^2}+c_{3}\lambda_{t}^2+c_{4}\gamma_{t}^2(t+1)^{2\varsigma},\label{D3upper}
\end{equation}
where the positive constants $c_{1}$, $c_{2}$, $c_{3}$, and $c_{4}$ are given by $c_{1}\triangleq\frac{12m\kappa^2}{\mu}(1+\frac{\mu\lambda_{0}}{8})$, $c_{2}\triangleq\frac{16(\kappa^2+D^2)(\lambda_{0}+8)}{\mu}(\frac{2}{\mu^2}+\frac{1}{L^2})$, $c_{3}\triangleq6m(1+\frac{\lambda_{0}\mu}{8})(\kappa^2+D^2)$, and $c_{4}\triangleq m(\sigma^{+})^2(3+\frac{3\mu\lambda_{0}}{8})$, respectively.

Following an argument similar to that of~\eqref{1T10upper} yields
\begin{equation}
\begin{aligned}
&\sum_{p=\lceil\frac{t+t_{0}}{2} \rceil}^{t}\Delta_{p}\leq \sum_{p=\lceil\frac{t+t_{0}}{2} \rceil}^{\infty}\Delta_{p}\\
&\leq \frac{c_{1}\lambda_{0}2^{v}}{v(t+t_{0})^{v}}+\frac{c_{2}2^{1-v}}{(1-v)\lambda_{0}(t+t_{0})^{1-v}}\\
&\quad +\frac{c_{3}\lambda_{0}^22^{2v-1}}{(2v-1)(t+t_{0})^{2v-1}}+\frac{c_{4}\gamma_{0}^22^{2u-2\varsigma-1}}{(2u-2\varsigma-1)(t+t_{0})^{2u-2\varsigma-1}}.\label{3T10}
\end{aligned}
\end{equation}

Based on the result in~\eqref{3T10}, we proceed to characterize  the first term on the right hand side of~\eqref{3T9}. To this end, we first characterize the term $e^{-\frac{\mu}{8}\sum_{q=\lceil\frac{t+t_{0}}{2} \rceil}^{t}\lambda_{q}}$. By using the inequality  $(t+1)^{v}\leq(t_{0}+3)^{v}(t-t_{0}-1)^{v}$ valid for all $t\geq t_{0}+2$ and the fact $t-\left\lceil\frac{t+t_0}{2}\right\rceil=\left\lfloor \frac{t-t_{0}}{2} \right\rfloor$, we can obtain
\begin{flalign}
&\sum_{q=\left\lceil\frac{t+t_0}{2}\right\rceil}^{t} \lambda_{q}\geq\int_{\left\lceil\frac{t+t_0}{2}\right\rceil}^{t}\frac{\lambda_{0}}{(x+1)^{v}}dx\nonumber\\
&=\frac{\lambda_{0}}{1-v}(x+1)^{1-v}\Bigg{|}_{\left\lceil\frac{t+t_0}{2}\right\rceil}^{t}\nonumber\\
&\geq \frac{\lambda_{0}}{(1-v)^2}\times \xi^{-v} \times \left\lfloor \frac{t-t_{0}}{2} \right\rfloor\nonumber\\
&>\frac{\lambda_{0}(t-t_0-1)}{2(1-v)^2(t+1)^{v}}\geq\frac{\lambda_{0}(t-t_0-1)^{1-v}}{2(t_0+3)^v(1-v)^2},\label{3T11}
\end{flalign}
with $\xi\in(\lceil\frac{t+t_0}{2}\rceil+1,t+1)$. By combining the relations $e^{-x}<\frac{1}{x}$ valid for all $x>0$, $(t-t_{0}-1)^{v-1}\leq 2^{1-v}(t-t_{0})^{v-1}$ valid for all $t\geq t_{0}+2$, and the inequality~\eqref{3T11}, we have
\begin{equation}
\begin{aligned}
&e^{-\frac{\mu}{8}\sum_{q=\lceil\frac{t+t_{0}}{2} \rceil}^{t}\lambda_{q}}\\
&<\frac{48(t_0+3)^v(1-v)^2(t-t_0-1)^{v-1}}{\mu\lambda_{0}}\\
&\leq\frac{48(t_0+3)^v(1-v)^22^{1-v}(t-t_0)^{v-1}}{\mu\lambda_{0}},\label{3T12}
\end{aligned}
\end{equation}
for all $t\geq t_{0}+2$. Moreover, when we consider the case $t=t_{0}+1$ (i.e., $q=t_{0}+1$), the relationship $e^{-\frac{\mu}{8}\mu\sum_{q=t_{0}+1}^{1}\lambda_{q}}<1$ always holds. Combining $e^{-\frac{\mu}{8}\mu\sum_{q=t_{0}+1}^{1}\lambda_{q}}<1$ and~\eqref{3T12}, we obtain the following inequality for any $t>t_{0}$:
\begin{equation}
e^{-\frac{\mu}{8}\sum_{q=\left\lceil\frac{t+t_0}{2}\right\rceil}^{t}\lambda_{q}}<\frac{c_{5}}{(t-t_{0})^{1-v}},\label{3T13}
\end{equation}
where $c_{5}$ is given by $c_{5}=\max\{1,\frac{48(t_{0}+3)^{v}(1-v)^22^{1-v}}{\mu\lambda_{0}}\}$.

Substituting~\eqref{3T13} into the first term on the right hand side of~\eqref{3T9} yields
\begin{equation}
\begin{aligned}
\sum_{p=0}^{\lfloor\frac{t+t_{0}}{2} \rfloor}\Delta_{p}e^{-\frac{\mu}{8}\mu\sum_{q=\lceil\frac{t+t_{0}}{2} \rceil}^{t}\lambda_{q}}\!<\!\left(\Delta_{0}\!+\!\!\!\sum_{p=t_{0}+1}^{\infty}\Delta_{p}\right)\!\frac{c_{5}}{(t-t_{0})^{1-v}}.\label{3T141}
\end{aligned}
\end{equation}
By using the relationship $\Delta_{0}\leq c_{1}\lambda_{0}+\frac{c_{2}}{\lambda_{0}}+c_{3}\lambda_{0}^2+c_{4}\gamma_{0}^2$ derived from~\eqref{D3upper} and the upper bound obtained in~\eqref{3T10}, we can rewritten~\eqref{3T141} as
\begin{equation}
\begin{aligned}
\sum_{p=0}^{\lfloor\frac{t+t_{0}}{2} \rfloor}\Delta_{p}e^{-\frac{\mu}{8}\mu\sum_{q=\lceil\frac{t+t_{0}}{2} \rceil}^{t}\lambda_{q}}<\frac{c_{6}}{(t-t_{0})^{1-v}},\label{3T14}
\end{aligned}
\end{equation}
where the positive constant $c_{6}$ is given by $c_{6}=c_{1}\lambda_{0}+\frac{c_{2}}{\lambda_{0}}+c_{3}\lambda_{0}^2+c_{4}\gamma_{0}^2+\frac{c_{1}\lambda_{0}2^{v}}{v}+\frac{c_{2}2^{1-v}}{(1-v)\lambda_{0}}+\frac{c_{3}\lambda_{0}^22^{2v-1}}{(2v-1)}+\frac{c_{4}\gamma_{0}^22^{2u-2\varsigma-1}}{(2u-2\varsigma-1)}$.

By substituting~\eqref{3T10} and~\eqref{3T14} into~\eqref{3T9}, we have
\begin{flalign}
&\sum_{p=t_{0}+1}^{t}\Delta_{p-1}e^{-\frac{\mu}{8}\sum_{q=p}^{t}\lambda_{q}}+\Delta_{t}\nonumber\\
&\leq \frac{c_{1}\lambda_{0}2^{v}}{v(t+t_{0})^{v}}+\frac{c_{2}2^{1-v}}{(1-v)\lambda_{0}(t+t_{0})^{1-v}}+\frac{c_{3}\lambda_{0}^22^{2v-1}}{(2v-1)(t+t_{0})^{2v-1}}\nonumber\\
&\quad +\frac{c_{4}\gamma_{0}^22^{2u-2\varsigma-1}}{(2u-2\varsigma-1)(t+t_{0})^{2u-2\varsigma-1}}+\frac{c_{6}}{(t-t_{0})^{1-v}}.\label{3T1412}
\end{flalign} 
We further incorporate~\eqref{3T1412} into~\eqref{3T8} to arrive at
\begin{equation}
\begin{aligned}
&\mathbb{E}[\|\boldsymbol{\tilde{\theta}}_{t+1}\|^2]\\
&\leq e^{-\frac{\mu}{8}\sum_{p=t_{0}}^{t}\lambda_{p}}\mathbb{E}[\|\boldsymbol{\tilde{\theta}}_{t_0}\|^2]+\frac{c_{1}\lambda_{0}2^{v}}{v(t+t_{0})^{v}}\\
&\quad+\frac{c_{2}2^{1-v}}{(1-v)\lambda_{0}(t+t_{0})^{1-v}}+\frac{c_{3}\lambda_{0}^22^{2v-1}}{(2v-1)(t+t_{0})^{2v-1}}\\
&\quad +\frac{c_{4}\gamma_{0}^22^{2u-2\varsigma-1}}{(2u-2\varsigma-1)(t+t_{0})^{2u-2\varsigma-1}}+\frac{c_{6}}{(t-t_{0})^{1-v}}.\label{3T15}
\end{aligned}
\end{equation}
Following an argument similar to that of~\eqref{3T11}, we have $\sum_{p=t_{0}}^{t}\lambda_{p}\geq \frac{\lambda_{0}(t-t_{0})^{1-v}}{2(t_{0}+3)^{v}(1-v)^2}$, which further implies
\begin{equation}
e^{-\frac{\mu}{8}\sum_{p=t_{0}}^{t}\lambda_{p}}
<\frac{16(t_0+3)^v(1-v)^2}{\mu\lambda_{0}(t-t_0)^{1-v}}.\label{3T16}
\end{equation}

Incorporating~\eqref{3T16} into~\eqref{3T15}, we arrive at
\begin{flalign}
&\mathbb{E}[\|\boldsymbol{\tilde{\theta}}_{t+1}\|^2]\nonumber\\
&\leq \frac{16(t_{0}+3)^{v}(1-v)^2\mathbb{E}[\|\boldsymbol{\tilde{\theta}}_{t_0}\|^2]}{\mu\lambda_{0}(t-t_{0})^{1-v}}+\frac{c_{1}\lambda_{0}2^{v}}{v(t+t_{0})^{v}}\nonumber\\
&\quad+\frac{c_{2}2^{1-v}}{(1-v)\lambda_{0}(t+t_{0})^{1-v}}+\frac{c_{3}\lambda_{0}^22^{2v-1}}{(2v-1)(t+t_{0})^{2v-1}}\nonumber\\
&\quad +\frac{c_{4}\gamma_{0}^22^{2u-2\varsigma-1}}{(2u-2\varsigma-1)(t+t_{0})^{2u-2\varsigma-1}}+\frac{c_{6}}{(t+t_{0})^{1-v}},\label{3T17}
\end{flalign}
for any $t>t_{0}$, which implies~\eqref{T3result} in Theorem~\ref{T3} since $\min\{1-v,v,2v-1,2u-2\varsigma-1\}=\min\{1-v,2u-2\varsigma-1\}$ always holds.
\subsection{Proof of Theorem~\ref{T4}}\label{AppendixF}
Note that Lemma~\ref{C2} is valid under the conditions in Theorem~\ref{T4}. Hence, we continue our proof of Theorem~\ref{T4} by using \eqref{2Cresult} in Lemma~\ref{C2}. 

The proof is divided into three steps: in Step 1), we simplify the result in Lemma~\ref{C2} to obtain \eqref{4T9}; in Step 2), we estimate an upper bound on the term on the right hand side of \eqref{4T9} and obtain \eqref{4T11}; in Step 3), we characterize \eqref{4T11} to arrive at \eqref{4T12}.

1) We first prove that the sum of the following three terms in~\eqref{2Cresult} is negative:
\begin{flalign}
&2\gamma_{t-s}\mathbb{E}[\tilde{\boldsymbol{\theta}}_{t-s}^{T}(W\otimes I_{n})\boldsymbol{\tilde{\theta}}_{t-s}]+3\gamma_{t-s}^2\mathbb{E}[\|(W\otimes I_{n})\boldsymbol{\tilde{\theta}}_{t-s}\|^2]\nonumber\\
&\quad+\frac{\lambda_{t-s}}{\eta_{t-s}}2(L^2+\kappa^2+D^2)\mathbb{E}[\|\boldsymbol{\check{\theta}}_{t-s}\|^2]\leq 0.\label{4T1}
\end{flalign}

We introduce an auxiliary variable $s\in[0,t-t'_{0}]$, which implies $t-s\in[t'_{0},t]$. 

By using the relation $t'_{0}\geq\sqrt[u]{-3\delta_{m}\gamma_{0}}-1$ given in the statement of Theorem~\ref{T4}, we have 
\begin{equation*}
\gamma_{t-s}\leq \gamma_{t'_0}=\frac{\gamma_{0}}{(t'_{0}+1)^{u}}\leq-\frac{1}{3\delta_{m}},
\end{equation*}
which further leads to
\begin{equation}
3\gamma_{t-s}^2\delta_{i}^2+\gamma_{t-s}\delta_{i}\leq 0,~\forall i\in[m].~\label{4lal}
\end{equation}
Incorporating~\eqref{4lal} into the first and second terms of~\eqref{4T1} yields
\begin{equation}
\gamma_{t-s}\mathbb{E}[\tilde{\boldsymbol{\theta}}_{t-s}^{T}(W\otimes I_{n})\boldsymbol{\tilde{\theta}}_{t-s}]+3\gamma_{t-s}^2\mathbb{E}[\|(W\otimes I_{n})\boldsymbol{\tilde{\theta}}_{t-s}\|^2]\leq 0.\label{4T3}
\end{equation}

Following an argument similar to that of~\eqref{checkee}, we have 
\begin{equation*}
\gamma_{t-s}\tilde{\boldsymbol{\theta}}_{t-s}^{T}(W\otimes I_{n})\boldsymbol{\tilde{\theta}}_{t-s}\leq \gamma_{t-s}\delta_{2}\|\boldsymbol{\check{\theta}}_{t}\|^2.
\end{equation*}

By using the relation $t'_{0}\geq\sqrt[\frac{3v-1}{2}-u]{\frac{2(L^2+\kappa^2+D^2)\lambda_{0}}{-\delta_{2}\gamma_{0}}}-1$ with $3v>2u+1$ given in the statement of Theorem~\ref{T4}, we obtain
\begin{equation}
\frac{\eta_{t'_{0}}\gamma_{t'_{0}}}{\lambda_{t'_{0}}}=\frac{\gamma_{0}}{\lambda_{0}}(t'_{0}+1)^{\frac{3v-1}{2}-u}\geq \frac{2(L^2+\kappa^2+D^2)}{-\delta_{2}},\nonumber
\end{equation}
which further leads to
\begin{equation*}
\delta_{2}\eta_{t'_{0}}\gamma_{t'_{0}}+2(L^2+\kappa^2+D^2)\lambda_{t'_{0}}\leq 0.
\end{equation*}

Since the sequences $\eta_{t-s}$, $\gamma_{t-s}$ and $\lambda_{t-s}$ are all decaying sequences, then the following inequality always holds:
\begin{equation*}
\gamma_{t-s}\delta_{2}+\frac{\lambda_{t-s}}{\eta_{t-s}}2(L^2+\kappa^2+D^2)\leq 0,
\end{equation*}
for all $t\geq t'_{0}$. Hence, we obtain
\begin{flalign}
&\gamma_{t-s}\mathbb{E}[\tilde{\boldsymbol{\theta}}_{t-s}^{T}(W\otimes I_{n})\boldsymbol{\tilde{\theta}}_{t-s}]\nonumber\\
&\quad+\frac{\lambda_{t-s}}{\eta_{t-s}}2(L^2+\kappa^2+D^2)\mathbb{E}[\|\boldsymbol{\check{\theta}}_{t-s}\|^2]\leq 0.\label{4T5}
\end{flalign}
Combining~\eqref{4T3} and~\eqref{4T5} yields~\eqref{4T1}.

Substituting the inequality~\eqref{4T1} into~\eqref{2Cresult}  and further summing both sides of~\eqref{2Cresult} from $s=0$ to $s=t-t'_{0}$, we obtain
\begin{flalign}
&\sum_{s=0}^{t-t'_{0}}\mathbb{E}[\|\boldsymbol{\tilde{\theta}}_{t+1-s}\|]\nonumber\\
&\leq-2m\sum_{s=0}^{t-t'_{0}}\lambda_{t-s}\mathbb{E}[F_{t+1}(\theta_{t-s}^{i})-F_{t+1}(\theta_{t+1}^*)]\nonumber\\
&\quad+\sum_{s=0}^{t-t'_{0}}\left(1+\frac{\lambda_{t-s}}{(s+1)^{r}}+\frac{\lambda_{t-s}}{(t-s+1)^{r}}\right)\mathbb{E}[\|\boldsymbol{\tilde{\theta}}_{t-s}\|^2]\nonumber\\
&\quad+3m(\sigma^{+})^2\sum_{s=0}^{t-t'_{0}}\gamma_{t-s}^2(t-s+1)^{2\varsigma}+m\sum_{s=0}^{t-t'_{0}}\frac{\lambda_{t-s}}{(t-s+1)^{r}}\nonumber\\
&\quad+8m(\kappa^2+D^2)\sum_{s=0}^{t-t'_{0}}\frac{\lambda_{t-s}(s+1)^{r+1}}{(t+2)(t-s+1)}\nonumber\\
&\quad+4m\kappa^2\sum_{s=0}^{t-t'_{0}}\frac{\lambda_{t-s}(t-s+1)^{r}}{t+1}+2m\kappa R\sum_{s=0}^{t-t'_{0}}\frac{\lambda_{t-s}}{\sqrt{t+1}},\label{4T7}
\end{flalign}
where in the derivation we have used the definition $\eta_{s}=\frac{1}{(s+1)^{r}}$ with $r=\frac{1-v}{2}$ from the statement of Lemma~\ref{C2}.

We proceed to characterize the second term on the right
hand side of~\eqref{4T7}. Following an argument similar to that of~\eqref{2T41}, we have
\begin{equation}
\begin{aligned}
&\sum_{s=0}^{t-t'_{0}}\left(1+\frac{\lambda_{t-s}}{(s+1)^{r}}+\frac{\lambda_{t-s}}{(t+1-s)^{r}}\right)\mathbb{E}[\|\boldsymbol{\tilde{\theta}}_{t-s}\|^2]\\
&\quad-\sum_{s=0}^{t-t'_{0}}\mathbb{E}[\|\boldsymbol{\tilde{\theta}}_{t+1-s}\|^2]\\
&\leq \sum_{s=0}^{t-t'_{0}}\left(\frac{\lambda_{t-s}}{(s+1)^{r}}+\frac{\lambda_{t-s}}{(t+1-s)^{r}}\right)\mathbb{E}[\|\boldsymbol{\tilde{\theta}}_{t-s}\|^2]\\
&\quad-\mathbb{E}[\|\boldsymbol{\tilde{\theta}}_{t+1}\|^2]+\mathbb{E}[\|\boldsymbol{\tilde{\theta}}_{t'_{0}}\|^2]\\
&\leq \sum_{s=1}^{t-t'_{0}}\left(\frac{\lambda_{t-s}}{(s+1)^r}+\frac{\lambda_{t-s}}{(t+1-s)^{r}}\right)\mathbb{E}[\|\boldsymbol{\tilde{\theta}}_{t+1-s}\|^2]\\
&\quad+\left(1+\frac{\lambda_{0}}{(t-t'_{0}+1)^r}+\frac{\lambda_{0}}{(t'_{0}+1)^{r}}\right)\mathbb{E}[\|\boldsymbol{\tilde{\theta}}_{t'_0}\|^2]\\
&\quad-\mathbb{E}[\|\boldsymbol{\tilde{\theta}}_{t+1}\|^2],\label{4T8}
\end{aligned}
\end{equation}
where $\boldsymbol{\tilde{\theta}}_{t'_0}$ is given by $\boldsymbol{\tilde{\theta}}_{t'_0}\triangleq\boldsymbol{\theta}_{t'_0}-\boldsymbol{\theta}_{t+1}^{*}$. 

Furthermore, since $\theta_{t}^{i}$ is restricted in a compact set $\Theta$, we have 
$\mathbb{E}[\|\boldsymbol{\theta}_{p}-\boldsymbol{\boldsymbol{\theta}}_{t+1}^*\|^2]\leq mR^2$ valid for all $p\geq 0$. By using the relationship $\mathbb{E}[F_{t+1}(\theta_{t-s}^{i})-F_{t+1}(\theta_{t+1}^*)]\geq \mathbb{E}[F_{t+1}(\theta_{t+1}^{i})-F_{t+1}(\theta_{t+1}^*)]$, we can substitute~\eqref{4T8} into~\eqref{4T7} and omit the negative term $-\mathbb{E}[\|\tilde{\boldsymbol{\theta}}_{t+1}\|^2]$to obtain 
\begin{equation}
2\sum_{s=0}^{t-t'_{0}}\lambda_{t-s}\mathbb{E}[F_{t+1}(\theta_{t+1}^{i})-F_{t+1}(\theta_{t+1}^*)]\leq \Delta_{t},\label{4T9}
\end{equation}
where
\begin{flalign}
&\Delta_{t}=\sum_{s=1}^{t-t'_{0}}\left(\frac{\lambda_{t-s}}{(s+1)^r}+\frac{\lambda_{t-s}}{(t+1-s)^{r}}\right)R^2\nonumber\\
&\quad+\left(1+\frac{\lambda_{0}}{(t-t'_{0}+1)^r}+\frac{\lambda_{0}}{(t'_{0}+1)^{r}}\right)R^2\nonumber\\
&\quad+3(\sigma^{+})^2\sum_{s=0}^{t-t'_{0}}\gamma_{t-s}^2(t-s+1)^{2\varsigma}+\sum_{s=0}^{t-t'_{0}}\frac{\lambda_{t-s}}{(t+1-s)^r}\nonumber\\
&\quad+8(\kappa^2+D^2)\sum_{s=0}^{t-t'_{0}}\frac{\lambda_{t-s}(s+1)^{r+1}}{(t+2)(t-s+1)}\nonumber\\
&\quad+4\kappa^2\sum_{s=0}^{t-t'_{0}}\frac{\lambda_{t-s}(t-s+1)^{r}}{t+1}+2\kappa R\sum_{s=0}^{t-t'_{0}}\frac{\lambda_{t-s}}{\sqrt{t+1}}.\label{D4}
\end{flalign}

2) By using the following relation $$\sum_{s=0}^{t-t'_{0}}\lambda_{t-s}=\sum_{s=t'_{0}}^{t}\lambda_{s}\geq 2\lambda_{0}\int_{t'_{0}}^{t+1}\frac{1}{(x+1)^{v}}dx,$$ we have that the first term on the left hand side of~\eqref{4T9} satisfies
\begin{equation}
2\sum_{s=0}^{t-t'_{0}}\lambda_{t-s}\geq \frac{2\lambda_{0}(t+2)^{1-v}-2\lambda_{0}(t'_{0}+1)^{1-v}}{1-v}.\label{D4i}
\end{equation}

We proceed with the calculation of the upper bound on $\Delta_{t}$ in~\eqref{D4}. Following similar arguments to those of~\eqref{2T6i}, \eqref{2T6ii}, \eqref{2T6iii}, \eqref{2T6iv2}, \eqref{2T6vi} and \eqref{2T6vii}, respectively, we can obtain the following inequalities:
\begin{equation}
\left\{\begin{aligned}
&\sum_{s=1}^{t-t'_{0}}\frac{\lambda_{t-s}}{({s+1})^{r}}\leq\sum_{s=1}^{t}\frac{\lambda_{t-s}}{({s+1})^{r}}\leq\frac{2\lambda_{0}}{1-r-v}((t\!+\!1)^{1-r-v}\!-\!1),\\
&\sum_{s=1}^{t-t'_{0}}\frac{\lambda_{t-s}}{({t-s+1})^{r}}=\sum_{s=t'_0}^{t-1}\frac{\lambda_{s}}{({s+1})^{r}}\leq\sum_{s=0}^{t}\frac{\lambda_{s}}{({s+1})^{r}}\\
&\quad\quad\quad\quad\quad\quad\quad\quad\leq\frac{\lambda_{0}}{1-r-v}(t+1)^{1-r-v},\\
&\sum_{s=0}^{t-t'_{0}}\gamma_{t-s}^2(t-s+1)^{2\varsigma}\leq \sum_{s=0}^{t}\gamma_{s}^2(s+1)^{2\varsigma}\leq \frac{\gamma_{0}^2(2u-2\varsigma)}{2u-2\varsigma-1},\\
&\sum_{s=0}^{t-t'_0}\frac{\lambda_{t-s}(s+1)^{r+1}}{(t+2)(t-s+1)}\leq \sum_{s=0}^{t}\frac{\lambda_{t-s}(s+1)^{r+1}}{(t+2)(t-s+1)}\\
&\quad\quad\quad\quad\quad\quad\quad\quad\quad~~\leq\lambda_{0}\left(\frac{1}{v}+1\right)(t+1)^{r},\\
&\sum_{s=0}^{t-t'_{0}}\frac{\lambda_{t-s}(t-s+1)^{r}}{(t+1)}\leq\sum_{s=0}^{t}\frac{\lambda_{t-s}(t-s+1)^{r}}{(t+1)}\\
&\quad\quad\quad\quad\quad\quad\quad\quad\quad~\leq\frac{1}{(1+r-v)(t+1)^{v-r}},\\
&\sum_{s=0}^{t-t'_{0}}\frac{\lambda_{t-s}}{\sqrt{t+1}}\leq\sum_{s=0}^{t}\frac{\lambda_{t-s}}{\sqrt{t+1}}\leq\frac{1}{(1-v)(t+1)^{v-\frac{1}{2}}}.\label{D4ii}
\end{aligned}
\right.
\end{equation}

Incorporating~\eqref{D4i} and~\eqref{D4ii} into~\eqref{4T9}, we have
\begin{flalign}
&\frac{2\lambda_{0}((t+2)^{1-v}-(t'_{0}+1)^{1-v})}{1-v}\mathbb{E}\left[F_{t+1}(\theta_{t+1}^{i})-F_{t+1}(\theta_{t+1}^*)\right]\nonumber\\
&\leq c_{1}+ \frac{\lambda_{0}R^2}{(t-t'_{0}+1)^r}+\frac{3\lambda_{0}R^2+\lambda_{0}}{(1-r-v)(t+1)^{r+v-1}}\nonumber\\
&\quad+\frac{8(\kappa^2+D^2)\lambda_{0}(1+v)(t+1)^{r}}{v}\nonumber\\
&\quad+\frac{4\kappa^2}{(1+r-v)(t+1)^{v-r}}+\frac{2\kappa R}{(1-v)(t+1)^{v-\frac{1}{2}}},\label{4T10}
\end{flalign}
with $c_{1}=(1+\frac{\lambda_{0}}{(t'_{0}+1)^{r}})R^2+\frac{3(\sigma^{+})^2\gamma_{0}^2(2u-2\varsigma)}{2u-2\varsigma-1}.$

Multiplying both sides of~\eqref{4T10} by $\frac{1-v}{2\lambda_{0}((t+2)^{1-v}-(t'_{0}+1)^{1-v})}$ yields
\begin{equation}
\begin{aligned}
&\mathbb{E}\left[F_{t+1}(\theta_{t+1}^{i})-F_{t+1}(\theta_{t+1}^*)\right]\\
&\leq\frac{c_{1}(1-v)}{2\lambda_{0}((t+2)^{1-v}-(t'_{0}+1)^{1-v})}\\
&\quad+\frac{R^2(1-v)}{2(t-t'_{0}+1)^{r}((t+2)^{1-v}-(t'_{0}+1)^{1-v})}\\
&\quad+\frac{(3R^2+1)(1-v)}{2(1-r-v)(t+1)^{r+v-1}((t+2)^{1-v}-(t'_{0}+1)^{1-v})}\\
&\quad+\frac{(\kappa^2+D^2)(1+v)(1-v)}{v(t+1)^{-r}((t+2)^{1-v}-(t'_{0}+1)^{1-v})}\\
&\quad+\frac{\kappa R}{\lambda_{0}(t+1)^{v-\frac{1}{2}}((t+2)^{1-v}-(t'_{0}+1)^{1-v})}\\
&\quad+\frac{\kappa^2(1-v)}{2\lambda_{0}(1+r-v)(t+1)^{v-r}((t+2)^{1-v}-(t'_{0}+1)^{1-v})}.\label{4T11}
\end{aligned}
\end{equation}

3) Using the relation $(t+2)^{1-v}(t-t'_{0}+1)^{r}\geq (t'_{0}+2)^{1-v}(t-t'_{0}+1)^{r}$, we obtain
\begin{equation}
\begin{aligned}
&(t-t'_{0}+1)^{r}((t+2)^{1-v}-(t'_{0}+1)^{1-v})\\
&\geq \left(1-\frac{(t'_{0}+1)^{1-v}}{(t'_{0}+2)^{1-v}}\right)(t-t'_{0}+1)^{r}(t+2)^{1-v}\\
&\geq c_{2}(t-t'_{0}+1)^{r}(t+1)^{1-v},
\end{aligned}
\end{equation}
with $c_{2}=1-\frac{(t'_{0}+1)^{1-v}}{(t'_{0}+2)^{1-v}}$. 
By using a similar argument for each item on the right hand side of~\eqref{4T11} and substituting $r=\frac{1-v}{2}$ given in the statement of Lemma~\ref{C2} into~\eqref{4T11}, we arrive at
\begin{flalign}
&\mathbb{E}\left[F_{t+1}(\theta_{t+1}^{i})-F_{t+1}(\theta_{t+1}^*)\right]\nonumber\\
&\leq \frac{c_{3}}{(t+1)^{1-v}}+\frac{c_{4}}{(t-t'_{0}+1)^{\frac{1-v}{2}}(t+1)^{1-v}}+\frac{c_{5}+c_{6}}{(t+1)^{\frac{1-v}{2}}}\nonumber\\
&\quad+\frac{c_{7}}{(t+1)^{\frac{1}{2}}}+\frac{c_{8}}{(t+1)^{\frac{1+v}{2}}}=\mathcal{O}((t+1)^{-\beta}),\label{4T12}
\end{flalign}
for any $t>t'_{0}$, where the rate $\beta$ is given by  $\beta=\frac{1-v}{2}$ and the constants $c_{3}$ to $c_{8}$ are given by $c_{3}=\frac{c_{1}(1-v)}{2\lambda_{0}c_{2}}$, $c_{4}=\frac{R^2(1-v)}{2c_{2}}$, $c_{5}=\frac{(3R^2+1)(1-v)}{2(1-r-v)c_{2}}$, $c_{6}=\frac{(\kappa^2+D^2)(1-v^2)}{vc_{2}}$, $c_{7}=\frac{\kappa R}{\lambda_{0}c_{2}}$, and $c_{8}=\frac{\kappa^2(1-v)}{2\lambda_{0}(1+r-v)c_{2}}$, respectively.
\subsection{Proof of Lemma~\ref{L2}}\label{SectionG}
\begin{proof}
For any given time instant $t$, when the dataset $\mathcal{D}_{t}^{i}$ and the initial parameter $\theta_{0}^{i}$ for all $i\in[m]$ are fixed, we show that $\mathcal{A}_{i}(\mathcal{D}_{t}^{i},\theta_{t}^{-i})$ gives a singleton.

For any given dataset $\mathcal{D}_{t}^{i}$, the communication graph $W$, the decaying sequences $\gamma_{t}$ and $\lambda_{t}$, and the initial parameters $\theta_{0}^{i}$ for all $i\in[m]$, we can conclude that $\theta_{t+1}^{i}$ is uniquely specified based on Line 7 in Algorithm~\ref{algorithm1}. Therefore, Learner $i$'s implementations of Algorithm~\ref{algorithm1} from $t=1$ to $t=T$, which we denote as $\alpha^{i} \triangleq\langle\theta_{1}^{i},y_{1}^{i}\rangle,\cdots,\langle\theta_{T}^{i},y_{T}^{i}\rangle$, are uniquely determined.

We denote an execution set $A^{i}$ from $t=1$ to $t=T$ as $A^{i}=\{\mathcal{A}_{i}(\mathcal{D}_{t}^{i},\theta_{t}^{-i})\}_{t=1}^{T}$ and define a correspondence $B^{i}$ between the sets $A^{i}$ and ${A^{i}}'$. For $\alpha^{i}\in A^{i}$ and ${\alpha^{i}}'\in {A^{i}}'$, ${B^{i}}(\alpha^{i})={\alpha^{i}}'$ holds if and only if they have the same observation sequence. For any fixed observation sequence in $\mathcal{O}^{i}$, there is a unique execution $\alpha^{i}\in A^{i}$ that can produce this observation. Similarly, ${\alpha^{i}}'$ is also unique in ${A^{i}}'$. Hence, $B^{i}$ is a bijection. We relate the probability measures of the sets of executions $A^{i}$ and ${A^{i}}'$:
\begin{equation}
\frac{\mathbb{P}[\{\mathcal{A}_{i}(\mathcal{D}_{t}^{i},\theta_{t}^{-i})\}_{t=1}^{T}]}{\mathbb{P}[\{\mathcal{A}_{i}({\mathcal{D}_{t}^{i}}',\theta_{t}^{-i})\}_{t=1}^{T}]}=\frac{\int_{\alpha^{i}\in A^{i}}\mathbb{P}[\alpha^{i}]d\mu}{\int_{{\alpha^{i}}'\in {A^{i}}'}\mathbb{P}[{\alpha^{i}}']d\mu'}.\label{2L1}
\end{equation}
By changing the variable using the bijection $B^{i}$, we have
\begin{equation}
\begin{aligned}
\int_{{\alpha^{i}}'\in {A^{i}}'}\mathbb{P}[{\alpha^{i}}']d\mu'&=\int_{B^{i}(\alpha^{i})\in {A^{i}}'}\mathbb{P}[B^{i}(\alpha^{i})]d\mu\\
&=\int_{\alpha^{i}\in {A^{i}}}\mathbb{P}[B^{i}(\alpha^{i})]d\mu.\label{2L2}
\end{aligned}
\end{equation}

We denote the $q$-th element of $\theta_{t}^{i}$ as $[\theta_{t}^{i}]_{q}$. Based on Line 5 in Algorithm~\ref{algorithm1}, we have that $y_{t}^{i}$ is generated by adding $n$ independent noises to $\theta_{t}^{i}$ from the distribution $\text{Lap}(\varrho_{t}^{i})$. Hence, the probability density of an execution of $\alpha^{i}$ is reduced to
\begin{equation}
\mathbb{P}[\alpha^{i}]=\prod_{q\in[n],t\in[T]}p_{\varrho_{t}^{i}}([y_{t}^{i}]_{q}-[x_{t}^{i}]_{q}),\label{2L3}
\end{equation}
where $p_{\varrho_{t}^{i}}(x)$ is the probability density function of $\text{Lap}(\varrho_{t}^{i})$ at $x$. Then, for any $t\in[1,T]$, we proceed to relate the distance at time $t$ between $\alpha^{i}$ and $B^{i}(\alpha^{i})$ with the sensitivity $\Delta_{t}^{i}$. 

Based on Definition~\ref{Definition5}, we have
\begin{equation}
\|\theta_{t}^{i}-{\theta_{t}^{i}}'\|_{1}\leq \Delta_{t}^{i}.\label{2L4}
\end{equation}
The definition of $l^{1}$-norm (i.e., $\|x\|_{1}=\sum_{i=q}^{n}|[x]_{q}|$ for any $x\in\mathbb{R}^{n}$) implies
\begin{equation}
\sum_{q=1}^{n}|[\theta_{t}^{i}]_{q}-[{\theta_{t}^{i}}']_{q}|=\|\theta_{t}^{i}-{\theta_{t}^{i}}'\|_{1}\leq \Delta_{t}^{i}.\label{2L5}
\end{equation}

By using the property of Laplace distribution (i.e., the inequality $\frac{p_{\varrho}(x)}{p_{\varrho}(y)}\leq e^{\frac{|y-x|}{\varrho}}$ always holds for any $x,y\in\mathbb{R}$), we have
\begin{equation}
\begin{aligned}
&\prod_{q\in[n]}\frac{p_{\varrho_{t}^{i}}([y_{t}^{i}]_{q}-[x_{t}^{i}]_{q})}{p_{\varrho_{t}^{i}}([{y_{t}^{i}}']_{q}-[{x_{t}^{i}}']_{q})}\\&\leq\prod_{q\in[n]}\text{exp}\left(\frac{|[y_{t}^{i}]_{q}-[x_{t}^{i}]_{q}-[{y_{t}^{i}}']_{q}+[{x_{t}^{i}}']_{q}|}{\varrho_{t}^{i}}\right)\\
&\leq \text{exp}\left(\sum_{q\in[n]}\frac{|[x_{t}^{i}]_{q}-[{x_{t}^{i}}']_{q}|}{\varrho_{t}^{i}}\right)\leq e^{\frac{\Delta_{t}^{i}}{\varrho_{t}^{i}}}.\label{2L6}
\end{aligned}
\end{equation}
Combining~\eqref{2L1},~\eqref{2L2}, and~\eqref{2L3} with~\eqref{2L6}, we obtain
\begin{equation}
\begin{aligned}
&\frac{\mathbb{P}[\{\mathcal{A}_{i}(\mathcal{D}_{t}^{i},\theta_{t}^{-i})\}_{t=0}^{T-1}]}{\mathbb{P}[\{\mathcal{A}_{i}({\mathcal{D}_{t}^{i}}',\theta_{t}^{-i})\}_{t=0}^{T-1}]}=\frac{\int_{\alpha^{i}\in A^{i}}\mathbb{P}[\alpha^{i}]d\mu}{\int_{\alpha^{i}\in A^{i}}\mathbb{P}[B^{i}(\alpha^{i})]d\mu}\\
&\leq \frac{\int_{\alpha^{i}\in A^{i}}e^{\sum_{t=1}^{T}\frac{\Delta_{t}^{i}}{\varrho_{t}^{i}}}\mathbb{P}[B^{i}(\alpha^{i})]d\mu}{\int_{\alpha^{i}\in A^{i}}\mathbb{P}[B^{i}(\alpha^{i})]d\mu}\leq e^{\sum_{t=1}^{T}\frac{\Delta_{t}^{i}}{\varrho_{t}^{i}}},
\end{aligned}
\end{equation}
which implies that Learner $i$'s implementation $\mathcal{A}_{i}$ is $\epsilon_{i}$ locally differentiable private with the cumulative privacy budget from $t=1$ to $t=T$ upper bounded by $\sum_{t=1}^{T}\frac{\Delta_{t}^{i}}{\varrho_{t}^{i}}$, no matter $T$ is finite or $T\rightarrow\infty$.
\end{proof}

\bibliographystyle{IEEEtran}  
\bibliography{nonconvexquantization}

\end{document}